\documentclass[11pt]{article}
\usepackage{mathtools}
\usepackage{amsfonts}
\usepackage{amssymb}
\usepackage[titletoc,title]{appendix}
\usepackage{authblk}
\usepackage[T1]{fontenc}
\usepackage[utf8]{inputenc}
\usepackage[symbol*]{footmisc}
\usepackage{color}
\usepackage{amsmath}
\usepackage{float}
\usepackage{epsf}
\usepackage{amsthm}
\usepackage{times}
\usepackage{bm}
\usepackage{natbib}
\usepackage{tabularx}
\usepackage[plain,noend]{algorithm2e}
\usepackage{epsfig, psfrag, psfig, pstricks}
\usepackage[outdir=./]{epstopdf}
\usepackage[font=small, margin = 1cm]{caption}
\usepackage{xspace}
\usepackage{multirow}
\newcommand{\specialcell}[2][c]{%
  \begin{tabular}[#1]{@{}c@{}}#2\end{tabular}}

\usepackage{xr}
\makeatletter
\renewcommand{\algocf@captiontext}[2]{#1\algocf@typo. \AlCapFnt{}#2} 
\def\@algocf@capt@plain{top}

\renewcommand{\algocf@makecaption}[2]{%
  \addtolength{\hsize}{\algomargin}%
  {
  \sbox\@tempboxa{\algocf@captiontext{#1}{#2}}%
  \ifdim\wd\@tempboxa >\hsize
    \hskip .5\algomargin%
    \parbox[t]{\hsize}{\algocf@captiontext{#1}{#2}}
  \else%
    \global\@minipagefalse%
    \hbox to\hsize{\box\@tempboxa}
  \fi%
  \addtolength{\hsize}{-\algomargin}} %
}
\makeatother

\def\T{{ \mathrm{\scriptscriptstyle T} }}

\def\v{{\varepsilon}}
\newcommand{\cov}{\mathrm{cov}}
\newcommand{\bO}{\Omega}
\newcommand{\tr}{\mbox{tr}}
\newcommand{\FTwo}{$\text{F}_2$\xspace}

\def\argmin{\mathop{\rm argmin}}
\def\argmax{\mathop{\rm argmax}}

\def\mtr{\mathop{\rm tr}}
\def\mone{\mathop{\rm one}}
\def\mEM{\mathop{\rm EM}}

\newtheorem{prop}{Proposition}

\newtheorem{theorem}{Theorem}[section]
\newtheorem{corollary}[theorem]{Corollary}
\newtheorem{lemma}[theorem]{Lemma}

\theoremstyle{remark}
\newtheorem{property}{Property}

\def\tabnotefont{\reset@font\fontsize{9}{11}\selectfont\leftskip\tabledim\rightskip\tabledim}%

\newcommand{\edit}[1]{#1}
\newcommand{\editmath}{}
\allowdisplaybreaks
\DeclarePairedDelimiter{\ceil}{\lceil}{\rceil}
\usepackage{geometry}
\begin{document}


\title{Joint Estimation of Multiple Dependent Gaussian Graphical Models with Applications to Mouse Genomics}

\author[1]{Yuying Xie}
\author[2]{Yufeng Liu\thanks{
    Correspondence to:
      Yufeng Liu, Department of Statistics and Operations Research, CB3260, University of North Carolina, Chapel Hill, NC
27599. E-mail: yfliu@email.unc.edu.}}
\author[3]{William Valdar}
\affil[1]{Department of Computational Mathematics, Science and Engineering, Michigan State University, MI, USA}
\affil[2]{Department of Statistics and Operations Research, University of North Carolina at Chapel Hill, NC, USA}
\affil[3]{Department of Genetics, University of North Carolina at Chapel Hill, NC, USA}
\date{}

\maketitle

\begin{abstract}
   \noindent Gaussian graphical models are widely used to represent conditional dependence among random variables. In this paper, we propose a novel estimator for data arising from a group of Gaussian graphical models that are themselves dependent. A motivating example is that of modeling gene expression collected on multiple tissues from the same individual: \edit{here the multivariate outcome is affected by dependencies acting not only at the level of the specific tissues, but also at the level of the whole body; existing methods that assume independence among graphs are not applicable in this case.} To estimate multiple dependent graphs, we decompose the problem into two graphical layers: the systemic layer, which affects all outcomes and thereby induces cross-graph dependence, and the category-specific layer, which represents graph-specific variation. We propose a graphical EM technique that estimates both layers jointly, establish estimation consistency and selection sparsistency of the proposed estimator, and confirm by simulation that the EM method is superior to a simple one-step method. We apply our technique to mouse genomics data and obtain biologically plausible results.
\end{abstract}

\section*{\normalsize \centering Keywords}
  EM algorithm, Gaussian graphical model, mouse genomics, shrinkage, sparsity, variable selection

\section{Introduction}
\setcounter{page}{1}
\label{s:intro}

Gaussian graphical models are widely used to represent conditional dependencies among sets of normally distributed outcome variables that are observed together. For example, observed, and potentially dense, correlations between measurements of expression for multiple genes, stock market prices of different asset classes, or blood flow for multiple voxels in functional magnetic resonance imaging, i.e., fMRI-measured brain activity, can often be more parsimoniously explained by an underlying graph whose structure may be relatively sparse. As methods for estimating these underlying graphs have matured, a number of elaborations to basic Gaussian graphical models have been proposed, including those that seek either to model the sampling distribution of the data more closely, or to model prior expectations of the analyst about structural similarities among graphs representing related data sets \citep{Guo2011,Danaher2013, leeliu15}. In this paper, we propose an elaboration that seeks to model an additional feature of the sampling distribution increasingly encountered in biomedical data, whereby correlations among the outcome variables are considered to be the byproduct of underlying conditional dependencies acting at different levels. For illustration, consider gene expression data obtained from multiple tissues, such as liver, kidney, and brain, collected on each individual. In this setting, observed correlations between expressed genes may be caused by dependence structures not only within a specific tissue but also across tissues at the level of the whole body. We describe these distinct graphical strata respectively as the category-specific and systemic layers, and \edit{model their respective outcomes as latent variables}.

The conditional dependence relationships among $p$ outcome variables, $Y = (Y_1, \ldots, Y_p)$, can be represented by a graph $\mathcal{G} = (\Gamma, E)$, where each variable is a node in the set $\Gamma$ and conditional dependencies are represented by the edges in the set $E$. If the joint distribution of the outcome variables is multivariate Gaussian, $Y \sim \mathcal{N}(0, \Sigma)$, then conditional dependencies are reflected in the non-zero entries of the precision matrix $\Omega = \Sigma^{-1}$. Specifically, two variables $Y_i$ and $Y_j$ are conditionally independent given the other variables if and only if the $(i,j)$th entry of $\Omega$ is zero. Inferring the dependence structure of such a Gaussian graphical model is thus the same as estimating which elements of its precision matrix are non-zero.

When the underlying graph is sparse, as is often assumed, the maximum likelihood estimator is dominated \edit{in terms of false positive rate} by shrinkage estimators. The maximum likelihood estimate of $\Omega$ typically implies a graph that is fully connected, which is unhelpful for estimating graph topology. To impose sparsity, and thereby provide a more informative inference about network structure, a number of methods have been introduced that estimate $\Omega$ under $\ell_1$ regularization. \citet{Meinshausen2006} proposed to iteratively determine the edges of each node in $\mathcal{G}$ by fitting an $\ell_1$ penalized regression model to the corresponding variable $Y_j$ using the remaining variables $Y_{-j}$ as predictors, an approach which can be viewed as optimizing a pseudo-likelihood \citep{Ambroise2009,Peng2009}. More recently, numerous papers have proposed estimation using sparse penalized maximum likelihood \citep[]{Yuan2007,Banerjee2008,dAspremont2008,Rothman2008,Ravikumar2011}. Efficient implementations include the graphical lasso algorithm \citep{Friedman2008} and the quadratic inverse covariance algorithm \citep{Hsieh2011}. The convergence rate and selection consistency of such penalized estimation schemes have also been investigated in theoretical studies \citep{Rothman2008,Lam2009}.

Although a single graph provides a useful representation of an underlying dependence structure, several extensions have been proposed. In the context where the precision matrix, and hence the graph, is dynamic over time, \citet{Zhou2010} proposed a weighted method to estimate the graph's temporal evolution. Another extension is to simultaneously estimate multiple graphs that may share some common structure. For example, when inferring how brain regions interact using fMRI data, each subject's brain corresponds to a different graph, but we would nonetheless expect some common interaction patterns across subjects, as well as patterns specific to an individual. In such cases, joint estimation of multiple related graphs can be more efficient than estimating graphs separately.  For joint estimation of Gaussian graphs, \citet{Varoquaux2010} and  \citet{Honorio2010} proposed methods using group lasso 
 and  multitask lasso, respectively. Both assume that all the precision matrices have the same pattern of zeros. To provide greater flexibility, \citet{Guo2011} proposed a joint penalized method using a hierarchical penalty, and derived the convergence rate and sparsistency properties for the resulting estimators. In the same setting, \citet{Danaher2013} extended the graphical lasso \citep{Friedman2008} to estimate multiple graphs from independent data sets using penalties based on the generalized fused lasso or, alternatively, the sparse group lasso.

\begin{figure}[bh]
	\begin{center}
		\includegraphics[width = 5.5 in]{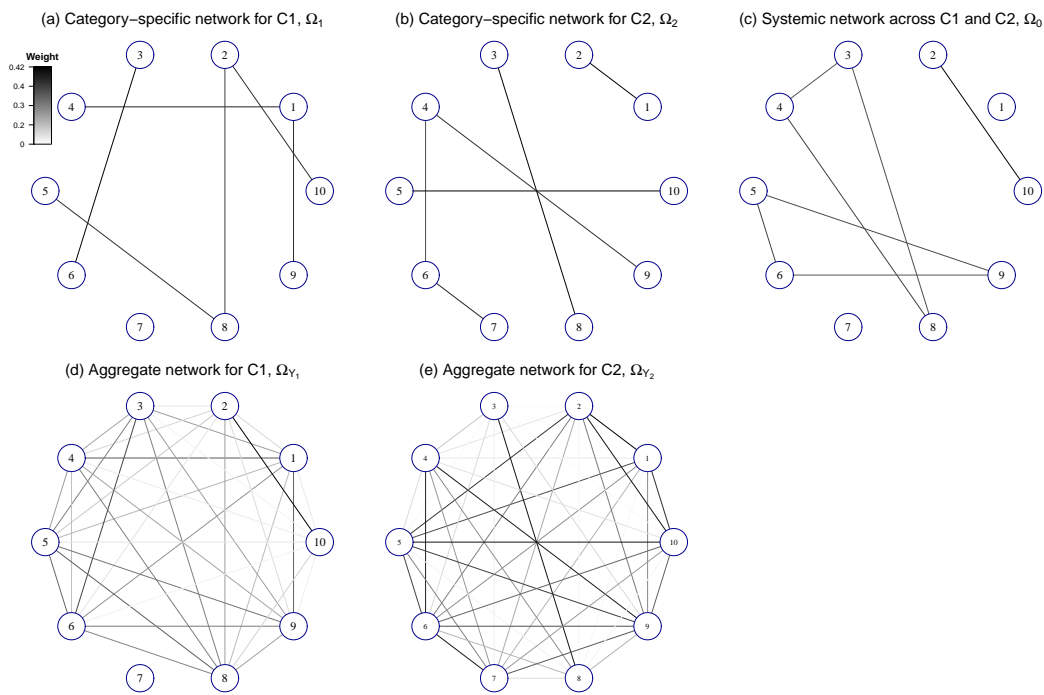}
	\caption{Illustration of systemic and category-specific networks  using a toy example with two categories ($C1$ and $C2$) and $p = 10$ variables. (a) Category-specific network for $C1$. (b) Category-specific network for $C2$. (c) Systemic network affecting variables in both $C1$ and $C2$. (d) Aggregate network, $\bO_{Y_1} = (\bO_{1}^{-1} + \bO_{0}^{-1})^{-1}$, for category $C1$. (e) Aggregate network, $\bO_{Y_2} = (\bO_{2}^{-1} + \bO_{0}^{-1})^{-1}$, for $C2$.}
	\label{fig:toy}
	\end{center}
\end{figure}

The above methods for estimating multiple Gaussian graphs focus on the settings in which data collected from different categories are stochastically independent. In some applications, however, data from different categories are more naturally considered as dependent. In a study considered here, gene expression data have been collected on multiple tissues in multiple mice. For each mouse we have expression measurements for $p$ genes in each of $K$ different tissues, that is, $K$ different categories, represented by the $p$-dimensional vectors $Y_k$ $(k = 1, \ldots, K)$. In this setting, the gene expression profiles of different mice may have arisen from the same network structure, but they are otherwise stochastically independent; in contrast, the gene expression profiles of different tissues within the same mouse are stochastically dependent. For such data, increasingly common in biomedical research, the above methods are not applicable.

To explore the gene network structure across different tissues, and to characterize the dependence among tissues, we consider a decomposition of the observed gene expression $Y_k$ into two latent vectors. In our model, we define 
\begin{align}\label{eq:yk}
	Y_k = Z + X_k,
\end{align}
where $Z, X_1, \ldots, X_K$ are mutually independent. Because $\text{cov}(Y_k, Y_l) = \text{var}(Z)$ for any $k \neq l $, $Z$ represents dependence across different tissues. Letting $\Omega_j$ denote the precision matrix of $X_j$ for tissue $j$, and defining $\text{var}(Z) = \Omega_0^{-1}$, we aim to estimate $\Omega_k$ $(k = 0, \ldots, K)$ from the observed outcome data on $\{Y_1, \ldots, Y_K \}$. To accomplish this joint estimation of multiple dependent networks, we propose a one-step method and an EM method.

In the above decomposition, $Z$ can be viewed as representing systemic variation in gene expression, that is, variation manifesting simultaneously in all measured tissues of the same mouse, whereas $X_k$ represents category-specific variation, that is, variation unique to tissue $k$. An important property of this two-layer model is that sparsity in the systemic and category-specific networks can produce networks for the outcome variable $Y$ that is highly connected. Conversely, highly connected graphs for the outcome $Y$ can easily arise from relatively sparse underlying dependencies acting at two levels. This phenomenon is illustrated in Fig. \ref{fig:toy}, which depicts category-specific networks $\Omega_1$ and $\Omega_2$ for two categories $C1$ and $C2$, which might correspond to, for example, liver and brain tissue-types, and a systemic network $\Omega_0$, which reflects relationships affecting all tissues at once, for example, gene interactions that are responsive to hormone levels or other globally-acting processes. Although all three underlying networks, $\Omega_0$, $\Omega_1$ and $\Omega_2$, are sparse, the precision matrix of observed variables within each tissue, that is, the aggregate network $\bO_{Y_k} = (\bO_0^{-1} + \bO_k^{-1})^{-1}$ following  \eqref{eq:yk} is  highly connected. Existing methods aiming to estimate a single sparse network layer are therefore ill-suited to this problem because they impose sparsity on the aggregate network rather than on the two simpler layers that generate it. 

\section{Methodology}
\label{s:method}

\subsection{Problem formulation}

The following notation is used throughout the paper. We denote the true precision and covariance matrices by $\bO^{*}$ and $\Sigma^{*}$. For any matrix $W = (\omega_{ij})$, we denote the determinant by $\det( W )$, the trace by $\tr(W)$ and  the off-diagonal entries of $W$ by $W^-$. We further denote the $j$th eigenvalue of $W$ by $\phi_j(W)$, and the minimum and maximum eigenvalues of $W$ by $\phi_{\text{min}}(W)$ and $\phi_{\text{max}}(W)$. The Frobenius norm, $\| W \|_F$, is defined as $ \sum_{i,j} \omega_{ij}^2$; the operator/spectral norm, $\| W \|^2$, is defined as $\phi_{\text{max}}(W W^{\T})$; the infinity norm, $\| W \|_\infty$, is defined as $\text{max} |w_{ij} | $; and the element-wise $L_1$ norm, $|W|_1$, is defined as $\sum_{i,j} |\omega_{ij}|$.

In the problem, we address, measurements are available on the same $p$ outcome variables in each of $K$ distinct categories on each of $n$ individuals. Some dependence is anticipated among outcomes both at the level of the category and at the level of the individual: dependence at the level of the category is described as category-specific; dependence at the level of the individual is described as systemic, that is, modeled as if affecting outcomes in all categories of the same individual simultaneously. Our primary example is the measurement of gene expression giving rise to transcript abundance readings on $p$ genes on $K$ tissues, such as liver, kidney and brain, in $n$ laboratory mice.

{Letting $Y_{k,i}$ be the $i$th data vector for the $k$th category, we model}
\begin{equation} \label{eq:str}
	Y_{k,i} = X_{k,i} + Z_i
	\quad  (i = 1,\ldots,n; \ k = 1, \ldots, K),
\end{equation}
where $Z_i$ is the random vector corresponding to the shared systemic random effect, and $X_{k,i}$ is the random vector corresponding to the $k$th category. We assume that $X_{k,i}$ and $Z_i$  $(i = 1, \ldots, n; \ k = 1, \ldots, K)$ are independent and identically distributed $p$-dimensional random vectors with mean $0$, and covariance matrices $\Sigma_{k}$ and $\Sigma_{0}$ respectively. We further assume that $X_{k,i}$, and $Z_i$ are independent of each other and each follows a multivariate Gaussian distribution.

For the $i$th sample in the $k$th category, we observe the $p$-dimensional realization of $ Y_{k, i}$, vector $y_{k,i} = (y_{k, i, 1}, \ldots, y_{k,i, p})^\T$.  Without loss of generality, we assume these observations are centered, i.e., $\sum_{i=1}^n y_{k,i,j}=0 \; (j = 1, \ldots, p; \  k = 1, \ldots, K$). Let $y_{\cdot,i}$ be the combined data vector with $y_{\cdot,i} = (y_{1,i}^\T, \ldots,y_{K,i}^\T )^\T$, such that $y_{\cdot,i}$ follows a Gaussian distribution with covariance $\Sigma_Y = \{_d \Sigma_k \} + J \otimes \Sigma_{0} = \{ \Sigma_{Y(l,m)}\}_{1 \leq l,m \leq K}$, where $\{_d \cdot \}$ is a block diagonal matrix, $J$ is a square matrix with all $1'{\text{s}}$ as the entries, $\otimes $ is the Kronecker product, and $\Sigma_{Y(l,m)}$ is the covariance matrix between $Y_l$ and $Y_m$. We denote the $n$ by $Kp$ dimensional data matrix by $y = (y_{\cdot,1} , \ldots, y_{\cdot, n})^\T$, and let  $\Omega_{k} = (\Sigma_{k})^{-1} = (\omega_{k(i,j)} )$, and $\Omega_Y = (\Sigma_Y)^{-1}$. Our goal is to estimate $\bO_k$  $(k = 0, \ldots, K)$. Although $X_{k}$ and $Z$ are latent variables, we can show that $\bO_k$ is identifiable under the model setup in \eqref{eq:str} with $K \geq 2$. More details can be found in the Supplementary Material. For simplicity, we denote $\bO$ and $\Sigma$ as $\{ \bO_k \}_{k = 0}^K$ and $\{ \Sigma_k \}_{k = 0}^K$ respectively in the following derivation.

The log-likelihood of the data can be written as
\begin{equation}\label{eq:SigmaY}
	\mathcal{L} (\Omega ; y) =
	-\frac{npK}{2}\log(2\pi)+\frac{n}{2} \big\{ \log \det(\Omega_Y)
	- \tr(\hat{\Sigma}_Y \Omega_Y) \big\} \, ,
\end{equation}
where
\begin{align*}
	\hat{\Sigma}_Y = n^{-1} \sum_{i=1}^n y_{\cdot,i} y^\T_{\cdot,i} 
	= \{ \hat{\Sigma}_{Y(l,m)}  \}_{1 \leq l,m \leq K}
\end{align*}
 is the $Kp \times Kp$ sample covariance matrix. In our setting,
\begin{align*}
	\mathcal{L}(\Omega;y) 
	\propto & \sum_{k=1}^K \big\{ \log \det(\Omega_{k})
	- \tr(\hat{\Sigma}_{Y(k,k)}\Omega_k) \big\}
	+ \log \det(\Omega_{0})  \\
	& - \log \det (A) + \sum_{l,m=1}^K 
	\tr \big( \Omega_l \hat{\Sigma}_{Y(l,m)} \Omega_m  A^{-1} \big) \,, 
\end{align*}
where $ A= \sum_{k = 0}^{K}  \Omega_{k} $; see the Supplementary Material for details.

A natural way to obtain a sparse estimate of $\bO$ is to maximize the penalized log-likelihood
\begin{align}\label{eq:Penlikelihood}\hspace{-0.1in}
	\hat{\bO} = \argmax_{\bO \succ 0} \mathcal{P}(\bO ; y)
	= \argmax_{\bO \succ 0} \mathcal{L}(\Omega ; y ) 
	- \lambda_1 \sum_{k = 1}^{K} |\bO_k^- |_1 - \lambda_2 | \bO_0^- |_1.
\end{align}
Because the likelihood is complicated, direct estimation of the precision matrices in \eqref{eq:Penlikelihood} is difficult. Estimation can proceed directly, however, given the values $z$ of the latent outcome vector $Z$. Therefore, we first estimate $\Sigma_0$ and then the other parameters. In Sections \ref{ss:onestep} and \ref{ss:EM}, we consider estimation of these multiple dependent graphs using a one-step procedure and a method based on the EM algorithm.

\subsection{One-step method}\label{ss:onestep}

The idea behind our one-step method is to generate a good initial estimate for $\Sigma$ and then obtain estimates for $\Omega$ by one-step optimization.
Because $\text{var}(Z) = \text{cov}(Y_l, Y_m)$, for any $ m \neq l$, it is natural to use the covariance matrix $\Sigma_{Y(l,m)}$ between all pairs of $Y_l$ and $Y_m$ to estimate $\Sigma_0$ by
\begin{align}\label{eq:Sigma0orig}
	\hat{\Sigma}_{0} = \frac{1}{K(K-1)}\sum_{m \neq l} \hat{\Sigma}_{Y(m,l)}
	=\frac{1}{K(K-1)n}\sum_{m \neq l}\sum_{i = 1}^n \big( y_{m,i} y_{l,i}^\T \big).
\end{align}
Using the fact that $\text{var}(X_k) = \text{var}(Y_k) - \text{var}(Z)$, we can then obtain an estimate for $\Sigma_{k}$ as
\begin{align}\label{eq:Sigmahatk}
	\hat{\Sigma}_{k}
	= \hat{\Sigma}_{Y(k,k)} - \hat{\Sigma}_{0}
	= \frac{1}{n} \sum_{i = 1}^n \big( y_{k,i}y_{k,i}^\T \big)
	- \hat{\Sigma}_{0}.
\end{align}
Although $\hat{\Sigma}_{k}$ is symmetric, it may not be positive semidefinite, but this can be ensured using projection \citep{Xu2012}. { For any symmetric matrix  $\hat{\Sigma}_{k}$ $(k = 0, \ldots, K)$, the positive-semidefinite projection is 
\begin{align} \label{eq:Projection}
	\hat{\Sigma}_{k}^{\prime} = \argmin_{\Sigma \succeq 0 }
	 \| \Sigma - \hat{\Sigma}_{k} \|_\infty.
\end{align} }
Lastly, we estimate
 $ \Omega $ by minimizing $K + 1$ separate functions, 
\begin{align} \label{eq:onestep}
	\mathcal{W}_{k}(\Omega_k) 
	= \tr(\hat{\Sigma}_{k}^{\prime} \Omega_{k}) - \log \det(\Omega_{k})
	+ \lambda |\bO_k^- |_1 \; \quad (k = 0, \ldots, K),
\end{align}
where $\lambda = \lambda_2$ when $k = 0$, and  $\lambda = \lambda_1$ otherwise. The minimization of \eqref{eq:onestep} can be solved efficiently by algorithms such as the graphical lasso \citep{Friedman2008} or by the quadratic inverse covariance algorithm  \citep{Hsieh2011}. We name this approach as the one-step method and later compare its performance with the EM method defined next.

\subsection{Graphical EM method}\label{ss:EM}

The one-step method provides an estimate of $\bO$. In the spirit of the classic EM algorithm \citep{Dempster1977}, this estimate of $ \bO$ can be used to obtain a better estimate of $\Sigma$, which in turn can be used to obtain a better estimate of $\bO$. This procedure is iterated until the estimates of $\bO$ converge, leading to a graphical EM algorithm, described below.  

{First, we rewrite the sampling model as}
\begin{align*} 
	\begin{pmatrix}
 		Z  \\
 		Y_1 - Z\\
 		\vdots  \\
		Y_K - Z
	\end{pmatrix}  & \sim \mathcal{N}  \left\{ 
		\begin{pmatrix}
 		0  \\
 		0 \\
 		\vdots  \\
		0 
	\end{pmatrix}, 	
  	\begin{pmatrix}
 		\Sigma_0            & 0 &
 		\ldots 				& 0 \\
		0   & \; \Sigma_1  & 
		\ldots				& 0 \\
 		\vdots 				& \vdots & 
 		\vdots				& \vdots \\
		0   & \;  0 		&
		\ldots				& \Sigma_K 
	\end{pmatrix}
	\right\},
\end{align*}
and the {log-likelihood given $Y=y$ and $Z = z = (z_1, \ldots, z_n)^\T$ as}
\begin{align}\label{eq:full_log_likelihood}
	\mathcal{L}(\bO; y, z )
	&\propto \log \det(\bO_0) - \tr\big( \bO_0 
	\sum_{i = 1}^n z_i z_i^\T /n  \big) \notag\\
	& + \sum_{k =1}^K \Big[ \log \det(\bO_k) 
	- \tr \big\{ \bO_k \sum_{i = 1}^n (y_{k, i} - z_i)(y_{k, i} - z_i)^\T /n  \big\} \Big].
\end{align}
Expression \eqref{eq:full_log_likelihood} cannot be calculated directly because $z_i$ and $z_i z_i^\T$ are unobserved. However, we can replace them with their expected values conditional on $\bO$ and $y$, and develop the EM algorithm with the following steps: 
\begin{itemize}
	\item[E step] Update the expectation of the log-likelihood conditional on $\bO$ using 		
	\begin{align}\hspace{-0.3in} \editmath
	\mathcal{Q}( \bO ; \bO^{(t)}) = & \editmath E_{Z|\bO^{(t)}}\{ \mathcal{L}(\bO; y, z ) 
	\}  \notag \\
	\editmath \propto & \editmath \log \det( \bO_{0}) 
            - \tr\Big\{ \bO_{0} \,
            E_{Z|\bO^{(t)}} \Big( \sum_{i=1}^n z_i z_i^\T/n \Big)
             \Big\} + \sum_{k=1}^K \log \det(\bO_{k}) \notag \\
            & \editmath - \sum_{k=1}^K \tr\Big[ \bO_{k} 
              E_{Z|\bO^{(t)}} \Big\{ \sum_{i=1}^n (y_{k,i}-z_i)(y_{k,i}-z_i )^\T / n  
              \Big\} \Big] \notag \\
			\editmath = & \editmath \sum_{k = 0}^{K}\Big\{ \log  \det(\bO_{k}) 
			-\tr \big(\bO_{k} \dot{\Sigma}_{k}^{(t)} \big)
			\Big\}. \notag 
    \end{align} 
	\item[M step] {Update $\bO$ that maximizes}
		\begin{align} \label{eq:Mstep}
			\editmath \bO^{(t + 1)} = \argmin_{\bO \succ 0} 
			- \mathcal{Q}(\bO;\bO^{(t)}) 
			+ \lambda_1 \sum_{k=1}^K | \bO_k^- |_1 
			+ \lambda_2 |\bO_0^- |_1,   \quad
    \end{align}
\end{itemize}
\edit{where $\bO^{(t)}$ denotes the estimates from the $t$th iteration, $E_{Z|\bO^{(t)}}(\cdot)$  denotes the conditional expectation with respect to $Z$ given $\bO^{(t)}$, and}
\begin{subequations}\label{eq:all1}
  \vspace{5pt}
	\begin{align}\editmath{}
		\dot{\Sigma}_{k}^{(t)} 
		= & \editmath{} E_{Z|\bO^{(t)}} \Big\{ \sum_{i=1}^n (y_{k,i}-z_i)(y_{k,i}-z_i)^\T / n  \Big\}  \nonumber \\
		\editmath{} 
          = & \editmath{} \ddot{\Sigma}_{Y(k,k)} 
		- \sum_{l=1}^K \Big( \ddot{\Sigma}_{Y(k,l)}\bO_{l}^{(t)} \Big)
		(A^{(t)})^{-1} - (A^{(t)})^{-1}
		\sum_{l=1}^K \Big( \bO_{l}^{(t)}\ddot{\Sigma}_{Y(l,k)} \Big) \notag \\
		&\editmath{}  + (A^{(t)})^{-1} \sum_{l,k=1}^K \Big(\bO_{l}^{(t)} 
		\ddot{\Sigma}_{Y(l,k)} \bO_{k}^{(t)} \Big)
		(A^{(t)})^{-1} + (A^{(t)})^{-1} \quad   (k =1, \ldots, K), \label{eq:sigmak}\\
	\editmath{}	\dot{\Sigma}_0^{(t)} =
		& \editmath{} \sum_{i=1}^n E_{Z|\bO^{(t)}} \big( z_i z_i^\T/n \big)
		=\editmath{} (A^{(t)})^{-1} +(A^{(t)})^{-1} \sum_{l,k=1}^K
		\big( \bO_{l}^{(t)} \ddot{\Sigma}_{Y(l,k)} \bO_{k}^{(t)} \big)
		(A^{(t)})^{-1} \label{eq:sigma0},
  \end{align}
\end{subequations}
where $\ddot{\Sigma}_Y= \hat{\Sigma}_Y$ is an estimator for $\Sigma_Y^*$, the true covariance matrix of $Y$.
Therefore, problem \eqref{eq:Mstep} is decomposed into $K + 1$ separate optimization problems:
\begin{align} \label{eq:glasso}
	\bO_{k}^{(t+1)} =  \argmin_{\bO_k \succ 0}
	\Big\{ \tr \big(\bO_{k} \dot{\Sigma}_{k}^{(t)} \big) - \log \det(\bO_{k})
	+ \lambda | \bO_k^- |_1 \Big\} \quad (k = 0, \ldots, K),
\end{align}
where $\lambda = \lambda_2$ when $k = 0$, and $\lambda = \lambda_1$ otherwise. We then can use the graphical lasso \citep{Friedman2008} to solve \eqref{eq:glasso}.

\newpage
We summarize the proposed EM method in the following steps:
\smallskip
\noindent
\begin{algorithm}
\caption{The graphical EM algorithm} \label{al1}
\vspace*{-20pt}
\begin{tabbing}
   \enspace (Initial value).
		Initialize $\hat{\Sigma}_0^\prime $ and  $\hat{\Sigma}_k^\prime $ $(k = 1,\ldots,K)$ using \eqref{eq:SigmaY}, \eqref{eq:Sigma0orig}--\eqref{eq:Projection}.\\
   \qquad (Updating rule: the M step).
		Update $\bO_{k}$ $(k = 0,\ldots,K)$  by \eqref{eq:glasso} using the graphical lasso.\\
   \qquad (Updating rule: the E step).
		Update  $\dot{\Sigma}_{k}$ using \eqref{eq:sigmak} and \eqref{eq:sigma0}. \\\
  \qquad Iterate the M and E steps until convergence.\\
\enspace Output $\hat{\bO}_k$  $(k = 0, \ldots, K)$.
\end{tabbing}
\end{algorithm}
\vspace*{-20pt}

\smallskip
The next proposition demonstrates convergence of our graphical EM algorithm.
\begin{prop}
\label{prop:pro1}
With any given $n$, $p$, $\lambda_1 >0$, and $\lambda_2 > 0$, the graphical EM algorithm solving \eqref{eq:Penlikelihood}
has the following properties:
   \begin{property}
    the penalized log-likelihood in \eqref{eq:Penlikelihood} is bounded above;
    \property for each iteration, the penalized log-likelihood is non-decreasing;
    \property for a prespecified threshold $\delta$, after a finite number of steps, the objective function in \eqref{eq:Penlikelihood} converges in the sense that  
$
	\big| \mathcal{P}(\bO^{(t+1)}; y) - \mathcal{P}(\bO^{(t)}; y) \big| < \delta.
$
   \end{property}
\end{prop}

\subsection{Model selection}\label{s:selection}
We consider two options for selecting the tuning parameter $\lambda = (\lambda_1, \lambda_2)$, minimization of the extended Bayesian information criterion \citep{Chen2008}, and cross-validation. The extended Bayesian information criterion is quick to compute and takes into account both goodness of fit and model complexity. Cross-validation, by contrast, is more computationally demanding and focuses on the predictive power of the model.

In our model, we define the extended Bayesian information criterion
\begin{align*} \editmath{}
  \textsc{bic}_{\gamma}({\lambda})= -2 \mathcal{L} (\{ \hat{{\Omega}}_k\}_{k = 0}^K; y)
    +\nu({\lambda})\log n + 2 \gamma \log \tau\{\nu({\lambda}) \}, 
\end{align*}
\edit{where $\{ \hat{{\Omega}}_k\}_{k = 0}^K$ are the estimates with the tuning parameter set at ${\lambda}$, $\mathcal{L}(\cdot)$ is the log-likelihood function, the degrees of freedom $\nu({\lambda})$ is the sum of the number of non-zero off-diagonal elements on $\{\hat{{\Omega}}_k\}_{k = 0}^K$, and $\tau\{\nu({\lambda}) \}$ is the number of models with size $\nu({\lambda})$, which equals $Kp(p-1)/2$ choose $\nu({\lambda})$.} This criterion is indexed by a parameter $\gamma \in [0, 1]$. The tuning parameter ${\lambda}$ is selected as $\hat{{\lambda}} = \argmin \{ \textsc{bic}_{\gamma}({\lambda}): \lambda_1, \lambda_2 \in (0, \infty) \}$.

In describing the cross-validation procedure, we define the predictive negative log-likelihood function as 
$
   \mathcal{F}({\Sigma}, {\Omega}) = \tr ({\Sigma}{\Omega}) - \log \det({\Omega}).
$
To select ${\lambda}$ using cross-validation, we randomly split the dataset equally into $J$ groups, and denote the sample covariance matrix from the $j$th group as $\hat{\Sigma}_{Y(j,\lambda)}$  and the precision matrix estimated from the remaining groups as $\hat{{\Omega}}_{Y(-j, {\lambda} )}$. Then we choose
\[
  \hat{{\lambda}} = \argmin_{\lambda}
  \Big\{
    \sum_{j = 1}^J \mathcal{F}({\Sigma}_{Y(j,\lambda)}, \hat{{\Omega}}_{Y(-j, {\lambda})}): \  \lambda_1, \lambda_2 \in (0, \infty)
  \Big\}.
\]
The performance of these two selection methods is \edit{reported} in Section \ref{s:simul}.

\section{Asymptotic properties}
\label{s:theory}

We introduce some notations and the regularity conditions. Let $\{ \bO^{*}_k\}_{k = 0}^K$ be the true precision matrices with $\bO^{*}_{k} = ( \omega^{*}_{k(i,j)} )$, $T_k = \{(i, j): i \neq j, \omega^{*}_{k(i, j)} \neq 0 \}$ be the index set corresponding to the nonzero off-diagonal entries in $\bO^{*}_{k}$, $q_k = | T_k |$ be the cardinality of $T_k$, and $q = \sum_{k=0}^{K} q_k$. Let $\{ \Sigma^{*}_k \}_{k = 0}^K$  be the true covariance matrices for $Z$ and $\{ X_k \}_{k = 1}^K$, and $ \Sigma^{*}_Y = \{ \Sigma^{*}_{Y(l,m)} \}_{1 \leq l,m \leq K}$ be the true covariance matrices for $Y$. We assume that the following regularity conditions hold.

\textit{Condition} 1. There exist constants $\tau_1$ and $\tau_2$ such that $0 < \tau_1<\phi_{\text{min}}(\bO^{*}_{k})\le \phi_{\text{max}}(\bO^{*}_{k})< \tau_2 < \infty \quad (k= 0, \ldots, K)$.

\textit{Condition} 2. There exist constants $a$ and $b$ such that  $$ a \{(\log p) /n \}^{1/2} \leq \lambda_j \leq  b \{ (1 + p/q) (\log p) / n\}^{1/2} \quad (j = 1, 2).$$

Condition 1 bounds the eigenvalues of $\bO^{*}_{k}$, thereby guaranteeing the existence of its inverse. Condition 2 is needed to facilitate the proof of consistency. The following theorems discuss estimation consistency and selection sparsistency of our methods.

\begin{theorem}[Consistency of the one-step method]
\label{thm:Thm1}
Under Conditions 1 and 2, if $(p + q) (\log p) /n = o(1)$, then the solution $\{\hat{\bO}_k^{\mone}\}_{k = 0}^K$ of the one-step method satisfies
\vspace{4pt}
\begin{equation*}
	\sum\limits_{k = 0}^{K}  
	\big\| \hat{\bO}_{k}^{\mone} - \bO^{*}_{k} \big\|_F 
	= O_p\Big[ \Big\{ \frac{(p+q) \log p}{n} \Big\}^{1/2} \Big].
\end{equation*}
\end{theorem}

We next present a corollary of Theorem \ref{thm:Thm1} which gives a good estimator of $\Sigma^{*}_Y$.
\begin{corollary}
\label{cor:cor2}
Under the assumptions of Theorem \ref{thm:Thm1} and $\hat{\bO}_k^{\mone} (k = 0, \ldots, K)$ being the one-step solution, $\check{\Sigma}_k = (\hat{\bO}_k^{\mone} )^{-1}$ satisfies
 \vspace{4pt}
	\begin{equation*}
		\big\| \check{\Sigma}_k - \Sigma_k^* \big\|_F 
		= O_p\Big[ \Big\{ \frac{(p+q) \log p}{n} \Big\}^{1/2} \Big].
	\end{equation*}
\end{corollary}
To study our EM estimator, we need an estimator for $\Sigma^{*}_Y$ that satisfies  the following condition.

\textit{Condition} 3. We assume there exists an estimator $\tilde{\Sigma}_Y$ such that
\begin{equation*}
	\| \tilde{\Sigma}_Y - \Sigma^*_Y \|_F 
	= O_p\Big[ \Big\{ \frac{(p+q) \log p}{n} \Big\}^{1/2} \Big].
 \end{equation*}

The rate in Condition 3 is required to control the convergence rate of the E-step estimating ${\Sigma}_{k}^{*}$ 
 and thus the consistency of the estimate from the EM method. Under the conditions in Theorem \ref{thm:Thm1}, we can use the one-step estimator $\hat{\bO}_k^{\mone}$ $(k= 0, \ldots, K)$ to obtain $\tilde{\Sigma}_Y = J \otimes \hat{\bO}_0^{-1} + \{_d \hat{\bO}_k^{-1} \} $,  where $\{_d \cdot \}$ is a block diagonal matrix. The resulting  $\tilde{\Sigma}_Y$ satisfies Condition 3 by Corollary \ref{cor:cor2}.

\begin{theorem}[Consistency of the EM method]
\label{thm:Thm2}
  If Conditions 1-3 hold and $(p + q) (\log p) /n = o(1)$, then after a finite number of iterations, the solution $\{\hat{\bO}_k^{\mEM} \}_{k = 0}^K $ of the EM method satisfies
\begin{equation*}
	\sum\limits_{k = 0}^{K}  \big\| \hat{\bO}_k^{\mEM} - \bO^{*}_k \big\|_F 
	= O_p\Big[ \Big\{ \frac{(p+q) \log p}{n} \Big\}^{1/2} \Big].
\end{equation*}
\end{theorem}

\begin{theorem}[Sparsistency of the one-step method]
\label{thm:Thm3}
{Under the assumptions of Theorem \ref{thm:Thm1}, {if we further assume that the one-step solution $\{\hat{\bO}_k^{\mone} \}_{k = 0}^K$ satisfies  $\sum_{k = 0}^{K} \| \hat{\bO}_k^{\mone} -  \bO^{*}_k \| = O_p(\eta_n) $  for a sequence $\eta_n \rightarrow 0$, and $ \{ (\log p) /n \}^{1/2}+ \eta_n = O(\lambda_1) = O(\lambda_2)$, then with probability tending to 1, $\hat{\omega}_{k(i,j)}^{\mone} = 0$ for all $(i,j) \in T^c_k$ $ (k = 0,\ldots,K)$.} }
\end{theorem}

For sparsistency we require a lower bound on the rates of $\lambda_1$ and $\lambda_2$, but for consistency, we need an upper bound for $\lambda_1$ and $\lambda_2$ to control the biases. In order to have consistency and sparsistency simultaneously, we need the bounds to be compatible, that is, we need $ \{ (\log p) /n \}^{1/2}+ \eta_n = O(\lambda_1, \lambda_2) = \{(1 + p/q) \log p/n \}^{1/2}$. From the inequalities $\| W  \|^2_F/p \leq \| W \|^2 \leq \| W \|^2_F$, there are two extreme scenarios describing the rate of $\eta_n$, as discussed in \citet{Lam2009}. In the worst case, where $\sum_{k = 0}^{K} \| \hat{\bO}_{k} - \bO^{*}_{k} \|$  has the same rate as $\sum_{k = 0}^{K} \| \hat{\bO}_{k} - \bO^{*}_{k}\|_F$, we achieve both consistency and sparsistency only when $q = O(1)$. In the most optimistic case, where $\sum_{k = 0}^{K} \| \hat{\bO}_{k} - \bO^{*}_{k} \|^2 = \sum_{k = 0}^{K}  \| \hat{\bO}_{k} - \bO^{*}_{k} \|_F^2/p$, we have $\eta_n^2 = (1 + q/p) \log p/n$, and the compatibility of the bounds requires $q = O(p)$.

\begin{theorem}[Sparsistency of the EM method]
\label{thm:Thm4}
{Under the assumptions of Theorem \ref{thm:Thm2}, {if we further assume the EM solution $\{\hat{\bO}_k^{\mEM} \}_{k = 0}^K$ satisfies  $\sum_{k = 0}^{K} \| \hat{\bO}_k^{\mEM} -  \bO^{*}_{k} \| = O_p(\zeta_n)$ for a sequence $\zeta_n \rightarrow 0$, and if $ \{(p + q) (\log p) /n \}^{1/2}  + \zeta_n = O(\lambda_1) = O(\lambda_2)$, then  } with probability tending to 1, $\hat{\omega}_{k(i,j)}^{\mEM} = 0$ for all $(i,j) \in T^c_k$ $(k = 0,\ldots,K).$}
\end{theorem}

Similar to the discussion above for the EM algorithm, we obtain both consistency and sparsistency when $q= O(1)$. See the Supplementary Material.

\section{Simulation}
\label{s:simul}

\subsection{Simulating category-specific and systemic networks}

We assessed the performance of the one-step and EM methods by applying them to simulated data generated by two types of synthetic networks: a chain network and a nearest-neighbor network as shown in Fig. \ref{fig:networks}. Twelve simulation settings were considered. These varied the base architecture of the category-specific network, the degree to which the actual structure could deviate from this base architecture, and the number of outcome variables.

Under each of the 12 simulation conditions, samples were independently and identically distributed, with systemic outcomes generated as $Z_i\sim \mathcal{N}(0, \Omega_0^{-1})$, category-specific outcomes as $X_{k,i}\sim \mathcal{N}(0, \bO_k^{-1})$, and observed outcomes as $y_{k,i} = x_{k,i} + z_i$, for $K=4$, and $n = 300$. The following architectures were considered for the five networks $\{ \bO_{k}\}_{k = 0}^4$:

(I) the $K$ category-specific networks are chain-networks and the systemic network is a nearest-neighbor network with the number of neighbors $m = 5$ and $25$ for $p = 100$ and $1000$;

(II) the $K$ category-specific networks and the systemic network are all nearest-neighbor networks with $m = 5$ and $25$ for $p = 100$ and $1000$ respectively.

Chain and nearest-neighbor networks were generated using the algorithms in \citet{Fan2009} and \citet{Li2006}. The structures of network (I) are shown in Fig. \ref{fig:networks}(a) and (d). Simulated networks were allowed to deviate from their base architectures by a specified degree $\rho$, through a random addition of edges following the method of \citet{Guo2011}. Specifically, for each $\bO_{k}$  $(k = 0,1,\dots, K)$ generated above, a symmetric pair of zero elements is randomly selected and replaced with a value generated uniformly from $[-1, -0.5] \cup [0.5, 1]$. We repeat this procedure $\rho T$ times, with $T$ being the number of links in the initial structure, and $\rho\in\{0, 0.2,1\}$. 

\begin{figure}[H]
	\begin{center}
      	\includegraphics[width = \textwidth]{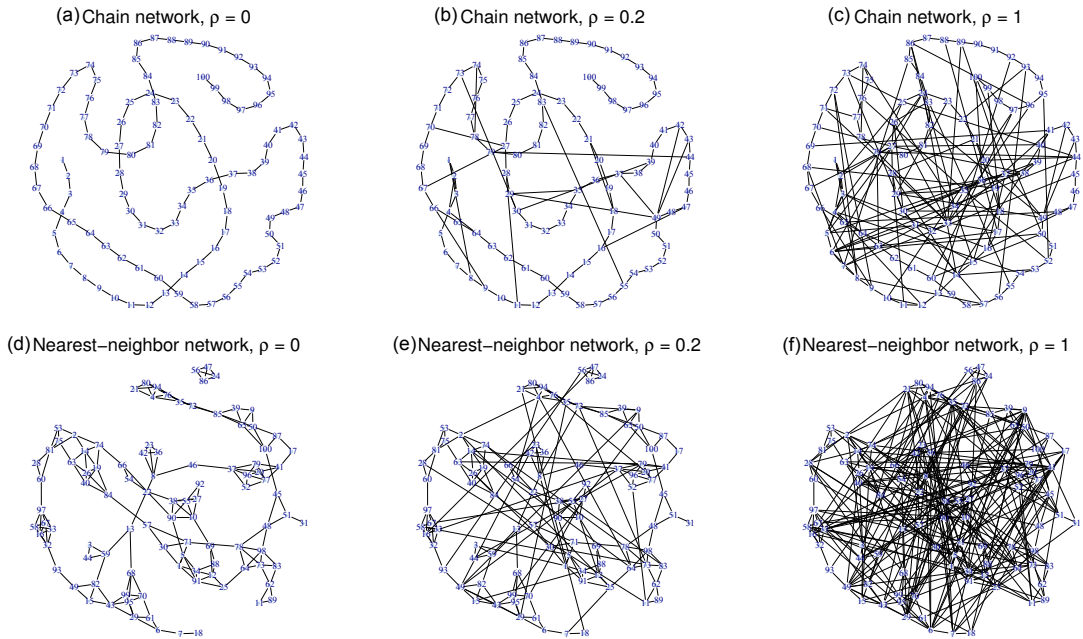}
	\end{center}
	\caption{Network topologies generated in the simulations. Top row (a-c) shows chain networks with noise ratios $\rho = 0$, $0.2$, and $1$. Bottom row (d-f) shows nearest-neighbor networks with $\rho = 0$, $0.2$, and $1$.}
	\label{fig:networks}
\end{figure}

We compared the performance of the one-step and EM methods by examining the average false positive rate, average false negative rate, average Hamming distance, average entropy loss
\begin{equation*}
	\text{EL} = \frac{1}{K + 1} \sum_{k = 0}^K 
	\Big\{ \tr \big( \bO_k^{*-1}\hat{\bO}_k \big) 
	- \log \det \big( \bO_k^{*-1}\hat{\bO}_k \big)   \Big\} - p \,,
\end{equation*}
and average Frobenius loss
\begin{equation*}
	\text{FL} =  \frac{1}{K + 1}\sum_{k = 0}^K \frac{\| \bO_k^* - \hat{\bO}_k\|_F^2}
	{ \| \bO_k^* \|_F^2 } \,.
\end{equation*}
We also examined receiver operating curves for the two methods. 

\subsection{Estimation of  category-specific $\bO_k$ and systemic networks $\bO_0$}

As shown in Fig. \ref{fig:toy}, existing methods are designed to estimate the aggregate networks $\bO_{Y_k}$ instead of category-specific $\bO_k$ and systemic $\bO_0$ networks. In this subsection, we focus only on our proposed one-step and EM methods.

Results of the simulations are reported in Table \ref{tab:simresults1}. Summary statistics are based on $50$ replicate trials under each of the 12 conditions, and given for model fitting under both extended Bayesian information criterion with $\gamma = 0.1$ and under cross-validation criteria. In general, the one-step method under either model selection criteria resulted in higher values of entropy loss, Frobenius loss, false positive rates and Hamming distance. For both methods, cross-validation tended to choose models with more false positive links but fewer false negative links, leading to a denser graph. {For model selection, a rule of thumb is to use cross-validation when $p > 500$, and to use the extended Bayesian information criterion  otherwise.}

Receiver operating curves for the one-step and EM methods are plotted in Fig. \ref{fig:roc}; each is based on 100 replications with the constraint $\lambda_1 = \lambda_2$. Under all settings, the EM method outperforms the one-step method, yielding greater improvements as the structures become more complicated.

\begin{figure}[htb!]\footnotesize
	\includegraphics[width = 5.5in]{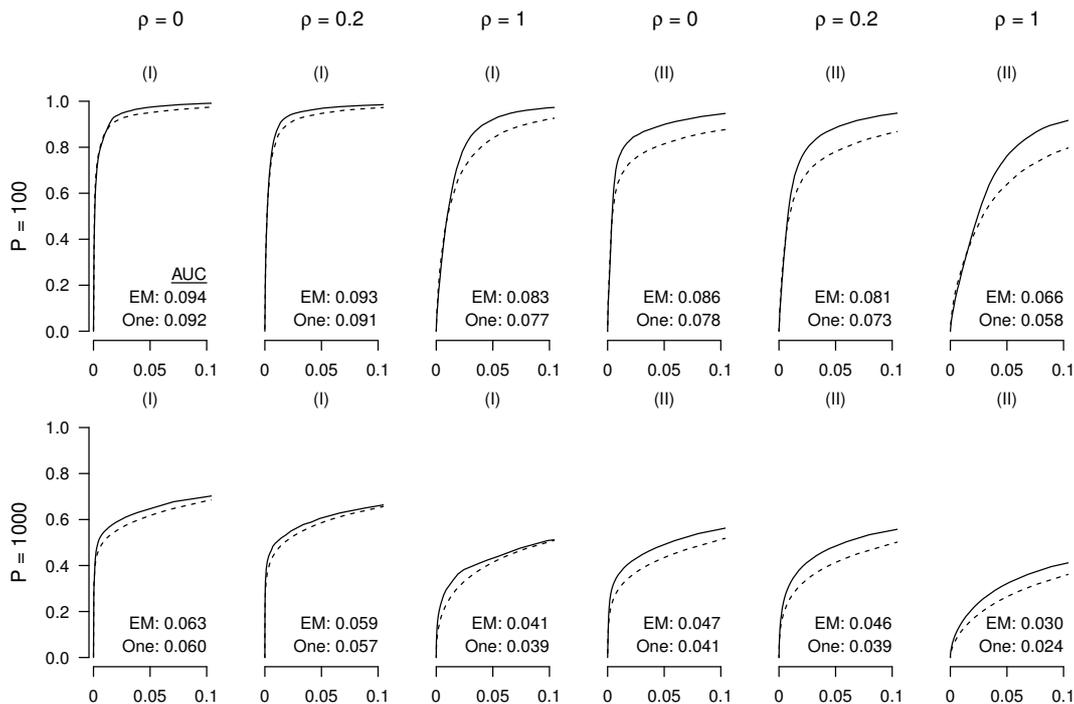}
	\caption{Receiver operating characteristic  curves assessing power and discrimination of graphical inference under different simulation settings. Each panel reports performance of the EM method (solid line) and the one-step method (dashed line), plotting true positive rates (y-axis) against false positive rates (x-axis) for a given noise ratio $\rho$, \edit{network base architecture I or II,} sample size $n = 300$, number of neighbors $m = 5$ and $25$ for $p = 100$ and $1000$ respectively. The numbers in each panel represent the areas under the curve for the two methods.}
	\label{fig:roc}
\end{figure}

\begin{table}[htb!]\footnotesize
\caption{\label{tab:simresults1} Summary statistics reporting performance of the EM and one-step methods inferring graph structure for different networks. The numbers before and after the slash are the results based on the extended Bayesian information criterion and cross-validation, respectively.}
{
\begin{center}
  \scalebox{1}{
  \begin{tabular}{c c c c c c c c c}
  \hline
  \emph{p} & \specialcell{Network\\architecture}  & $\rho$ & Method & EL 
  & FL & FP$(\%)$ & FN$(\%)$  & HD $(\%)$  \\
  \hline\multirow{12}{*}{$100 $ }
    &\multirow{6}{*}{(I)}
    & 0    & One-step & 12.1/10.0           & 0.24/0.16           & 5.5/20.9            &
    4.2/0.9           & 5.5/20.4  \\
    &&   0 & EM       & 6.7/4.7   & 0.15/0.08 & 4.2/15.8  &
    3.4/0.6 & 4.2/{15.4} \\
    && 0.2 & One-step & 10.6/8.6            & 0.22/0.15           & 5.4/19.4            &
    3.7/0.9           & 5.3/18.8 \\
    && 0.2 &  EM      & {6.4}/{4.8}   & {0.15}/{0.09} & {4.9}/{14.3}  &
    {3.5}/{0.6} & {4.8}/ {14.0} \\
    && 1   & One-step & 12.6/9.9            & 0.24/0.17           & 7.3/23.3            &
    9.5/2.9           & 7.5/22.3           \\
    && 1  &  EM       & {8.3}/{6.0}   & {0.17}/{0.11} & {6.7}/{15.3}  &
    {5.3}/{1.6} & {6.6}/{14.6} \\
    \cline{2-9}
    &\multirow{6}{*}{(II)}
    & 0   & One-step  & 12.1/9.6           & 0.27/0.19            & {3.4}/19.6        &
    22.0/7.6          & 4.1/19.1 \\
    &&0   & EM        & {7,9}/{6.0}  & {0.20}/{0.14}  & 3.8/{13.5}        &
    {12.4}/{4.2}&{4.1}\ {13.4}\\
    &&0.2 & One-step  & 12.5/9.7           & 0.26/0.18            & 4.6/21.0             &
    23.0/7.8          & 5.5/20.4 \\
    &&0.2 & EM        & {8.7}/{6.1}  & {0.19}/{0.12}  & {4.5}/{15.2}   &
    {14.1}/{3.2}&{5.0}/{14.6} \\
    && 1  & One-step  & 16.3/12.6          & 0.27/0.17            & 8.7/30.4             &
    24.0/8.8          & 9.9/28.7 \\
    && 1  & EM        & {11.3}/{7.6} & {0.20}/{0.11}  & {8.1}/{22.9}   &
    {13.7}/{2.7}&{8.6}/{21.4} \\
    \hline\multirow{12}{*}{$1000$}
	&\multirow{6}{*}{(I)}
    &  0  & One-step   & 276.7/240.6          & 0.44/0.36           & 0.6/5.5            &
    52.1/{34.6}     &0.9/5.6   \\
    && 0  & EM         & {120.3}/{94.9} & {0.22}/{0.16} & {0.5}/{2.5}  &
    {48.9}/35.7     &{0.8}/{2.7} \\
    && 0.2& One-step   & 201.5/162.3          & 0.35/0.27           & 0.2/5.0            &
    64.3/{37.9}     & 0.6/5.3 \\
    && 0.2& EM         & {117.7}/{88.5} & {0.19}/{0.13} & {0.2}/{2.2}  &
    { 57.8}/39.8    &{0.6}/ {2.5}  \\
    && 1 & One-step    & 171.6/146.0          & 0.28/0.22           & 0.0/5.3            &
    100/{54.1}          & 1.2/5.9\\
    && 1 & EM          & {147.1}/{108.1}& {0.20}/{0.14} & {0.0}/{2.3}  &
    {99.2}/56.5    &{1.2}/{2.9}  \\
    \cline{2-9}
    &\multirow{6}{*}{(II)}
    & 0   & One-step   & 301.0/234.4          & 0.43/0.33           &  {0.1}/6.7      &
    83.5/{53.7}     & 2.0/7.7\\
    && 0  & EM         & {206.7}/{160.9}& {0.29}/{0.23} &  0.2/{2.6}      &
    {73.8}/56.4     &{1.9}/{3.8}  \\
    && 0.2& One-step   & 349.8/257.5          & 0.44/0.31           & {0.1}/8.4       &
    89.2/{52.9}     & 2.5/9.6 \\
    && 0.2& EM         & {275.0}/{190.8}& {0.32}/{0.23} & 0.2/{3.9}       &
    {82.7}/53.8     &{2.4}/{5.2} \\
    && 1& One-step     & 325.4/268.8          & 0.41/0.29           & 0.0/10.1           &
    99.9/{64.3}     & 4.4/12.5\\
    &&1 & EM           & {301.6}/{232.6}& {0.31}/{0.23} & { 0.0}/{4.8} &
    {99.8}/68.2     & {4.4}/ {7.6}\\
   \hline
     \end{tabular}}
  \end{center}}
\end{table}

\subsection{Estimation of aggregate networks $\bO_{Y_k}$}

Although our goal is to estimate the two network layers, we can also use our estimators of  $\bO_k$ $(k = 0, \ldots, K)$ to estimate the aggregate network $\bO_{Y_k} = (\bO_k^{-1} + \bO_0^{-1} )^{-1} $ as a derived statistic. Doing so allows us to compare our method with methods that aim to estimate the aggregate network $\bO_{Y_k}$, these methods otherwise being incomparable.

We compared the performance of the EM method with two existing  \edit{single-level} methods for estimating multiple graphs: the hierarchical penalized likelihood method of \citet{Guo2011}, and the joint graphical lasso of \citet{Danaher2013}. \edit{As shown by simulation results reported in the Supplementary Material, these two single-level methods tended to give similar, sparse estimates that were very different from the true aggregate graph. The true aggregate graph tended to be highly connected, as illustrated in Fig \ref{fig:toy}, and under most settings was much better estimated by the EM.
An exception was setting (II) with $\rho = 0$ and $0.2$, where $\bO_{Y_k}$ is relatively sparse, and where the best performance came from the method of \citet{Guo2011}. Sparsity in $\bO_{Y_k}$ arises under this setting because when $\bO_k$ and $\bO_0$ are chain networks $\bO_{Y_k}$ has a strong banding structure, with large absolute values within the band and small absolute values outside.}


\section{Application to gene expression data in mouse}

To demonstrate the potential utility of our approach, we apply the EM method to mouse genomics data from \citet{Dobrin2009} and \citet{Crowley2014}. In each case, we aim to infer systemic and category-specific gene co-expression networks from transcript abundance as measured by microarrays. In describing our inference on these datasets we find it helpful to distinguish two interpretations of a network: the potential network is the network of biologically possible interactions in the type of system under study; the induced network is the subgraph of the potential network that could be inferred in the population sampled by the study. The induced network is therefore a statistical, not physical, phenomenon, and describes the dependence structure induced by the interventions, or perturbations, applied to the system.

A simple example is the relationship between caloric intake, sex, and body weight.  Body weight is influenced by both the state of being male or female and the degree of calorie consumption; these relations constitute edges in the potential network. Yet in a population where caloric intake varies but where individuals are exclusively male, the effect of sex is undefined and the corresponding edges relating sex to body weight are undetectable; these edges are therefore absent in the induced network. More generally, the induced network for a system is defined both by the potential network and the intervention applied to it: two populations of mice could have the same potential network, but when subject to different interventions their induced networks could differ. Conversely, when estimating the dependence structure of variables arising from population data, the degree to which the induced network reflects the potential network is a function of the underlying conditions being varied and interventions at work.

The \citet{Dobrin2009} dataset comprises expression measurements for 23,698 transcripts on 
301 male mice in adipose, liver, brain and muscle tissues. These mice arose from an \FTwo cross between two contrasting inbred founder strains, one with normal body weight physiology and the other with a heritable tendency for rapid weight-gain. In a cross of this type, the analyzed offspring constitutes an independent and identically distributed sample of individuals who are genetically distinct and have effectively  been subject to a randomized allocation of normal and weight-inducing DNA variants, or alleles, at multiple locations along its genome. As a result of this allocation, gene expression networks inferred on such a population would be expected to emphasize more strongly those subgraphs of the underlying potential network that are related to body weight. Moreover, since the intervention alters a property affecting the entire individual, we might expect it to exert at least some of its effect systemically, that is, globally across all tissues in each individual.

Using a subset of the data, we inferred the dependence structure of gene co-expression among three groups of well-annotated genes in brain and liver: an obesity-related gene set, an imprinting-related gene set, and an extracellular matrix, i.e., the ECM-related gene set. These groups were chosen based on criteria independent of our analysis and represent three groups whose respective effects would be exaggerated under very different interventions. 
%
%
The tissue-specific and systemic networks inferred by our EM method are shown in Fig. \ref{fig:f2smallnet}. Each node represents a gene, and the darkness of an edge represents the magnitude of the associated partial correlation. The systemic network in Fig. \ref{fig:f2smallnet}(c) includes edges on the \textit{Aif1} obesity-related pathway only, which is consistent with the \FTwo exhibiting a dependence structure induced primarily by an obesity-related genetic intervention that acts systemically.  The category-specific networks in Fig. \ref{fig:f2smallnet}(a) and (b) still include part of the \textit{Aif1} pathway, suggesting that variation in this pathway tracks variation at  both the systemic and tissue-specific level; in other ways their dependence structures differ, with, for instance, \textit{Aif1} and \textit{Rgl2} being linked in the brain but not in the liver. The original analysis of \citet{Dobrin2009} used a correlation network approach, whereby unconditional correlations with statistical significance above a predefined threshold were declared as edges; that analysis also supported a role for \textit{Aif1} in tissue-to-tissue co-expression.

\begin{figure}[h]
  \begin{center}
	\includegraphics[width=0.65\textwidth]{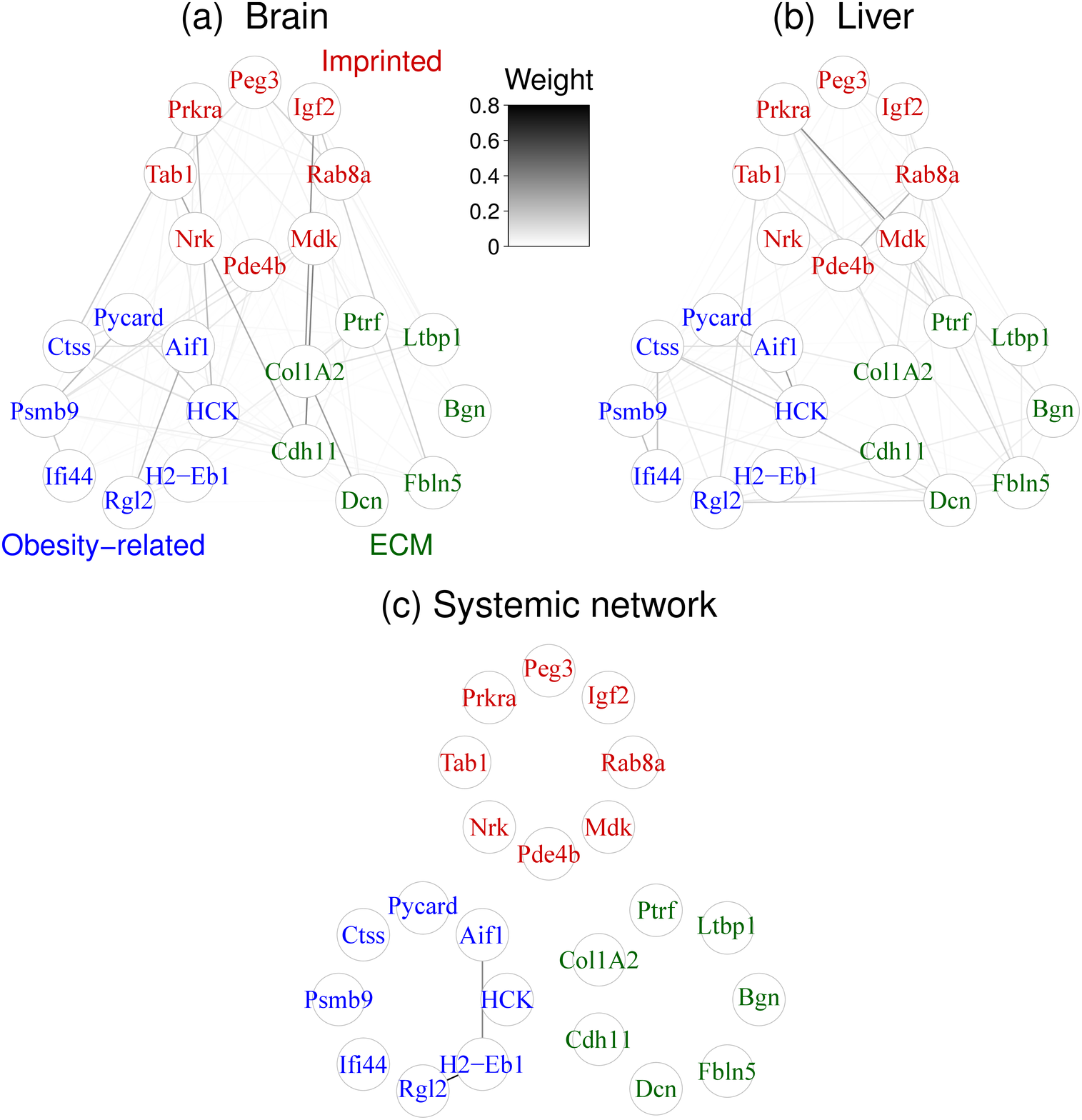}
	\caption{\label{fig:f2smallnet}Topology of gene co-expression networks inferred by the EM method for the data from a population of \FTwo mice with randomly allocated high-fat versus normal gene variants. Panels (a) and (b) display the estimated brain-specific and liver-specific dependence structures. Panel (c) shows the estimated systemic structure describing whole body interactions that simultaneously affect variables in both tissues.}
  \vspace{-0.25in} 
\end{center}
\end{figure}

The \citet{Crowley2014} data comprise expression measurements of $23,000$ transcripts in brain, liver, lung and kidney tissues in $45$ mice arising from three independent reciprocal $\text{F}_1$ crosses. A reciprocal F1 cross between two inbred strains A and B generates two sub-populations: the progeny of strain A females mated to strain B males denoted by AxB, and the progeny of strain B females to strain A males denoted by BxA. Across the two progeny groups, the set of alleles inherited is identical, with each mouse having inheriting half of its alleles from A and the other half from B; but the route through which those alleles were inherited differs, with, for example, AxB offspring inheriting their A alleles only from their fathers and BxA inheriting them only from their mothers. The underlying intervention in a reciprocal cross is therefore not the varying of genetics as such but the varying of parent-of-origin, or epigenetics, and so we might expect some of this epigenetic effect to manifest across all tissues.

We applied our EM method to a normalized subset of the \citet{Crowley2014} data, restricting attention to brain and liver, and removing gross effects of genetic background. Our analysis identified three edges on the systemic network as shown in Fig. \ref{fig:f1smallnet}(c) that include the genes \textit{Igf2}, \textit{Tab1}, \textit{Nrk} and \textit{Pde4b}, all from the imprinting-related set implicated in mediating epigenetic effects.  Thus, the inferred patterns of systemic-level gene relationships in the two studies coincide with the underlying interventions implied by the structure of those studies, with genes affecting body weight in the \citet{Dobrin2009} data and genes affected by parent-of-origin in the \citet{Crowley2014} data.

\begin{figure}[htb]
  \begin{center}
    \includegraphics[width=0.65\textwidth]{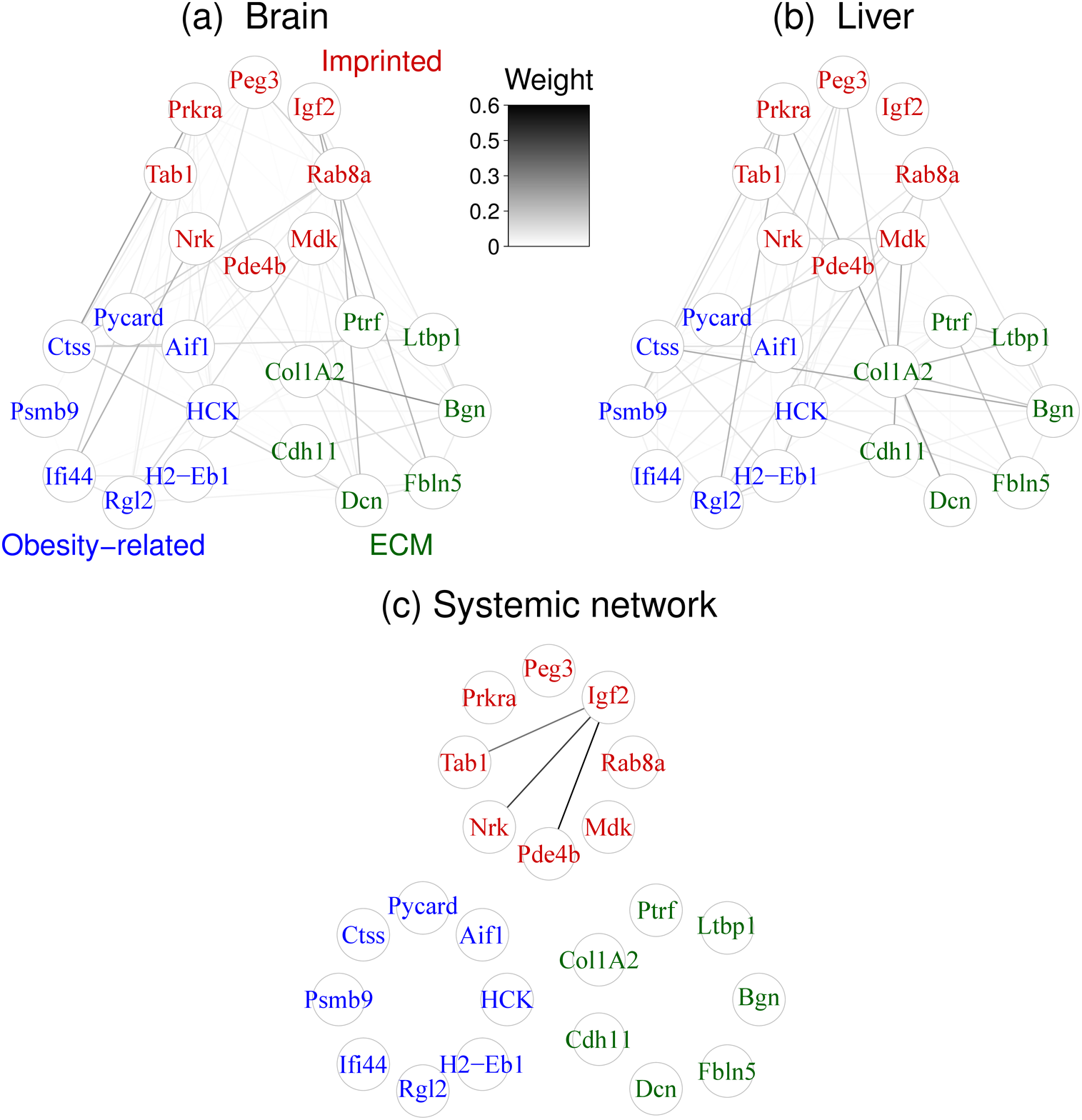}
	\caption{\label{fig:f1smallnet} Topology of gene co-expression networks inferred by the EM method for the data from a population of reciprocal $\text{F}_1$  mice. Panels (a) and (b) display the estimated brain-specific and liver-specific dependence structures. Panel (c) shows the estimated systemic structure describing whole body interactions that simultaneously affect variables in both tissues.} 
  \end{center}
  \vspace{-10pt}
\end{figure}

\begin{figure}[htb]\footnotesize
	\begin{center}
    \includegraphics[width=\textwidth]{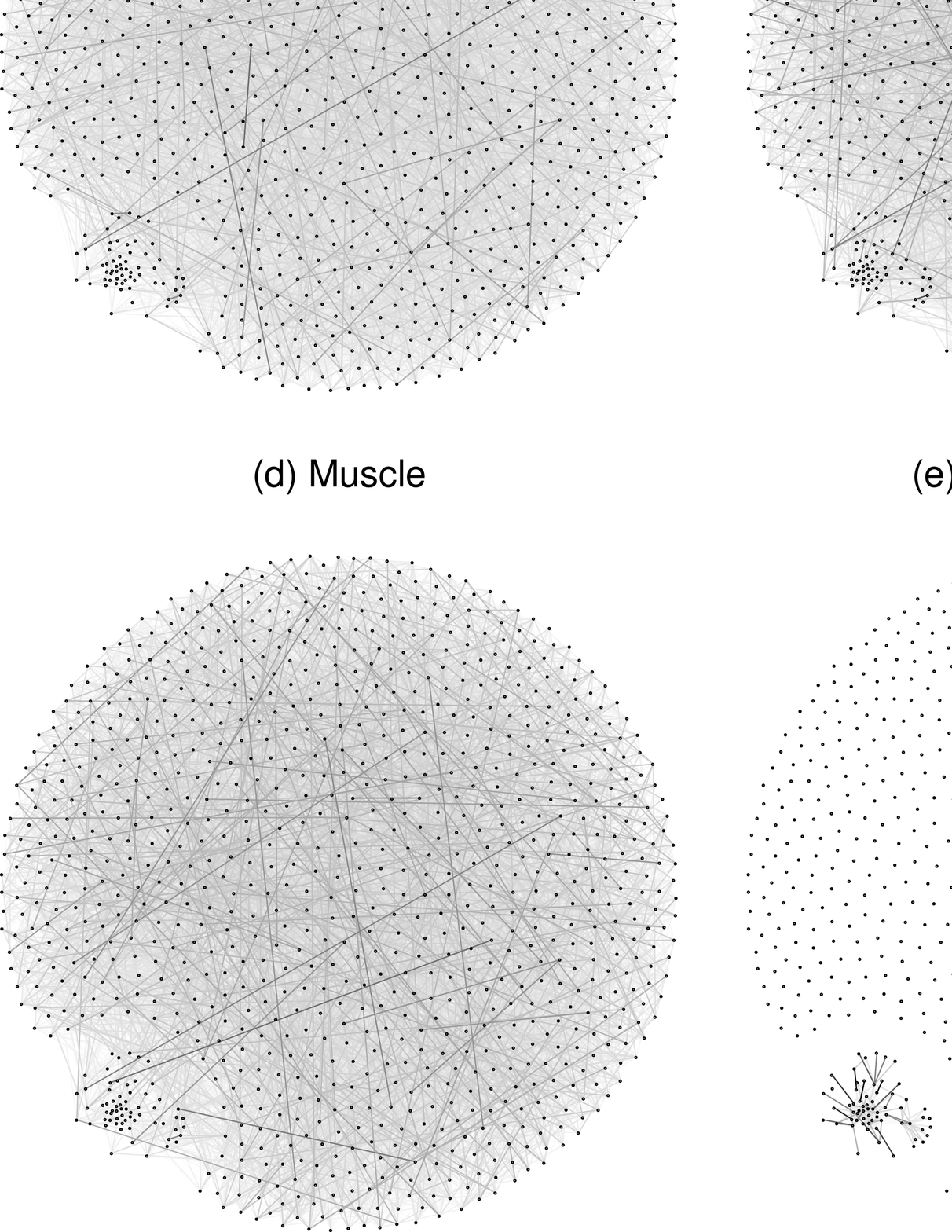}	
	\end{center}
	\caption{\label{fig:f2bignet} Topology of co-expression networks inferred by the EM method applied to measurements of the 1000 genes with highest within-tissue variance in a population of \FTwo mice. Panels (a-d) show category-specific networks estimated for adipose, hypothalamus, liver and muscle tissue. Panel (e) shows the structure of the estimated systemic network, describing across-tissue dependencies, with panel (f) showing a zoomed-in view of the connected subset of nodes in this graph. }
\end{figure}

To demonstrate the use of our method for higher dimensional data, we examined a larger subset of genes from \citet{Dobrin2009}. Selecting the $p = 1,000$ genes that had the largest within-group variance among the four tissues in the \FTwo population, we applied our graphical EM method, using the extended Bayesian information criterion to select the tuning parameters $\lambda_1$ and $\lambda_2$. The existence of a single, non-zero systemic layer for these data was strongly supported by significance testing, as described in the Supplementary Material. The topologies of the estimated tissue-specific and systemic networks are shown in Fig. \ref{fig:f2bignet}(a-d), with a zoomed-in view of the edges of the systemic network shown in Fig. \ref{fig:f2bignet}(f). The systemic network is sparse, with 249 edges connecting 62 of the 1000 genes in Fig. \ref{fig:f2bignet}(e); this sparsity may reflect there being few interactions simultaneously occurring across all tissues in this \FTwo population, with one contributing reason being that some genes are being expressed primarily in one tissue and not others. The systemic network also includes a connection between two genes, \textit{Ifi44} and \textit{H2-Eb1}, that are members of the \textit{Aif1} network of Fig. \ref{fig:f2smallnet}.
To characterize more broadly the genes identified in the systemic network, we conducted an analysis of gene ontology enrichment \citep{Shamir2005} in which  the distribution of gene ontology terms associated with connected genes in the systemic network was contrasted against the background distribution of gene ontology terms in the entire 1000-gene set; this showed that the systemic network is significantly enriched for genes associated with immune and metabolic processes, which accords with recent studies linking obesity to strong negative impacts on immune response to infection \citep{Milner2012,Lumeng2013}.
The original study of \citet{Dobrin2009} also showed that the enrichment of inflammatory response processes in co-expression from liver and adipose, again using unconditional correlations. 
%

\section{Discussion}\label{s:discussion}
{In this paper we consider joint estimation of a two-layer Gaussian graphical model that is different from but related to the single-layer model. In our setting, the single-layer model estimates an aggregate graph $\bO_{Y_k}$ by imposing sparsity on $\bO_{Y_k}$ directly. Our model, by contrast, estimates the two graphical layers that compose the aggregate, namely $\bO_k$ and $\bO_0$, and imposes sparsity on each. This can imply an aggregate graph $\bO_{Y_k}$ that is less sparse; but this is appropriate because in our setting $\bO_{Y_k}$ is a byproduct and, as such, is a secondary subject of inference. Importantly, our two-layer model includes the single-layer model as a special case, since in the absence of an appreciable systemic dependence, when $\Sigma_Z=0$, the two-layer model reduces to a single layer.}

{Our model lends itself to several immediate extensions. First, we currently assume that the systemic graph affects all tissues equally, but, as suggested by one reviewer, we can extend our model to allow the influence of the systemic layer to vary among tissues. For example, since muscle and adipose are both developed from the mesoderm, we might expect them to be more closely related to each other as compared with the pancreas, which is developed from the endoderm. We can accommodate such variation in our model as: 
\begin{equation*}
	Y_{k, i} = X_{k, i} + \alpha_k Z_i \quad (k = 1, \ldots, K; \; i = 1, \ldots, n),
\end{equation*}
where $\alpha_k$ quantifies the level of systemic influence in each tissue $k$. Our EM algorithm can also be modified to calculate $\alpha_k$ and $\Omega_k$. More details can be found in the Supplementary Material. }

Second, we can extend the $\ell_1$ penalized maximum likelihood framework to other nonconvex penalties such as the truncated $\ell_1$-function \citep{Shen2012} and the smoothly clipped absolute deviation penalty \citep{Fan2001}. Furthermore, we believe it would be both practicable and useful to extend these methods beyond the Gaussian assumption \citep[]{Cai2011,Liu2012,Xue2012}.

\section*{Acknowlegements}
The authors thank the editor, the associate editor and two reviewers for their helpful suggestions.
This work was supported in part by the U.S. National Institutes of Health and the National Science Foundation. Yuying Xie is also affiliated with Department of Statistics and Probability at Michigan State University.
Yufeng Liu is affiliated with Department of Genetics and Carolina Center for Genome Sciences, and both he and William Valdar  are also affiliated with Department of Biostatistics and the Lineberger Comprehensive Cancer Center 
at the University of North Carolina.

\begin{appendices}
\section{Derivation of the likelihood}
\label{appendix:likelihood}
For simplicity, we write $\bO$ for $\{\bO_k \}_{k=0}^K$ in the following derivation. To derive the log-likelihood of $y$, which is expressed as
\begin{align}\label{eq:Likelihood2}
	\mathcal{L}(\bO; y) 
	\propto & \; \sum_{k=1}^K \left\{ \log  \det(\Omega_{k})
	- \tr(\hat{\Sigma}_{Y(k,k)}\Omega_k) \right\}
	+ \log  \det(\Omega_{0}) \notag \\
	& \; - \log \det (A) + \sum_{l,m=1}^K 
	\tr \left( \Omega_l \hat{\Sigma}_{Y(l,m)} \Omega_m  A^{-1} \right),
\end{align}
we first state Sylvester's determinant theorem.
\vspace{-0.1 in}
\begin{lemma}[Sylvester theorem]	\label{Lemma.A1}
If $A$, $B$ are matrices of sizes $p \times n$ and $n \times p$, respectively, then
\vspace{-0.04in}
	\begin{equation*}
 		\det(I_p + AB) = \det(I_n + BA),
 \end{equation*}
\vspace{-0.1in}
where $I_a$ is the identity matrix of order a.
\end{lemma}

Since $Y$ follows a $Kp-$variate Gaussian distribution with mean ${0}$ and covariance matrix $\Sigma_Y = \{ \Sigma_{Y(l,m)} \}_{1 \leq l,m \leq K}$, we have $ f_Y(s) \propto \exp( s^\T \bO_Y s )$. In addition, we can derive $f_Y(s)$ from the joint probability $f_{Y, Z}(s,t)$ by integrating out $Z$ as follows:
\vspace{-0.04in}
	\begin{align*}
		f_Y(s)= & \; \int_{-\infty}^{\infty} 
		{f_Y(s \mid Z = t)f_Z(t) \mathrm{d}t} \\
		\propto & \; \int_{-\infty}^{\infty} \exp
		\Big[ \sum_{k=1}^K \Big\{ (s_{k} - t)^\T \bO_{k} (s_k - t) \Big\} 
		+ t^\T \bO_0 t \Big] \mathrm{d} t,
	\end{align*}
\vspace{-0.1in}
where $s = (s_1^\T, \ldots, s_K^\T)^\T$. We then expand the formula and have 
\begin{align*}
	f_Y(s) & = \exp \Big\{ \sum_{k=1}^K (s_k^\T \bO_k s_k) \Big\}
	\int_{-\infty}^{\infty}{\exp \Big[ t^\T \Big(\sum_{k = 0}^K \bO_{k} \Big) t
	-2 \Big\{ \sum_{k=1}^K (s_k^\T \bO_{k}) \Big\} t \Big] \mathrm{d} t} \\ 
	& = \exp \Big\{ \sum_{k=1}^K (s_k^\T \bO_k s_k) \Big\}
	\int_{-\infty}^{\infty}{\exp\Big(t^\T A t - 2 c^\T t \Big) \mathrm{d} t} \\
	& = \exp \Big\{ \sum_{k=1}^K (s_k^\T \bO_k s_k) \Big\}           
	\exp(-c^\T A^{-1} c)  
	\int_{-\infty}^{\infty}{\exp\Big\{( A t- c)^\T A^{-1}(At - c) \Big\} \mathrm{d}t} \\
	& = \exp \Big\{ \sum_{k=1}^K (s_k^\T \bO_k s_k) \Big\} \exp(-c^\T A^{-1} c) 
	\int_{-\infty}^{\infty}{\exp\Big\{(t - A^{-1} c)^\T A (t - A^{-1}c) \Big\} \mathrm{d}t} \\
	&\propto \exp \Big[\sum_{k=1}^K (s_k^\T \bO_{k} s_k)
	- \big\{ \sum_{k=1}^K (s_k^\T \bO_{k}) \big\} A^{-1}
	\big\{ \sum_{k=1}^K (\bO_{k} s_k) \big\} \Big] \\[5pt]
	& = \exp \Big\{ s^\T \Big(\{_d \bO_{k} \}_{1 \leq k \leq K} 
	- \{\bO_{l}A^{-1}\bO_{k} \}_{1 \leq l, k \leq K} \Big) s \Big\} \\[8pt]
	& = \exp ( s^{\T} \bO_Y s ) \, ,
\end{align*}
 where $ A = \sum_{k = 0}^{K} \bO_{k} $ and $c = \sum_{k=1}^K  \bO_{k} s_k $. Let $\{\bO_{l} A^{-1} \bO_{k} \}_{1 \leq l, k \leq K}$ as a block matrix in which the $(l, k )$th block is $\bO_{l} A^{-1} \bO_{k}$. Then we have $ Y \sim \mathcal{N}(0,\left[\{_d \bO_{k} \}_{ k=1}^K - \{\bO_{l} A^{-1} \bO_{k} \}_{1 \leq l, k \leq K}\right]^{-1})$ and $\bO_Y = \{_d \bO_{k} \}_{1 \leq k \leq K} - \{\bO_{l} A^{-1} \bO_{k} \}_{1 \leq l, k \leq K}$.
 
Next, we derive the expression for $\det(\Omega_Y)$. We know that
\begin{align*}
	\Sigma_Y = & \; \{_d \, \Sigma_{k} \}_{1 \leq k \leq K} + 
	\begin{pmatrix}
		I_p  \\
  		\vdots  \\
  		I_p
 	\end{pmatrix}
 	\begin{pmatrix}
  		\Sigma_0, & \ldots &, \Sigma_0  
	\end{pmatrix}  \\
	= & \; \{_d  \  \Sigma_{k} \}_{1 \leq k \leq K} 
	\left\{ I_{Kp} + \{_d  \  \bO_{k} \}_{1 \leq k \leq K}    
	\begin{pmatrix}
   		I_p  \\
  		\vdots  \\
  		I_p
 	\end{pmatrix}
 	\begin{pmatrix}
  		\Sigma_0, & \ldots &, \Sigma_0
	\end{pmatrix} \right\}, 
\end{align*}
 where $I_p$ and $I_{Kp}$ are $p \times p$ and $Kp \times Kp$  identity matrices, respectively. Using Lemma \ref{Lemma.A1}, 
\begin{align*}
	\det(\Sigma_Y) 
	& = \Big\{ \displaystyle\prod\limits_{k = 1}^K \det( \Sigma_{k}) \Big\} 
	\det \left\{ I_p +
	\begin{pmatrix}
  		\Sigma_0, & \ldots &, \Sigma_0  
  	\end{pmatrix}
  	\{_d \bO_{k} \}_{1 \leq k \leq K} 
  	\begin{pmatrix}
 		I_p  \\
 		\vdots  \\
		I_p
	\end{pmatrix}
	\right\} \\
	& = \Big\{ \prod_{k = 1}^K \det(\Sigma_{k}) \Big\}
	\det \Big(I_p + \Sigma_{0}\sum\limits_{k = 1}^K \bO_{k} \Big) \\
   & = \Big\{ \prod_{k = 1}^K \det(\Sigma_{k}) \Big\}
   \det \Big( \Sigma_{0} \bO_{0} 
   + \Sigma_{0} \sum_{k = 1}^K \bO_{k} \Big)  \\
   & = \Big\{ \prod_{k = 0}^{K }\det(\Sigma_{k})\Big\} \det(A).
\end{align*}
\vspace{-0.2in}
Therefore, we have
\begin{align*}
	\log \det(\bO_Y)  = - \log  \det(\Sigma_Y) 
	= - \log \det (A) +  \sum_{k = 0}^{K} \log \det(\bO_{k}).
\end{align*}
\vspace{-0.1in}
Combining the above results, the log-likelihood can be written as follows:
\begin{align*}
	\mathcal{L}( \bO_Y; y)
	= & \; -\frac{npK}{2} \log (2\pi)  
	+ \frac{n}{2} \log  \det(\bO_Y)
	- \frac{n}{2} \tr\big( \hat{\Sigma}_Y \bO_Y \big) \\
	= & \; - \frac{npK}{2} \log (2\pi) 
	+ \frac{n}{2} \Big\{ - \log \det (A)
	+  \sum_{k = 0}^{K} \log  \det(\bO_{k})  \Big\} \\
	& \; - \frac{n}{2} \tr\Big(\hat{\Sigma}_Y
	\Big[ \{_d \bO_{k} \}_{1 \leq k \leq K} 
	- \{\bO_{l} A^{-1} \bO_{k} \}_{1 \leq l, k \leq K} \Big] 
	\Big) \\ 
	= & \; -\frac{npK}{2} \log (2\pi) 
	+ \frac{n}{2} \Big\{ - \log \det(A) + \log \det(\bO_{0})
	+ \sum_{k = 1}^{K} \log  \det(\bO_{k}) \Big\} \\
	& \; - \frac{n}{2} \tr\Big(\hat{\Sigma}_Y
	\{_d \bO_{k} \}_{1 \leq k \leq K} 
	- \hat{\Sigma}_Y \{\bO_{l} A^{-1} \bO_{k} \}_{1 \leq l, k \leq K}  
	\Big) \\ 
	= & \; -\frac{npK}{2} \log (2\pi)  
	+ \frac{n}{2} \Big\{ \log \det(\bO_{0}) - \log \det (A) \Big\}
	+ \frac{n}{2}  \sum_{k = 1}^{K} \log \det(\bO_{k})
	\\
	& \; - \frac{n}{2} 
	 \sum_{k = 1}^{K} \tr \Big( \hat{\Sigma}_{Y(k, k)} \bO_{k} \Big)
	+ \frac{n}{2} \tr\Big\{\hat{\Sigma}_Y 
	\begin{pmatrix}
  		\Omega_{1}, & \ldots &, \Omega_{K}  
  	\end{pmatrix}^\T
  	A^{-1} 
	\begin{pmatrix}
  		\Omega_{1}, & \ldots &, \Omega_{K}  
  	\end{pmatrix}
	\Big\}\\ 
	= & \; -\frac{npK}{2} \log (2\pi)  
	+ \frac{n}{2} \tr\Big\{
	\begin{pmatrix}
  		\Omega_1, & \ldots&, \Omega_k  
  	\end{pmatrix}
	\hat{\Sigma}_Y
	\begin{pmatrix}
  		\Omega_1, & \ldots &, \Omega_k
  	\end{pmatrix}^\T
  	A^{-1} 
	\Big\} \\
	& \; + \frac{n}{2} \big\{ \log \det(\bO_{0})- \log \det(A) \big\}
	 + \frac{n}{2} \sum_{k = 1}^{K} \Big\{ \log \det(\bO_{k})
	- \tr \Big( \hat{\Sigma}_{Y(k, k)} \bO_{k} \Big) \Big\} \\ 
	= & \; -\frac{npK}{2} \log (2\pi) 
	+ \frac{n}{2} \sum_{l,m=1}^K \tr 
	\left( \bO_l \hat{\Sigma}_{Y(l,m)} \bO_m A^{-1} \right)  
	+ \frac{n}{2} \log \det(\bO_{0}) \\
	& - \frac{n}{2}\log  \det(A) + \frac{n}{2}\sum\limits_{k=1}^K 
	\Big\{ \log \det(\bO_{k}) - \tr(\hat{\Sigma}_{Y(k,k)}\bO_k) \Big\}.
\end{align*}


\section{Proof of Identifiability }\label{appendix:identifiable}
To demonstrate identifiability, it is sufficient to show the parameters $\bO_k$  $(k = 0, \ldots, K)$ are identifiable. To that end, we decompose $Y_k$ in two different ways as follows:
\begin{align}
	Y_k = X_k + Z = X_k - U + Z + U = X_k^* + Z^* \quad  (k = 1, ..., K), \notag
 \end{align}
where $U$ is a $p$-dimensional of random vector. With $ U \neq 0$, we have nonunique decompositions of $Y_k$. Under the model assumption, the resulting $X_k^*$ and $Z^*$  satisfy
\begin{align}
	&\cov(X^*_l, X^*_m) = 0 \quad (1 \leq l, m \leq K);  \label{eq:CovXstarX}\\
	& \cov(X^*_l, Z^*) = 0   \quad (l = 1, \ldots, K).  \label{eq:CovXstarZstar}
\end{align}
Expanding \eqref{eq:CovXstarX}, we have
\begin{align*}
	\cov(X^*_l, X^*_m) = & \;\cov(X_l, X_m) 
	+ \text{var}(U)  - \cov(X_l, U) - \cov(X_m, U)  \\
	= & \; 0 + \text{var}(U)  - \cov(X_i, U) - \cov(X_j, U)\\
	= & \; 0. 
\end{align*}
It follows that
\begin{align}
	\text{var}(U) = \cov(X_l, U) + \cov(X_m, U).\label{eq:ZX}
\end{align}
Similarly, from \eqref{eq:CovXstarZstar} we have
\begin{align*}
	\cov(X_l^*, Z^*) 
	& = \cov(X_l - U, Z + U)\\
    & = \cov(X_k, Z) - \text{var}(U) - \cov(U, Z) + \cov(U, X_l) \\
	& = 0 - \text{var}(U) - \cov(U, Z) + \cov(U, X_l) = 0, \notag
\end{align*}
which implies that
\begin{align}
	\text{var}(U)= - \cov(U, Z) 
	+ \cov(U, X_l). \label{eq:UIJ}
\end{align}
Since \eqref{eq:UIJ} holds for any $l$, we have
\begin{align}
	\cov(U, X_l) = \cov(U, X_m)  \quad (1 \leq l, m \leq K). \label{eq:XX}
\end{align}
Combining  \eqref{eq:ZX}--\eqref{eq:XX} gives
\begin{align}
	\cov(U, X_k) = & \; - \cov(U, Z), \label{eq:CovUX} \\
	\text{var}(U) = & \; -2 \, \cov(U, Z) \notag\\
	 = &\; 2 \,\cov(U, X_l) \quad (1 \leq l \leq K).  \label{eq:VarU}
\end{align}
By \eqref{eq:CovUX} and \eqref{eq:VarU}, we can further show that
\begin{align*}
	\text{var}(X_l^*) 
	= & \;  \text{var}(X_l - U) \\
	= & \; \text{var}(X_l) + \text{var}(U) - 2 \, \cov(U, X_l) \\
	= & \; \text{var}(X_l); \\
	\text{var}(Z^*) = & \; \text{var}(Z + U) \\
	= & \; \text{var}(Z) + \text{var}(U) - 2 \, \cov(U, Z) \\
	= & \; \text{var}(Z).
\end{align*}
Therefore, the resulting $\text{var}(Z^*)$ and $\text{var}(X_l^*)$ remain the same for different decomposition of  $Y_{l}$. Consequently, our model is identifiable.


\section{Proof of Proposition \ref{prop:pro1}}\label{appendix:Convergence}

We divide the proof into two parts. For the first part, we prove that the penalized log-likelihood is bounded above; and for the second part, we show that the penalized log-likelihood does not decrease for each iteration of the graphic EM algorithm. 

From \eqref{eq:Likelihood2}, the log-likelihood is
\begin{align*}
	\mathcal{L}(\bO; y )
	\propto & \; \sum_{k=1}^K 
	\Big\{ \log  \det(\bO_{k})
	- \tr(\hat{\Sigma}_{Y(k,k)}\bO_k) \Big\} 
	+ \log \det(\bO_{0}) \notag \\
	& \; - \log \det (A) +  \sum\limits_{l,m=1}^K \tr 
	\Big( \bO_l \hat{\Sigma}_{Y(l,m)} \bO_m  A^{-1} \Big) \\
	= & \; \sum_{k=0}^K  \big\{ \log \det(\bO_{k}) \big\} 
	- \log\det (A)
	- \sum_{k = 1}^K \tr\big(\hat{\Sigma}_{Y(k, k)} \bO_{k} \big)\\
	& \; + \sum\limits_{l,m=1}^K \tr 
	\Big( \bO_l \hat{\Sigma}_{Y(l,m)} \bO_m  A^{-1} \Big).
\end{align*}

For $\lambda_1 >0$ and $\lambda_2 > 0$, by Lagrangian duality, the problem \eqref{eq:Penlikelihood} is equivalent to the following constrained optimization problem:
\begin{align}
	\max & \Bigg[  \sum_{k=0}^K \big\{ \log \det(\bO_{k}) \big\}
	- \log \det(A)
	- \sum_{k = 1}^K \big\{ \tr\big(\hat{\Sigma}_{Y(k, k)} \bO_{k} \big) \big\}  \notag \\
	& + \sum\limits_{l,m=1}^K \tr 
	\left( \bO_l \hat{\Sigma}_{Y(l,m)} \bO_m  A^{-1} \right) \Bigg], \label{eq:bdd} 
	\end{align}
subject to $\bO_k \succ 0$, $|\bO_k^-|_1 \leq C_{\lambda_1, \lambda_2}$  $(k = 0,\ldots, K)$ and $C_{\lambda_1, \lambda_2} < \infty$. Here $\bO^-$ represents the off-diagonal entries of $\bO$, and $C_{\lambda_1, \lambda_2}$ is a constant depending on the values of $\lambda_1$ and $\lambda_2$. Since  $|\bO^-_k|_1$ is bounded, the potential problem comes from the behavior of the diagonal entries which could grow to infinity. Because of the positive-definite requirement, the diagonal entries of $\{ \bO_k \}_{k=0}^K$ are positive. After some algebra, \eqref{eq:bdd} becomes
\begin{align} 	
	& \;\sum_{k=0}^K  \Big\{ \log \det(\bO_{k}) \Big\} 
	- \log \det(A)
	- \sum_{k = 1}^K \tr\big(\hat{\Sigma}_{Y(k, k)} A A^{-1} \bO_{k}
	\big)  \notag \\
	& \; + \sum_{l,m=1}^K \tr 
	\left( \hat{\Sigma}_{Y(l,m)} \bO_m  A^{-1} \bO_l \right) \notag \\
	= & \; \sum_{k=0}^K \Big\{ \log \det(\bO_{k}) \Big\}	- \log\det (A) 
	- \sum_{k = 1}^K 
	 \tr(\hat{\Sigma}_{Y(k, k)} \bO_0 A^{-1} \bO_k) \notag \\
	& \; - \sum_{l \neq m}  \tr \big\{ (\hat{\Sigma}_{Y(m, m)} - \hat{\Sigma}_{Y(m, l)} )
	\bO_l A^{-1} \bO_m \big\} \notag \\
	= & \; \sum_{k=0}^K \Big\{ \log \det(\bO_{k}) \Big\}
	- \log \det (A)	- \sum_{k = 1}^K 
	\tr(\hat{\Sigma}_{Y(k, k)}\bO_0 A^{-1} \bO_k) \notag \\
	& \; - \sum_{l > m \geq 1}  \tr( M_{(l,m)} \bO_l A^{-1} \bO_m) 
	, \label{eq:target}
\end{align}
where $M_{(l,m)} = \hat{\Sigma}_{Y(l, l)} + \hat{\Sigma}_{Y(m, m)} - 2 \hat{\Sigma}_{Y(m, l)}$. The equality in \eqref{eq:target} comes from the fact that
\begin{align*}
	\tr \big\{ (\hat{\Sigma}_{Y(m, m)} - \hat{\Sigma}_{Y(m, l)})
	\bO_l A^{-1} \bO_m \big\} 
	 = & \; \tr \big[ \big\{ (\hat{\Sigma}_{Y(m, m)} - \hat{\Sigma}_{Y(m, l)} )
	\bO_l A^{-1} \bO_m \big\}^\T \big]\\
	 = & \; \tr \big\{ \bO_m A^{-1} \bO_l (\hat{\Sigma}_{Y(m, m)} - \hat{\Sigma}_{Y(m, l)}^\T )\big\} \\
	= & \;  \tr \big\{ (\hat{\Sigma}_{Y(m, m)} - \hat{\Sigma}_{Y(l, m)} ) \bO_m A^{-1} \bO_l \big\}.
\end{align*}

Since $\bO_k$ is positive definite and $|\bO_k^-|_1$ is bounded, we decompose them into $\bO_k = B_k + D_k$. Let $B_k$ be a matrix with bounded diagonal entries and satisfy $0 \leq \tau_3 \leq \phi_{min}( B_k) \leq \phi_{max}( B_k) \leq \tau_4$, and $D_k$ be a diagonal matrix whose diagonal entries are greater than some positive number $\v$ and possibly grow to infinity, namely, $ 0 \leq \| D_j^{-1} \| \leq 1/{\v}$. Let $B_A = \sum_{k = 0}^K B_k$ and $D_A = \sum_{k = 0}^K D_k$. By Weyl's inequality, we have 
\begin{align*} 
	(K+ 1)\tau_3 \leq \phi_{min}(B_A) \leq 
	& \; \phi_{max}(B_A) \leq (K + 1) \tau_4, \\
	0 < & \; (K+ 1)\v \leq \phi_{min}(D_A), \\
	\phi_{max}(D_A^{-1}) \leq & \; \frac{1}{(K+ 1)\v}.
\end{align*}

Now we consider four different cases:

Case One: $| D_k |_1$  $(k = 0, \ldots, K)$ is bounded.

 In this case, $\det(\bO_k)$ and $ \| \bO_k \|_\infty$ $(k =  0, \ldots, K)$ are all bounded above. Thus, the function in \eqref{eq:target} is also bounded above.

Case Two: All $| D_k |_1$ are bounded except $D_l$.

In this case,  we only need to control the behavior of the following terms
\begin{align*}
	& \; \log \det(\bO_{l}) - \log  \det(A)  
	- \sum_{K \geq k > l} 
 	\tr( M_{(k, l)} \bO_k A^{-1} \bO_l ) 
	- \sum_{l > m \geq 1} 
	\tr( M_{(l, m)} \bO_l A^{-1} \bO_m )  \\ 
	& \; - \tr(\hat{\Sigma}_{Y(l, l)}\bO_0 A^{-1} \bO_l) \\ 
	= & \; \Big\{ \log \det(\bO_{l}) - \log  \det(A) \Big\} 
	- \sum_{K \geq k > l} \Big[ 
	\tr\big\{ M_{(k, l)} \bO_k (B_A + D_A)^{-1} (B_l + D_l) \big\} \Big] \\
 	& \; - \sum_{l > m\geq 1} \Big[ 
	\tr\big\{ M_{(l, m)} (B_l + D_l ) (B_A + D_A)^{-1} \bO_m \big\} \Big]
	- \tr\big\{ \hat{\Sigma}_{Y(l, l)}\bO_0 (B_A + D_A)^{-1} (B_l + D_l) \big\} \\
 	= & \; \text{I} + \text{II} + \text{III} + \text{IV}.  
\end{align*}
We first want to bound Term I: $ \log \det(\bO_{l}) - \log  \det(A)$. 
Since $A = \sum_{k = 0}^K \bO_k$ and all $\bO_k$ are positive definite, by the  Minkowski determinant theorem, it follows that 
\begin{align*}
	\det(A) \geq & \; \Big\{ \det \Big(\sum_{k \neq l} \bO_k \Big)^{1/p} 
	+ \det(\bO_l)^{1/p} \Big\}^p \\
	\geq & \; \{ \det(\bO_l)^{1/p} \}^p \\ 
	= & \; \det(\bO_l).
\end{align*}
Therefore,  we have $\text{term I}  < 0$. 

To bound Terms II and III, using the Woodbury matrix identity, we have
\begin{align} 
	\bO_l A^{-1} = & \; (B_l + D_l) (B_A + D_A  )^{-1} \\
	 = & \; (B_l + D_l )
	\{ D_A^{-1} - D_A^{-1} (D_A^{-1} + B_A^{-1})^{-1} D_A^{-1} \} \notag \\
	= & \;  B_l D_A^{-1} + D_l D_A^{-1} 
	- B_l D_A^{-1} (D_A^{-1} + B_A^{-1})^{-1} D_A
	^{-1} \notag \\
	& \;  - D_l D_A^{-1} (D_A^{-1} + B_A^{-1})^{-1} D_A^{-1},
	 \label{eq:target2} \\
	A^{-1} \bO_l = & \; (B_A + D_A )^{-1}(B_l + D_l) \notag\\
	 = & \; \{D_A^{-1} - D_A^{-1} (D_A^{-1} 
	+ B_A^{-1})^{-1} D_A^{-1} \}(B_l + D_l ) \notag \\
	= & \; D_A^{-1} B_l +  D_A^{-1} D_l 
	- D_A^{-1} (D_A^{-1} + B_A^{-1})^{-1} D_A^{-1}B_l \notag \\
	& \; - D_A^{-1} (D_A^{-1} + B_A^{-1})^{-1} D_A^{-1} D_l. \label{eq:target2a}
\end{align}
We want to bound the spectral normal of \eqref{eq:target2} and \eqref{eq:target2a}. Since 
\begin{align*}
	D_A^{-1} = \Big\{_d \Big(\sum_{k = 0}^K D_{k(i,i)}\Big)^{-1} \Big\},
\end{align*}
we first show
\begin{align} 
	&\|D_A^{-1} \| = \Big\| \Big\{_d \frac{1}{\sum_{k = 0}^K D_{k(i,i)}} \Big\} \Big\|
	\leq \frac{1} {(K + 1)\v}, \label{eq:bound1}\\
	&\|B_l D_A^{-1}\| \leq \|B_l\| \;  \|D_A^{-1}\|
	\leq \frac{ \tau_4} {(K + 1)\v}, \label{eq:bound2}\\
	&\|D_A^{-1} B_l\| \leq \|B_l\| \; \|D_A^{-1}\|
	\leq \frac{ \tau_4} {(K + 1)\v}, \label{eq:bound3}\\
	& \|D_l D_A^{-1}\| = \| D_A^{-1} D_l \|
	 = \Big\| \Big\{_d \frac{D_{l(i,i)}}{ D_{A(i,i)}} \Big\} \Big\| 
	= \Big\| \Big\{_d \frac{D_{l(i, i)}}{\sum_{k =0}^K D_{k(i,i)}}
	 \Big\} \Big\| \leq 1.\label{eq:bound4} 
\end{align}
By the Weyl's inequality, we have
\begin{align}
	& \| ( B_A^{-1} + D_A^{-1})^{-1}\| 	= \frac{1}{\phi_{\text{min}} ( B_A^{-1} + D_A^{-1})}
	\leq \frac{1}{\phi_{\text{min}} ( B_A^{-1}) +\phi_{\text{min}}( D_A^{-1})} \notag \\
	&\; \; \quad \quad \quad \quad \quad \quad \quad
	\leq \frac{1}{ \phi_{\text{min}}(B_A^{-1})} 
	= \phi_{\text{max} }(B_A) \leq (K + 1)\tau_4, \label{eq:bound5}\\
	& \|B_l D_A^{-1} (B_A^{-1} + D_A^{-1})^{-1} D_A^{-1} \|
	\leq \| B_l \| \; \| D_A^{-1} \|^2 \; \|(B_A^{-1} + D_A^{-1})^{-1} \|  
	\leq \frac{\tau_4^2}{(K + 1) \v^2}, \label{eq:bound6} \\
	& \|D_A^{-1} (B_A^{-1} + D_A^{-1})^{-1} D_A^{-1} B_l \|
	\leq \| B_l \| \; \| D_A^{-1} \|^2 \; \|(B_A^{-1} + D_A^{-1})^{-1} \|  
	\leq \frac{\tau_4^2}{(K + 1) \v^2}, \label{eq:bound7}\\
	&\| D_l D_A^{-1} (B_A^{-1} + D_A^{-1})^{-1} D_A^{-1} \|
	\leq \|D_l D_A^{-1}\| \; \|(B_A^{-1} + D_A^{-1})^{-1} \| \; \| D_A^{-1} \|
	\leq  \frac{\tau_4}{\v}, \label{eq:bound8}\\
	&\| D_A^{-1} (B_A^{-1} + D_A^{-1})^{-1} D_A^{-1} D_l\|
	\leq \|D_A^{-1} D_l\| \; \|(B_A^{-1} + D_A^{-1})^{-1} \| \; \| D_A^{-1} \|
	\leq  \frac{\tau_4}{\v}. \label{eq:bound9}
\end{align}

Combining \eqref{eq:bound1}--\eqref{eq:bound9}, the spectral norms of \eqref{eq:target2} and \eqref{eq:target2a} are bounded above as 
\begin{align}
	\| \bO_l A^{-1} & \| = \|(B_l + D_l) (B_A + D_A  )^{-1} \| \notag \\
	 & = \| B_l D_A^{-1} + D_l D_A^{-1} 
	- B_l D_A^{-1} (D_A^{-1} + B_A^{-1})^{-1} D_A^{-1}
	- D_l D_A^{-1} (D_A^{-1} + B_A^{-1})^{-1} D_A^{-1} \| \notag \\
	& \leq \| B_l D_A^{-1} \| + \| D_l D_A^{-1}\| 
	+ \| B_l D_A^{-1} (D_A^{-1} + B_A^{-1})^{-1} D_A^{-1}\| \notag \\
	& \quad + \| D_l D_A^{-1} (D_A^{-1} + B_A^{-1})^{-1} D_A^{-1} \| \notag \\
	& \leq \frac{ \tau_4} {(K + 1)\v} + 1 +  \frac{\tau_4^2}{(K + 1) \v^2}
	+  \frac{\tau_4}{\v} \notag \\
	& = \frac{(K + 2)\v \tau_4 + (K + 1) \v^2 + \tau_4^2}
	{(K + 1) \v^2}
	< \infty; \label{eq:boundOA1} \\
	\|A^{-1} \bO_l & \| = \| (B_A + D_A  )^{-1}(B_l + D_l) \| \notag \\
	 & = \| D_A^{-1}B_l + D_A^{-1} D_l 
	-  D_A^{-1} (D_A^{-1} + B_A^{-1})^{-1} D_A^{-1}B_l
	- D_A^{-1} (D_A^{-1} + B_A^{-1})^{-1} D_A^{-1} D_l \| \notag\\
	& \leq \| D_A^{-1} B_l\| + \| D_A^{-1} D_l\| 
	+ \| D_A^{-1}(D_A^{-1} + B_A^{-1})^{-1} D_A^{-1} B_l\| \notag \\
	& \quad + \| D_A^{-1} (D_A^{-1} + B_A^{-1})^{-1} D_A^{-1} D_l\| \notag \\
	& \leq \frac{ \tau_4} {(K + 1)\v} + 1 +  \frac{\tau_4^2}{(K + 1) \v^2}
	+  \frac{\tau_4}{\v} \notag \\
	& = \frac{(K + 2)\v \tau_4 + (K + 1) \v^2 + \tau_4^2}{(K + 1) \v^2}
	< \infty. \label{eq:boundOA2}
\end{align}

Since $M_{(l,k)}$ and $M_{(k,l)}$ only depend on the value of the sample covariance $\hat{\Sigma}_Y$, they are bounded above for any $k \neq l$. Based on the assumption that $| \bO_k |_1$ is bounded above for any $k \neq l$, we can bound $ \| \bO_k \|$ using the fact $ \| \bO_k \| < p \| \bO_k \|_\infty < p | \bO_k |_1 < \infty $. Therefore, 
\begin{align}
	\text{II} = &\; \sum_{K \geq k \geq l} 
	 -\tr\big( M_{(k, l)} \bO_k A^{-1} \bO_l \big) \notag \\ 
	= & \; \sum_{K \geq k \geq l} \sum_{j = 1}^p  
	\Big\{ - \big( M_{(k, l)} \bO_k 
	A^{-1} \bO_l \big)_{(j,j)} \Big\}  \notag\\
	\leq & \; \sum_{K \geq k \geq l} 
	\big( p \| M_{(k, l)} \bO_k A^{-1} \bO_l \| \big) \label{eq:caseII1} \\
	\leq & \; \sum_{K \geq k \geq l} 
	\Big( p \| M_{(l, k)}\| \; \| \bO_k\| \; \| A^{-1} \bO_l \| \Big) 
	< \infty, \notag
\end{align}
where the inequality of \eqref{eq:caseII1} is due to Lemma \ref{Lemma.A5}. Similarly, the term III is also bounded above. 

Since $\| \bO_0 \| $ is bounded by assumption, we can bound term IV as follows: 
  \begin{align*}
	\text{IV} = & \; -\tr \big( \hat{\Sigma}_{Y(l, l)} 
	\bO_0 A^{-1} \bO_l \big) \notag \\ 
	= & \;\sum_{j = 1}^p \Big\{ - \big( \hat{\Sigma}_{Y(l, l)} \bO_0 
	A^{-1} \bO_l \big)_{(j,j)} \Big\}   \\
	\leq & \; p\| \hat{\Sigma}_{Y(l, l)} \bO_0 A^{-1} \bO_l \| \\
	\leq & \;  p\|\hat{\Sigma}_{Y(l, l)} \| \; 
	\| A^{-1}\bO_l \| \; \| \bO_0 \| < \infty. 
\end{align*}
Therefore, the log-likelihood in \eqref{eq:target} is bounded above in this case. 

Case Three: All $| D_k |_1$ are bounded except $D_0$.

In this case, we only need to control the behaviour of 
\[
 \log\det(\bO_0) - \log \det(A)
 - \sum_{k = 1}^K  \tr (\hat{\Sigma}_{Y(k,k)} \bO_0 A^{-1}\bO_k) .
\]
Following the same argument as \eqref{eq:boundOA1}, we have $\|\bO_0 A^{-1} \| < \infty$ and 
\begin{align*}
	-\sum_{k = 1}^K \big\{ \tr (\hat{\Sigma}_{Y(k,k)} \bO_0 A^{-1}\bO_k) \big\} 
	= & \; \sum_{k = 1}^K\sum_{j = 1}^p 
	\Big\{ -\big(\hat{\Sigma}_{Y(k,k)} \bO_0 A^{-1}\bO_k \big)_{(j,j)} \Big\}  \\
	\leq &\; \sum_{k = 1}^K \big( p\| \hat{\Sigma}_{Y(k, k)} \bO_0 A^{-1} \bO_k \| \big) \\
	\leq & \; \sum_{k = 1}^K \big( p \|\hat{\Sigma}_{Y(k, k)} \| \;
	 \|\bO_0 A^{-1} \| \; \| \bO_k \| \big) 
	 \leq  \infty.
\end{align*}
Combining with the fact that $\log \det(\bO_{0}) - \log \det(A) < 0$, we prove that the log-likelihood in \eqref{eq:target} is also bounded in this case.

Case Four: $|\bO_r|_1$ and $|\bO_s|_1$ are not bounded. Namely, $|D_r |_1$ and $|D_s|_1$ have the same rate going to infinity. 

By the Hadamard's inequality, it follows 
\[ 
\sum_{i = 1}^p ( \log  \bO_{k(i,i)}) \geq \log \det(\bO_k).
\] Also, since $ \bO_{k(i,i)} = B_{k(i,i)} + D_{k(i,i)}$ with  bounded $B_{k(i,i)}$,  $\sum_{i = 1}^p (\log \bO_{k(i,i)} )$ has the same rate going to infinity as $\sum_{i = 1}^p (\log  D_{k(i,i)} )$. Then the order of log-likelihood in \eqref{eq:target} is equivalent to the order of
\begin{align}
	& \sum_{k=0}^K \sum_{i=1}^p  ( \log D_{k(i,i)}) 
	 - \log \det(A)
	- \sum_{l >  m \geq 1} \big\{ \tr( M_{(l,m)} \bO_l A^{-1} \bO_m) \big\} \notag \\
	& - \sum_{k=1}^K \big\{ \tr(\hat{ \Sigma}_{Y(k, k)}\bO_0 A^{-1} 
	\bO_k)\big\}. \label{eq:target3}
\end{align}
By the Minkowski determinant theorem, we have
\begin{align*}
	\det(A) & =  \det(B_A + D_A) \geq \{ \det(D_A)^{1/p} + \det(B_A)^{1/p} \}^p \\
	& \geq \{ \det(D_A)^{1/p} \}^p  = \det (D_A)=  \det \Big(\sum_{k=0}^K D_k \Big).
\end{align*}
With the Woodbury matrix identity,  we have
\[
A^{-1} = (B_A + D_A)^{-1} = D_A^{-1} - D_A^{-1} (D_A^{-1} + B_A^{-1})^{-1} D_A^{-1}.\]
Combining the above results and the fact that $\bO_l = B_l + D_l$ and $\bO_m = B_m + D_m$, we can show that \eqref{eq:target3} is bounded by:
\begin{align*}
	& \; \sum_{k=0}^K \sum_{i=1}^p \log D_{k(i,i)}
	 - \log  \det (\sum\limits_{k=0}^K D_k) 
	- \sum_{l >  m \geq 1}  \tr( M_{(l,m)} \bO_l A^{-1} \bO_m)   \\
	& \; - \sum_{k=1}^K \tr(\hat{ \Sigma}_{Y(k, k)}\bO_0 A^{-1} \bO_k) \\
	= & \;  \sum_{k=0}^K \sum_{i=1}^p  \log D_{k(i,i)}
	-  \log  \det (\sum\limits_{k=0}^K D_k) 
	- \sum_{l > m \geq 1} \tr(M_{(l,m)} D_l D_A^{-1} D_m) \\
	& \; + \sum_{l > m \geq 1} 
	 \tr\big\{ M_{(l,m)} D_l D_A^{-1}(D_A^{-1} 
	+ B_A^{-1})^{-1} D_A^{-1}D_m \big\} \\ 
	& \; - \sum_{l > m \geq 1}  \tr(M_{(l,m)} B_l A^{-1} \bO_m)  
	- \sum_{l > m \geq 1}  \tr(M_{(l,m)} D_l A^{-1} B_m) \notag \\
	& \; - \sum_{k=1}^K  \tr( \hat{\Sigma}_{Y(k, k)} D_0 D_A^{-1} D_k )  
	+ \sum_{k = 1}^K  \tr\big\{ \hat{\Sigma}_{Y(k, k)}  D_0 D_A^{-1}(D_A^{-1} 
	+ B_A^{-1})^{-1} D_A^{-1}D_k \big\} \\
	& \; - \sum_{k=1}^K  \tr(\hat{\Sigma}_{Y(k, k)} B_0 A^{-1} \bO_k )  
	- \sum_{k=1}^K  \tr(\hat{\Sigma}_{Y(k, k)} D_0 A^{-1} B_k ). 
\end{align*}

Therefore, we have  
\begin{align}
	\|D_k D_A^{-1}\| & = \| D_A^{-1} D_k \|
	 = \Big\| \Big\{_d \frac{D_{k(i,i)}}{ D_{A(i,i)}} \Big\} \Big\| 
	 = \Big\| \Big\{_d \frac{D_{k(i, i)}}{\sum_{l =0}^K D_{l(i,i)}}
	 \Big\} \Big\| \leq 1  \quad (k = 0, \ldots K). \label{eq:bound11} 
\end{align}

Combining \eqref{eq:bound1}--\eqref{eq:bound9} and using Woodbury matrix identity, we have 
\begin{align}
	\|A^{-1} D_l\| = & \; \|(B_A + D_A )^{-1} D_l \| \notag \\
	 = & \; \| D_A^{-1} D_l- D_A^{-1} (D_A^{-1} + B_A^{-1})^{-1} D_A^{-1} D_l \| \notag\\
	 \leq & \; \| D_A^{-1} D_l\| 
	+ \| D_A^{-1}(D_A^{-1} + B_A^{-1})^{-1} D_A^{-1} D_l\| \notag \\
	\leq & \;  1 +  \frac{\tau_4}{\v} < \infty \quad  (l = 0, \ldots, K). \label{eq:boundOA3}
\end{align}
Using \eqref{eq:bound5} and \eqref{eq:bound11},  
\begin{align}
	& \; \sum_{l > m \geq 1} \tr\big\{ M_{(l,m)} D_l D_A^{-1}(D_A^{-1} 	
	+ B_A^{-1})^{-1} D_A^{-1}D_m \big\} \notag \\
	= & \; \sum_{l > m \geq 1}\sum_{j = 1}^p 
	 \big\{ M_{(l,m)} D_l D_A^{-1}(D_A^{-1} 	
	+ B_A^{-1})^{-1} D_A^{-1}D_m \big\}_{(j, j)} \notag \\
	\leq & \; \sum_{l > m \geq 1}  \Big(
	 p \| M_{(l,m)} D_l D_A^{-1}(D_A^{-1} 	
	+ B_A^{-1})^{-1} D_A^{-1}D_m \| \Big) \notag \\
	\leq & \; \sum_{l > m \geq 1}  \Big(
	 p \| M_{(l,m)} \| \; \| D_l D_A^{-1} \|
	 \; \|(D_A^{-1} + B_A^{-1})^{-1} \| \; \|D_A^{-1}D_m \| \Big)
	 \leq \infty,  \label{eq:bound12}\\
	& \; \sum_{k = 1}^K  \tr\big\{ \hat{\Sigma}_{Y(k, k)} 
	D_0 D_A^{-1}(D_A^{-1} + B_A^{-1})^{-1} D_A^{-1}D_k \big\}  \notag \\
	= & \; \sum_{k = 1}^K \sum_{j = 1}^p \big\{ \hat{\Sigma}_{Y(k, k)} 
	D_0 D_A^{-1}(D_A^{-1} + B_A^{-1})^{-1} D_A^{-1}D_k \big\}_{(j,j)}  \notag \\
	\leq & \; \sum_{k = 1}^K \Big( p \| \hat{\Sigma}_{Y(k, k)} 
	 D_0 D_A^{-1}(D_A^{-1} + B_A^{-1})^{-1} D_A^{-1}D_k \| \Big) \notag \\
	\leq & \; \sum_{k = 1}^K \Big( p \| \hat{\Sigma}_{Y(k, k)} \| \; 
	 \| D_0 D_A^{-1}\| \; \| (D_A^{-1} + B_A^{-1})^{-1}\| \; \| D_A^{-1}D_k \| \Big)
	 \leq \infty   \label{eq:bound13}.
\end{align}

Additionally, using \eqref{eq:boundOA2} and \eqref{eq:boundOA3}, we have
\begin{align}
	 & \; - \sum_{l > m \geq 1} \tr ( M_{(l,m)} B_l A^{-1}\bO_m )  \notag \\
	= & \;  \sum_{l > m \geq 1}\sum_{j = 1}^p 
	 \big\{- M_{(l,m)} B_l A^{-1}\bO_m \big\}_{(j, j)} 
	 \leq \sum_{l > m \geq 1}  \big(
	 p \| M_{(l,m)} B_l A^{-1}\bO_m \| \big) \notag \\
	\leq & \; \sum_{l > m \geq 1} \big( 
	p \| M_{(l,m)} \|  \; \| B_l \|
	\; \|A^{-1}\bO_m \| \big) \leq \infty,  \label{eq:bound14}\\
	& \; - \sum_{l > m \geq 1}  \tr ( M_{(l,m)} D_l A^{-1} B_m )  \notag \\
	= & \; \sum_{l > m \geq 1}\sum_{j = 1}^p 
	 \big\{ - M_{(l,m)} D_l A^{-1} B_m \big\}_{(j, j)} 
	 \leq 
	 \sum_{l > m \geq 1} \big( p \| M_{(l,m)} D_l A^{-1}B_m \| \big) \notag \\
	\leq & \; \sum_{l > m \geq 1} \big(  
	 p \| M_{(l,m)} \|  \; \| D_l A^{-1}\| \:
	 \| B_m \|  \big)\leq \infty; \label{eq:bound15}\\
	 & \; - \sum_{k = 1}^K \tr \big\{ \hat{\Sigma}_{Y(k, k)} 
	 B_0 A^{-1}\bO_k \big\} \notag \\
	 = & \;  \sum_{k = 1}^K \sum_{j = 1}^p \big\{ - \hat{\Sigma}_{Y(k, k)} 
	 B_0 A^{-1}\bO_k \big\}_{(j,j)}  
	 \leq \sum_{k = 1}^K \big( p \| \hat{\Sigma}_{Y(k, k)} 
	 B_0 A^{-1} \bO_k \| \big) \notag \\
	 \leq & \; \sum_{k = 1}^K  \big( p \| \hat{\Sigma}_{Y(k, k)} \| \; 
	 \| B_0 \| \;  \| A^{-1}\bO_k \| \big) 
	 \leq \infty,  \label{eq:bound16} \\
	 & \; - \sum_{k = 1}^K \tr\big\{ \hat{\Sigma}_{Y(k, k)} 
	 D_0 A^{-1} B_k \big\} \notag \\
	= & \; \sum_{k = 1}^K \sum_{j = 1}^p \big\{ -\hat{\Sigma}_{Y(k, k)} 
	D_0 A^{-1} B_k \big\}_{(j,j)}  
	\leq \sum_{k = 1}^K \big( p \| \hat{\Sigma}_{Y(k, k)} 
	D_0 A^{-1} B_k \| \big) \notag \\
	\leq & \;  \sum_{k = 1}^K \big( p \| \hat{\Sigma}_{Y(k, k)} \| \; 
	\| D_0  A^{-1} \|\; \| B_k \|  \big)
	\leq \infty.   \label{eq:bound17}
\end{align}

Using \eqref{eq:bound12}--\eqref{eq:bound17}, the order of  \eqref{eq:target} is equivalent to
\begin{align}\label{eq:target4}
	\ \ & \sum_{k=0}^K \sum_{j = 1}^p  \log D_{k(j,j)} 
	-  \log  \det \big(\sum_{k=0}^K D_k \big) 
	- \sum_{l > m \geq 1}  \tr(M_{(l,m)} D_l D_A^{-1} D_m) \notag \\
	& - \sum_{k=1}^K \tr( \hat{\Sigma}_{Y(k, k)} D_0 D_A^{-1} D_k )  \notag \\
	= \ & \; \sum_{j = 1}^p \sum_{k=0}^K   \log D_{k(j,j)} 
	- \sum_{j = 1}^p \Big( \log \sum_{k=0}^K D_{k(j,j)} \Big) \notag \\
	& - \sum_{j = 1}^p \sum_{l > m \geq 1} \bigg( M_{(l,m)(j,j)} \frac{D_{l(j,j)} D_{m(j,j)}} 
	{\sum_{k=0}^K D_{k(j,j)}} \bigg) 
	\ - \sum_{j = 1}^p \sum_{l=1}^K \bigg( 
	\frac{\hat{\Sigma}_{Y(l, l)(j,j)} D_{0(j,j)} D_{l(j,j)}}
	{\sum_{k=0}^K D_{k(j,j)}} \bigg) \; ,
\end{align}
where $\hat{\Sigma}_{Y(l, l)(i,j)}$ and $M_{(l,m)(i,j)}$ represent the entry in $i$th row and $j$th column of the matrix $\hat{\Sigma}_{Y(l, l)}$ and $M_{(l,m)}$, respectively. 

Next, we want to show that the diagonal entries of $M_{(l,m)}$ are positive. By definition, we know that
\begin{align*}
	M_{(l,m)(j, j)} & = \hat{\Sigma}_{Y(l, l)(j, j)} + 
	\hat{\Sigma}_{Y(m, m)(j, j)} - 2 \hat{\Sigma}_{Y(m, l)(j, j)} \\
	& = \sum_{i = 1}^n \big( y_{l,i, j}^2 
	+ y_{m,i,j}^2 - 2 y_{m, i, j} y_{l, i, j} \big) / n \\
	& = \sum_{i = 1}^n  (y_{l,i, j} - y_{m,i, j})^2/n \geq 0,
\end{align*}
where $M_{(l, m)(j, j)} = 0$ if and only if $y_{l,i,j} = y_{m,i,j}$ for all $i = 1, \ldots, n$.

Under Condition 1, we have
\[
	\text{corr}(Y_{l,i,j}, Y_{m,i,j}) 
	= \frac{\text{var}(Z_{i,j})}{\text{var}(X_{l,i,j} + Z_{i,j})\text{var}( X_{m, i, j} + Z_{i,j})} \neq 1.
\] 
Therefore, we have $\sum_{i = 1}^n  (y_{l,i, j} - y_{m,i, j})^2/n > 0$ with probability 1, which implies that the diagonal entries of $M_{(l,m)}$ are positive. 

For a specific $j \in (1, ...., p)$, the only positive term is $\sum_{k=0}^K \log D_{k(j,j)} $. Thus, if we could bound it with the remaining terms in \eqref{eq:target4}, we complete the proof. 

Without loss of generality, we assume $D_{s(j, j)}$ and $D_{r(j, j)}$ have the highest and second highest rates of those positive terms. Then we have that the rate of $M_{(s,r)(j, j)} D_{s(j, j)} D_{r(j, j)}/ \{\sum_{k=0}^K D_{k(j, j)}\}$ equals to $ M_{(s,r)(j, j)} D_{r(j, j)}$. Since $M_{(s,r)(j, j)} > 0$, if $D_{r(i,i)} \rightarrow \infty$ we have 
\[ 
	\log D_{r(j, j)}- M_{(s,r)(j, j)}
	\frac{ D_{s(j, j)} D_{r(i,i)}}{\sum_{k=0}^K D_{k(i,i)}} \; = \left\{ \begin{array}{ll}
         \rightarrow -\infty  & \mbox{ if $ D_{r(i,i)} \rightarrow \infty$ };\\
        < \infty & \mbox{if $ D_{r(i,i)}$ is bounded}.
        \end{array} \right. 
\] 
If the second highest rate for the positive term is $D_{0(i,i)}$, we can simply replace $M_{(s,r)(i,i)}$ with $\hat{\Sigma}_{Y(s, s)(i,i)}$, and the proof can also be carried out. Combining the fact that $\log D_{s(j, j)} - \log (\sum_{k=0}^K  D_{k(j,j)}) < 0$ and \eqref{eq:target4} is bounded above, the proof of the first part is completed.

For part II, we will show that the penalized log-likelihood does not decrease for each step of our EM algorithm. For simplicity, we write $\bO$ for $\{\bO_k \}_{k=0}^K$ in the following derivation. 

Given $y = (y_{\cdot,1} , \ldots, y_{\cdot, n})^\T$ and $z = (z_1, \ldots, z_n)^\T$, the full log-likelihood is
\begin{align*}
	\mathcal{L}(\bO; y, z)
	\propto & \; \log \det(\bO_0) - 
	\tr\Big\{\bO_0 \sum_{i = 1}^n (z_i^\T  z_i) /n \Big\} \\
	& \; + \sum_{k =1}^K \Big( \log  \det(\bO_k) 
	- \tr\big[\bO_k \sum_{i = 1}^n \big\{ (y_{k, i} - z_i)^\T (y_{k,i} - z_i)\big\} /n \big] \Big).
\end{align*}
The above log-likelihood cannot be calculated directly because the values of $z$ and $zz^\T$ are unobserved. However, we can calculate the following function $\mathcal{Q}(\bO ; \bO^{(t)}, y)$, in which $z$ and $zz^\T$ are replaced by their expected values conditional on $\bO$ and $y$. We define 
\begin{align}
	\mathcal{Q}(\bO; \bO^{(t)}) & = E_{Z|\bO^{(t)}} \big\{ \mathcal{L}(\bO; y, z) \big\} \notag \\
	& = E_{Z|\bO^{(t)}} \Big\{ \log  f(y, z; \bO)  \Big\} \notag \\
	& = E_{Z|\bO^{(t)}} \Big\{ \log  f(y; \bO)  + \log  f(z \mid y, \bO)  \Big\} \notag \\
	& = \log  f(y; \bO)  + E_{Z|\bO^{(t)}} \Big\{ \log f(z \mid y, \bO)  \Big\},  \label{eq:EM_prof1}
\end{align}
where $f(y; \bO)$ and $f(z \mid y, \bO)$ are the probability density functions for $y$ and  $(y, z)$, respectively. The equality in \eqref{eq:EM_prof1} is because  the expectation is over the values of $z$, and $\log  f(y; \bO)$ is a constant with respect to the expectation since $y$ is observed.
Based on \eqref{eq:EM_prof1}, we have 
\begin{align*}
	\log  f(y; \bO) - Pen(\bO)
	= \mathcal{Q} \big( \bO; \bO^{(t)} \big) 
	- E_{Z|\bO^{(t)}} \big\{ \log f(z \mid y, \bO)  \big\}-  Pen(\bO),
\end{align*}
where $Pen(\bO)$ is the penalty function $\lambda_1 \sum_{k = 1}^K |\bO_k^-|_1 + \lambda_2 |\bO_0^-|_1$. 

The M step in the EM algorithm is to update $ \bO^{(t)} \rightarrow \bO^{(t + 1)}$ through
\begin{align}\label{eq:profP1Mstep}
	\bO^{(t + 1)} 
	= \argmax_{ \bO} \mathcal{Q}(\bO; \bO^{(t)}) - Pen(\bO). 
\end{align}
Comparing the penalized log-likelihoods for steps $t$ and $t + 1$, we have
\begin{align*}
	& \; \log f(y; \bO^{(t + 1)}) - Pen\big(\bO^{(t + 1)}\big) 
	- \log f(y; \bO^{(t)})  + Pen(\bO_k^{(t)}) \\
	= & \; \mathcal{Q}\big(\bO^{(t + 1)}; \bO^{(t)} \big) - Pen( \bO^{(t + 1)})
	- E_{Z|\bO^{(t)}} \big\{ \log f(z \mid y,  \bO^{(t + 1)}) \big\} \\
	& \; - \mathcal{Q}\big(\bO^{(t)}; \bO^{(t)}\big)+ Pen( \bO^{(t)})
	 + E_{Z|\bO^{(t)}} \big\{ \log  f \big( z \mid y, \bO^{(t)}\big)  \big\} .
\end{align*}
By \eqref{eq:profP1Mstep}, it follows that 
\begin{align}
	\mathcal{Q} \big( \bO^{(t + 1)}; \bO_k^{(t)} \big) - Pen(\bO^{(t + 1)}) 
	- \mathcal{Q}\big(\bO^{(t)}; \bO^{(t)} \big) + Pen( \bO^{(t)}) \geq 0, \label{eq:EMtermI}
\end{align}
since $\bO^{(t +1)}$ is the maximizer over the term $\mathcal{Q}(\bO; \bO^{(t)}) - Pen(\bO)$.
Also, using the Gibbs' inequality, we have
\begin{align}\label{eq:Gibbs}
  - E_{Z|\bO^{(t)}} \big\{ \log  f(z \mid y, \bO^{(t + 1)})  \big\}
  + E_{Z|\bO^{(t)}} \big\{ \log f(z \mid y, \bO^{(t)}) \big\} \geq 0. 
\end{align}
Combining  \eqref{eq:EMtermI} and \eqref{eq:Gibbs}, we have 
\begin{align}
	\log  f(y; \bO^{(t + 1)})  - Pen( \bO^{(t + 1)}) 
	- \log f(y; \bO^{(t)}) 
	+ Pen(\bO^{(t)}) \geq 0, \notag
\end{align}
which completes the proof for part II.

Let $d$ be the upper bound of the penalized log-likelihood $\mathcal{P}(\bO; y)$ and $\delta$ be a prespecified threshold. Then for at most $\ceil[\Big]{ \big\{ d - \mathcal{P}(\bO^{(0)}; y)\big\} / \delta} $ steps, there are two consecutive steps $t$ and $t + 1$ satisfying 
\[
	\big| \mathcal{P}( \bO^{(t+1)}; y) 
	- \mathcal{P}(\bO^{(t)}; y) \big| < \delta.
\]
This completes the proof.

\section{Proof of Theorem \ref{thm:Thm1}}

In this proof, we need to use Lemma 3 of \citet{Bickel2008}. We state the result here for completeness.
\begin{lemma}\label{Lemma.A2}
Let $Z_i$ be independent and identically distributed from  $\mathcal{N}(0, \Sigma_p)$ and $\phi_{max}(\Sigma_p) \leq \bar{k} < \infty$. Then , if $\Sigma_p = \{ \sigma_{ab} \}$,
	\begin{equation*}
		\emph{pr}\Big( \Big|\sum\limits_{i=1}^n(Z_{ij}Z_{ik} 
		- \sigma_{jk}) \Big|
		\geq n\nu \Big) 
		\leq C_1 \exp(-C_2n\nu^2), \quad \text{for} \quad |\nu| \leq \delta \,,
	\end{equation*}
where $C_1, C_2$ and $\delta$ depend on $\bar{k}$ only.
\end{lemma}
We first show that $\phi_{max}(\Sigma_Y^*)$ is bounded above. Let $v = (v_0^\T, \ldots, v_K^\T)^\T \in \mathbb{R}^{(K + 1)p}$ and $v^\T v = 1 $. Under Condition 1, we have
\begin{align*}
	v^\T \Sigma_Y^* v & =  \sum_{k = 1}^K \big( v_k^\T \Sigma_k^* v_k \big)
	+ \big(\sum_{k = 1}^K v_k^\T \big) \Sigma_0^*
	\big(\sum_{k = 1}^K v_k \big) \\
    & \leq K \tau_2 +  \big(\sum_{k = 1}^K v_k\big)^\T  
    \big( \sum_{k = 1}^K v_k \big) 
    \frac{\big(\sum_{k = 1}^K v_k^\T\big)
    \Sigma_0^*\big(\sum_{k = 1}^K v_k \big)}
    { \big(\sum_{k = 1}^K v_k\big)^\T  
    \big( \sum_{k = 1}^K v_k \big) } \\
    & \leq K\tau_2 + \big(\sum_{k = 1}^K v_k\big)^\T 
    \big(\sum_{k = 1}^K v_k \big)\tau_2 \notag \\
    & = K\tau_2  + \tau_2 \| \sum_{k = 1}^K v_k \|^2_2 \\
    &\leq K\tau_2  + \tau_2 \Big(\sum_{k = 1}^K \| v_k \|_2\Big)^2
    \leq (K + K^2)\tau_2  < \infty \notag,
\end{align*}
where $\|\cdot \|_2$ is the vector Euclidean norm.
 
To estimate $\bO_{k}$, we need to minimize \eqref{eq:onestep} with $\hat{\Sigma}_{k}^{\prime}$ being the only input. First, we would bound $\| \hat{\Sigma}_{k} - \Sigma^{*}_{k} \|_\infty$. Let $\hat{\Sigma}_{0} =\sum_{i = 1}^n y_{l,i} y_{m,i}^\T / n $ for some $l \neq m $, and $\hat{\Sigma}_{k} = \sum_{i = 1}^n y_{k,i} y_{k,i}^\T / n - \hat{\Sigma}_{0}$ $\; (k = 1...K )$. Using the union sum inequality and Lemma \ref{Lemma.A2}, we have
\begin{align*}
	& \; \text{pr} \Big( \max_{1 \leq i,j \leq p}|\hat{\sigma}_{0(i,j)}
	- \sigma^{*}_{0(i,j)}| 	\geq C_3 \{(\log p) / {n} \}^{1/2} \Big) \\
	 = & \; \text{pr} \bigg( \bigcup_{1 \leq i,j \leq p}
	\left[ |\hat{\sigma}_{0(i,j)} - \sigma^{*}_{0(i,j)}| 
	\geq C_3 \{ (\log p) / {n} \}^{1/2} \right] \bigg) \\[6pt] 
	\leq & \; \sum_{1 \leq i,j \leq p} 
	\text{pr} \Big( |\hat{\sigma}_{0(i,j)} - \sigma^{*}_{0(i,j)}|  
	\geq C_3 \{ (\log p) / {n} \}^{1/2} \Big) \\ 
	\leq & \;  p^2 C_1 \exp \{ -C_2 n C_3^2 (\log p)/n \} \\
	= & \; C_1 p^{2-C_3^2C_2} \rightarrow 0\,,
\end{align*}
for any sufficiently large $C_3$. Therefore, with probability tending to 1,
\[
   \| \hat{\Sigma}_0 - \Sigma^*_0 \|_\infty 
   \leq C_3 \{ ( \log p) / n \}^{1/2}.
\]
Similarly, we have
\begin{align*}
 	\| \hat{\Sigma}_{0}	+ \hat{\Sigma}_k - \Sigma_0^* 
 	- \Sigma^*_k \|_\infty
 	\leq C_4 \{ (\log p) / {n} \}^{1/2} \quad (k = 1, \ldots, K).
\end{align*}
Together with the triangle inequality, this implies
\begin{align*}
	\| \hat{\Sigma}_k - \Sigma^{*}_k \|_{\infty} 
	\leq (C_3 + C_4) \{(\log p) / {n} \}^{1/2}.
\end{align*}
Thus, $\| \hat{\Sigma}_{k} - \Sigma^{*}_{k} \|_{\infty} = O_P\big[ \{ (\log p) / {n} \}^{1/2} \big]$  $( k = 0, \ldots, K)$. 
The same rate can also be derived for 
\begin{align*}
	\hat{\Sigma}_{0} = \sum_{m \neq l}\sum_{i = 1}^n 
	 \frac{ y_{m,i} y_{l,i}^\T}{  K (K -1)n}, \quad
	\hat{\Sigma}_{k} = \sum_{i = 1}^n \frac{ y_{k,i}  y_{k,i}^\T}{ n} 
	- \hat{\Sigma}_{0}, 
\end{align*}
following the same proof strategy.

Next, we will bound $\| \hat{\Sigma}_k^\prime - \Sigma^{*}_{k} \|_\infty$ $\ (k = 0, \ldots, K)$. By the triangle inequality and the definition of projection in \eqref{eq:Projection}, we have
\begin{align*}
	\|\hat{\Sigma}_{k}^{\prime} - \Sigma^*_k \|_\infty
	= & \;  \|\hat{\Sigma}_{k}^{\prime} - \hat{\Sigma}_{k} 
	+ \hat{\Sigma}_{k} - \Sigma^{*}_{k}\|_\infty\\
	\leq & \; \|\hat{\Sigma}_{k}^{\prime} - \hat{\Sigma}_{k}\|_\infty 
	+ \| \Sigma^*_{k} - 	\hat{\Sigma}_{k} \|_\infty \\
	\leq & \; 2 \| \Sigma^*_k - \hat{\Sigma}_k \|_\infty. 
\end{align*}
Thus, we have $\| \hat{\Sigma}_{k}^{\prime} - \Sigma^{*}_{k}\|_{\infty} = 
O_P\big[ \{ (\log p) / {n} \}^{1/2} \big]$ $\ (k = 0, \ldots, K)$. 

For simplicity, we will write $\bO = \bO_{k}$, $\bO^{*} = \bO^{*}_k$, $ \hat{\Sigma}^{\prime} = \hat{\Sigma}_{k}^{\prime}$ and $\Delta = \Delta_{k}$, where $\Delta_{k} = \Omega_{k} - \Omega^{*}_{k}$ $\ (k = 0, \ldots, K)$ and $\lambda = \lambda_1$ or $\lambda_2$.  Let $\hat{\bO}$ be our estimate minimizing \eqref{eq:onestep} and define $\mathcal{V}(\bO)$ as the normalized function from equation \eqref{eq:onestep} with
\begin{align} \label{eq:Q}
	\mathcal{V}(\Omega) = & \; \tr(\bO \hat{\Sigma}^{\prime}) 
	- \log  \det{(\bO)}  + \lambda | \bO^{-} |_1 
	- \tr(\bO^{*}\hat{\Sigma}^{\prime}) + \log  \det(\bO^{*}) 
	- \lambda | \bO^{*-} |_1 \notag \\ 
	= & \; \tr\{(\bO - \bO^{*})(\hat{\Sigma}^{\prime} - \Sigma^*)\} 
	- \big\{ \log \det{(\bO)} - \log \det{(\bO^{*})}  \big\} \notag \\ 
	& \; + \tr\{(\bO - \bO^{*}) \Sigma^{*}\} + \lambda( | \bO^{-} |_1 - |\bO^{*-} |_1) .
\end{align}

For the one-step algorithm, our estimate $\hat{\bO}$ minimizes $\mathcal{V}(\bO)$.  Since $\mathcal{V}(\bO)$ is also a function of $\Delta$, we define $ \mathcal{G}(\Delta) = \mathcal{V} (\bO^{*} + \Delta)$. It can be checked that $\mathcal{G}(0) = 0$, and that $\hat{\Delta} = \hat{\bO} - \bO^{*}$ minimizes the function $ \mathcal{G}(\Delta)$. The main idea of the proof is as follows: we first define a closed bounded convex set $\mathcal{A}$ including $0$, and show that $\mathcal{G} > 0$ on the boundary of $\mathcal{A}$. Because $\mathcal{G}$ is continuous and $\mathcal{G}(0) = 0$, it implies that the solution minimizing $\mathcal{G}$ is inside $\mathcal{A}$. Let
\begin{align} \label{eq:set}
	\mathcal{A} = \{\Delta: \Delta = \Delta^\T, \| \Delta \|_F \leq M r_n \} \, , \notag \\
	\partial \mathcal{A} = \{ \Delta: \Delta = \Delta^\T, \| \Delta \|_F = M r_n \} \, ,
\end{align}
where $M$ is a positive constant and $ r_n = \left\{ (p + q) (\log p) / n \right\}^{1/2} \rightarrow 0.$

Using the Taylor expansion of $f(t) = \log \det (\bO + t \Delta)$ and the fact that $\Delta$, $\Sigma^{*}$, and $\bO^{*}$ are all symmetric, we have
\begin{align*}
	& \; \log  \det(\bO^{*} + \Delta)  - \log \det(\bO^{*}) \\
	= & \; \tr(\Sigma^{*} \Delta) 
	- \text{vec}(\Delta^\T) \left\{ \int_0^1 (1-v)(\bO^{*} 
	+ v \Delta)^{-1} \otimes (\bO^{*} + v \Delta)^{-1} \mathrm{d} v\right \} 
	\text{vec}(\Delta)\, ,
\end{align*}
where $\otimes$ is the Kronecker product and $\text{vec}(\cdot)$ returns the vectorization of a matrix. Thus, 
\begin{align} \label{eq: G}
	\mathcal{G}(\Delta) = & \; \tr \{\Delta (\hat{\Sigma}^{\prime} - \Sigma^{*})\} 
	+ \text{vec}(\Delta^\T) \left\{ \int_0^1 (1-v)(\bO^{*} 
	+ v \Delta)^{-1} \otimes (\bO^{*} + v \Delta)^{-1}  \mathrm{d}v \right\} 
	\text{vec}(\Delta) \notag \\
	& \; + \lambda (|\bO^{*-} + \Delta^{-} |_1 - |\bO^{*-} |_1) = \text{I} + \text{II} + \text{III}\,.
\end{align}

To show that $ \mathcal{G}(\Delta)$ is strictly positive on $ \partial \mathcal{A}$, we need to bound $\text{I}, \text{II}$ and $\text{III}$. First, using the symmetry arguments and the triangular inequality, we can bound I as 
\begin{align*}
	|\tr\{\Delta (\hat{\Sigma}^{\prime} - \Sigma^*)\}| 
	 = & \; |\text{vec}(\Delta)^\T 
	 \text{vec}(\hat{\Sigma}^{\prime} - \Sigma^*)| 
	= \Big|\sum_{i,j} \big\{ \delta_{ij}(\hat{\sigma}_{ij}^{\prime}
	- \sigma^{*}_{ij}) \big\} \Big| \\ 
	\leq & \; \Big| \sum_{i\neq j} \big\{
	\delta_{ij}(\hat{\sigma}_{ij}^{\prime} - \sigma^{*}_{ij}) \big\} \Big|
	+ \Big| \sum\limits_{i = 1}^p \big\{ \delta_{ii} 
	(\hat{\sigma}_{ii}^\prime - \sigma^{*}_{ii}) \big\} \Big| \\
	 = & \; \text{I}^{\prime} + \text{II}^{\prime}.
\end{align*}
As discussed above, with probability tending to 1,
\begin{align*}
	\max_{i \neq j} |\hat{\sigma}_{ij}^\prime - \sigma^{*}_{ij}| 
	\leq  \| \hat{\Sigma}^\prime - \Sigma^{*} \|_{\infty} 
	\leq 2(C_3 + C_4) \{ (\log p) / n \}^{1/2} \,,
\end{align*}
and hence term $\text{I}^\prime$ is bounded by
\begin{align} \label{eq:I'}
	\text{I}^{\prime} 
	\leq | \Delta^- |_1 \max_{i \neq j} |\hat{\sigma}_{ij}^\prime - \sigma^{*}_{ij}|
 	\leq 2(C_3 + C_4) \{(\log p) / n \}^{1/2} | \Delta^- |_1.
\end{align}
Using the Cauchy-Schwartz inequality and Lemma \ref{Lemma.A2}, we bound the term $\text{II}^{\prime}$ with probability tending to 1 as
\begin{align} \label{eq:II'}
	\text{II}^{\prime} 
	\leq & \; \Big\{ \sum\limits_{i=1}^p (\hat{\sigma}_{ii}^\prime 
	- \sigma^{*}_{ii})^2 \Big\}^{1/2} \| \Delta^+ \|_F 
	\leq p^{1/2} \max_{1 \leq i \leq p} |\hat{\sigma}_{ii}^\prime - \sigma^{*}_{ii}| \, 
	\| \Delta^+ \|_F \notag \\[6pt]
	\leq & \;  2(C_3 + C_4) \{ p (\log p) / n \}^{1/2} \| \Delta^ + \|_F 
	\leq 2(C_3 + C_4) \{(p + q) (\log p) / n \}^{1/2} \| \Delta^+ \|_F,
\end{align}
where $\Delta^+$ is the digonal entries of $\Delta$.

To bound II, we use the results established in \citet[Theorem 1]{Rothman2008}:
\begin{align}\label{eq:boundII}
	\text{vec}(\Delta^\T) \left\{ \int_0^1 (1-v)(\bO^{*} 
	+ v \Delta)^{-1}\otimes (\bO^{*} 
	+ v \Delta)^{-1} \mathrm{d} v \right\} \text{vec}(\Delta) 
	\geq  \| \Delta \|_F^2 / (4\tau_2^2) \,.
\end{align}
Lastly, we would bound III. For an index set $B$ and a matrix $M = ( m_{ij} )$, we define $ M_B \equiv \{ m_{ij}:  (i,j) \in B \}$. Recall that  $\text{T} = \{(i, j): i \neq j, \omega^*_{i,j} \neq 0 \}$, and let $\text{T}^c$ be its complement.  Using the triangular inequality and the facts that $|\bO^{*-} |_1 = |\bO_{\text{T}}^{*-} |_1$ and $|\bO^{*-} + \Delta^- |_1 = |\bO_{\text{T}}^{*-} + \Delta_{\text{T}}^{-} |_1 + | \Delta_{\text{T}^c}^{-} |_1$, we have 
\begin{align}\label{eq:boundIII}
	\lambda( | \bO^{*-} + \Delta^{-} |_1 - | \bO^{*-} |_1) 
	\geq \, \lambda \,( | \Delta_{\text{T}^c}^{-} |_1 
	- | \Delta_{\text{T}}^{-} |_1) \, .
\end{align}

Combining  \eqref{eq:I'} -- \eqref{eq:boundIII}, we can show
\begin{align*}
	 \mathcal{G}(\Delta) \, 
	\geq & \; \| \Delta \|_F^2/(4\tau^2_2) 
	- 2(C_3 + C_4) \{ (\log p) / n \}^{1/2} | \Delta^- |_1 \\
	& \; - 2(C_3 + C_4) \big\{(p + q) (\log p) / n \big\}^{1/2} \| \Delta^+ \|_F 
	+ \lambda \, ( | \Delta_{\text{T}^c}^- |_1
	- |\Delta_{\text{T}}^{-} |_1) \\ 
	= & \; \| \Delta \|_F^2 / (4\tau^2_2) 
	+ \Big[ \lambda - 2(C_3 + C_4) \{(\log p) / n \}^{1/2} \Big] | 
	\Delta^-_{\text{T}^c} |_1 \\
	& \; - \Big[ 2(C_3 + C_4) \{ (\log p) / n \}^{1/2} + \lambda \Big] \;
	|\Delta_{\text{T}}^{-} |_1 	
 	 - 2(C_3 + C_4) \{ (p + q) (\log p) / n \}^{1/2} \| \Delta^+ \|_F \\
	\geq & \;  \| \Delta \|_F^2 / (4\tau^2_2) 
	+ (a_1 - 2C_3 - 2C_4)\{ (\log p) / n\}^{1/2} |\Delta^-_{\text{T}^c}|_1 \\
	& \; - \left[ 2(C_3 + C_4) \{ (\log p) / n\}^{1/2} 
	+ b_1 \big\{ (1 + p/q) ( \log p) /n \big\}^{1/2} \right] 
	|\Delta_{\text{T}}^- |_1 \\
	& \; - 2(C_3 + C_4) \big\{ (p + q) (\log p) / n \big\}^{1/2} \| \Delta^+ \|_F,
\end{align*}
where the last inequality uses the condition  $a \{ (\log p) /n  \}^{1/2} \leq \lambda \leq  b \{ (1 + p/q) (\log p) / n\}^{1/2}$. When $a$ is large enough, the term $(a - 2C_3 -2C_4)\{ (\log p) / n\}^{1/2} | \Delta^-_{\text{T}^c}|_1 $ is always positive. Using the Cauchy-Schwartz inequality, we have
\begin{align} \label{eq:l1}
	| \Delta_{\text{T}}^- |_1 
	\leq \sqrt{q} \| \Delta_{\text{T}}^-  \|_F 
	\leq \sqrt{q} \| \Delta^- \|_F \leq \sqrt{q}\| \Delta \|_F \,.
\end{align}
Therefore, we have
\begin{align}
	\mathcal{G}(\Delta) 
	\geq & \;  \| \Delta \|_F^2 / (4\tau^2_2)  
	- \{ 2q^{1/2}(C_3 + C_4) + b_1 (p + q)^{1/2} \} \{ (\log p) / n \}^{1/2} \| \Delta \|_F \notag \\
	& \; - (2C_3 + 2C_4) \{ (p + q) (\log p) / n \}^{1/2} 
	\| \Delta \|_F \notag \\
	\geq & \; \| \Delta \|_F^2 \left[ 1 / (4\tau^2_2) 
	- (4C_3 + 4C_4 + b_1) \{(p + q)  (\log p) 
	/ n \}^{1/2} \| \Delta \|_F^{-1} \right]. \label{eq:thm1f2}
\end{align}
For $\Delta \in \partial \mathcal{A}$, where $\partial \mathcal{A} = \{ \Delta: \Delta = \Delta^\T, \| \Delta \|_F = M r_n \}$ and $r_n = \{(p + q)  (\log p) / n \}^{1/2}$, we have $\|\Delta \|_F^{-1} \{(p + q) (\log p) / n \}^{1/2}   = 1/M$. Plugging it into \eqref{eq:thm1f2}, we have
\begin{align*}
	\mathcal{G}(\Delta) \geq \|\Delta \|_F^2 \left[ 1/ (4\tau^2_2) 
	- (4C_3 + 4C_4 + b_1)/M \right] > 0,
\end{align*}
for sufficiently large $M.$
Since $\mathcal{G}$ is continuous and $\mathcal{G}(0) = 0$, with the fact that $\mathcal{G} >0 $  on $\partial \mathcal{A}$, it implies that $\hat{\bO}$ is inside $\mathcal{A}$. Therefore, we have $\| \hat{\Omega} - \Omega^* \|_F \leq M r_n$ and $\| \hat{\Omega} - \Omega^* \|_F = O_p (r_n) = O_p(\left\{ {(p + q) (\log p)}/ {n} \right\}^{ 1/2})$. This completes the proof.

\section{Proof of Corollary \ref{cor:cor2}}

First, we state a known matrix result and provide a short proof for completeness.
\begin{lemma}
\label{Lemma.A4}
	Let $F$ be any $p \times p$ matrix with $\| F \| <1$. Then $(I_p - F)^{-1} = \sum_{k=0}^{\infty} F^k$, and
\begin{align}
	\| (I_p - F)^{-1} \|
	\leq \frac{1}{1 - \| F \|}. \notag
\end{align}
\end{lemma}
\begin{proof}
By direct calculation,  
\begin{align} \label{eq:lem3}
	\Big(\sum_{k=0}^N F^k \Big) (I_p - F) = I_p - F^{N + 1}.
\end{align} 

Since $\| F^k \| \leq \| F \|^k$ and $\| F \| <1$, we have $ \| F^k \|_\infty \leq \| F^k \| \rightarrow 0 $ as $k \rightarrow \infty$. As a result, taking limit on both sides of \eqref{eq:lem3} we have
\[
	\lim_{N \rightarrow \infty}
	\Big\{ \Big( \sum_{k=0}^N F^k \Big)(I_p - F) \Big\}	= I_p,
\] 
and thus $(I_p - F)^{-1} = \sum_{k=0}^{\infty} F^k$. Consequently, we have
\[
	\| (I_p - F)^{-1} \| 
	= \| \sum_{k=0}^{\infty} F^k \| 
	\leq  \sum_{k=0}^{\infty} \| F^k \|  
	\leq \sum_{k=0}^\infty \| F \|^k 
	\leq  \frac{1}{1 - \| F \|} \,.
\]
\end{proof}

In this proof, we also need the Lemma 1 from \citet{Lam2009}, and thus we state the result here for completeness.
\begin{lemma} \label{Lemma.A5}
	Let $A$ and $B$ be real matrices such that the product $AB$ is defined. Then we have
\begin{align}
\|AB \|_F \leq \| A \| \| B \|_F. \notag
\end{align}
In particular, if $A = \{a_{ij}\}$, then $|a_{ij} | \leq \| A \|$ for each $i, j$. When both $A$ and $B$ are symmetric matrices,  we also have
\begin{align}
	\| AB \|_F = \| B^\T A^\T \|_F = \|BA \|_F  \leq \| B \| \, \| A \|_F. \notag
\end{align}
\end{lemma}

Let $\hat{\bO}_k = \bO_k^* + \Delta_k$ be the one-step solution. By Theorem \ref{thm:Thm1}, we have
\begin{align}
	\|\hat{\bO}_k - \bO^*_k \|_F 
	= \|\Delta_k \|_F
	=  O_p \left[ \left\{ \frac{(p+q)\log p}{n} \right\}^{1/2} \right]. \notag
\end{align}
Using the Woodbury matrix identity twice, we have
\begin{align}
	\check{\Sigma}_k = & \; (\bO^*_k + \Delta_k )^{-1}
	= \Sigma^*_k 
	- \Sigma^*_k( \Delta_k^{-1} + \Sigma^*_k )^{-1} \Sigma^*_k \notag \\
	= & \; \Sigma^*_k 
	- \Sigma^*_k( \Delta_k - \Delta_k(\bO^*_k + \Delta_k )^{-1} \Delta_k)
	\Sigma^*_k \notag \\
	= & \; \Sigma^*_k - \Sigma^*_k \Delta_k \Sigma^*_k 
	+ \Sigma^*_k \Delta_k (\Delta_k + \bO_k^*)^{-1} \Delta_A \Sigma^*_k . \notag
\end{align}
By Condition 1, we have $ 
{\tau_2}^{-1}<\phi_{min}(\Sigma_k^*) <\phi_{max}(\Sigma_k^*) < {\tau_1}^{-1}
$. Using Lemmas \ref{Lemma.A4}--\ref{Lemma.A5}, we have
\begin{align}
	\|\check{\Sigma}_k - \Sigma^*_k \|_F 
	\leq & \;  \| \Sigma^*_k \Delta_k \Sigma^*_k \|_F 
	+ \|\Sigma^*_k \Delta_k (\Delta_k + \bO_k^*)^{-1} \Delta_k \Sigma^*_k \|_F \notag \\
	\leq & \; \| \Sigma_k^* \|^2 \| \Delta_k \|_F 
	+ \|\Sigma_k^* \|^2 \| \Delta_k \|_F^2 \| (\Delta_k + \bO_k^*)^{-1} \| \notag \\
	\leq & \; \| \Delta_k \|_F / (\tau_1^2) 
	+ \| \Delta_k \|^2_F \| (I_p
	+ \Sigma_k^* \Delta_k )^{-1} \| \, \|\Sigma_k^* \| / (\tau_1^2)  \notag \\
	\leq & \; \|\Delta_k \|_F / (\tau_1^2) 
	+ \frac{1}{\tau_1^3} \| \Delta_k \|^2_F 
	(1 - \| \Sigma_k^* \Delta_k\|)^{-1} \label{eq:cor1f1} \\
	\lesssim & \; \|\Delta_k \|_F / (\tau_1^2) 
	+ \frac{2}{\tau_1^3} \| \Delta_k \|^2_F 
	 \label{eq:cor1f2} \\
	= & \;  O_p \left[ \left\{ \frac{(p+q)\log p}{n} \right\}^{1/2} \right] \notag.
\end{align}
The inequality of $\eqref{eq:cor1f1}$  is due to Lemma \ref{Lemma.A4} since  $\| \Sigma_k^* \Delta_k \| <1 $ when $n$ is large enough, and the inequality of \eqref{eq:cor1f2} holds when $n$ is large enough since
\begin{align}
	\|\Sigma_k^* \Delta_k \| \leq \| \Sigma_k^* \Delta_k \|_F 
	\leq \|\Sigma_k^*  \| \; \| \Delta_k \|_F
	\leq  \frac{1}{\tau_1^2} \| \Delta_k \|_F  \rightarrow 0 \,. \notag
\end{align}
The proof is complete. 
\section{Proof of Theorem \ref{thm:Thm2}}

To prove Theorem \ref{thm:Thm2}, we use the following lemma.

\begin{lemma}
\label{Lemma.A3}
Suppose that Conditions 1-2 hold, $(p+q)(\log  p)/n = o(1)$, and  $\| \tilde{\Sigma}_{k} - \Sigma^{*}_{k} \|_F  = O_p\left[ \left\{ (p+q) (\log p)/ n \right\}^{1/2} \right]$.  Let 
\begin{align*}\hspace{-0.1in}
	\tilde{\bO}_k = \argmax_{\bO_k \succ 0} \mtr (\tilde{\Sigma}_{k} \Omega_{k}) - \log \det(\Omega_{k}) 
	+ \lambda |\bO_k^- |_1 \; \; \quad (k = 0, \ldots, K),
\end{align*} where $\lambda = \lambda_2$ when $k = 0$, and otherwise $\lambda = \lambda_1$. Then we have
\[ \sum\limits_{k=0}^K  \left\| \tilde{\bO}_{k} - \bO^{*}_{k} \right\|_F = O_p\left[ \left\{ \frac{(p+q) \log p}{n} \right\}^{1/2} \right].
\]
\end{lemma}

\begin{proof}
This proof is analogous to the proof of Theorem \ref{thm:Thm1}. Define $\mathcal{G}(\Delta)$  as in \eqref{eq:Q} and a closed bounded convex set $\mathcal{A}$ as \eqref{eq:set}. We only need to show $\mathcal{G}(\Delta)$ is strictly positive on $\partial \mathcal{A}$. We can write $\mathcal{G}(\Delta) = \text{I} + \text{II} + \text{III}$ as in \eqref{eq: G}. Using matrix symmetry and the Cauchy-Schwarz inequality, we can bound I as
\begin{align*}
	|\tr \{ \Delta (\tilde{\Sigma} - \Sigma^{*})\} | 
	= & \big| \sum_{i,j = 1}^p \big\{  
	\delta_{ij} (\tilde{\sigma}_{ij} - \sigma_{ij}^*) \big\} \big| 
	\leq  \| \tilde{\Sigma} - \Sigma^{*}  \|_F \, \| \Delta \|_F\\ 
	 = & D_1 \left\{(p+q) (\log p)/ n \right\}^{1/2} \| \Delta \|_F,
\end{align*}
where $D_1$ is some constant.

For II and III, they have the same bound as \eqref{eq:boundII} and \eqref{eq:boundIII}, respectively. We can show
{\allowdisplaybreaks
\begin{align}
	\mathcal{G}(\Delta) 
	\geq & \; \| \Delta \|_F^2/(4\tau^2_2) 
	- D_1 \left\{ (p+q) (\log p) / n \right\}^{1/2} \| \Delta \|_F
	+ \lambda(| \Delta_{\text{T}^c}^- |_1 - |\Delta_\text{T}^-|_1) \notag \\[6pt]
	\geq & \; \| \Delta \|_F^2 / (4\tau^2_2)  
	- D_1 \left\{(p+q) (\log p)/ n \right\}^{1/2} \| \Delta \|_F  
	- \lambda  | \Delta_{\text{T}}^-|_1 \notag \\[6pt]
	\geq & \; \| \Delta \|_F^2 / (4\tau^2_2) 
	- D_1 \left\{(p+q) (\log p) / n \right\}^{1/2} \| \Delta \|_F \notag \\
	& \; - b \left\{(p+q)(\log p) / n \right\}^{1/2} \| \Delta \|_F \label{eq:Lemma2} \\[6pt]
	= & \; \| \Delta \|_F^2 \left[ \frac{1}{4\tau^2_2} 
	-(D_1 +  b ) \left\{(p+q) (\log p) / n \right\}^{1/2} 
	\| \Delta \|_F^{-1} \right] \notag \\
  = & \; \| \Delta \|_F^2 \left[ \frac{1}{4\tau^2_2} -(D_1 + b )/M  \right] > 0, \notag
 \end{align}}
for sufficiently large $M$ defined in \eqref{eq:set}. The inequality of \eqref{eq:Lemma2} uses the result of \eqref{eq:l1} and the fact that $ \lambda \leq  b \{ (1 + p/q) (\log p) / n\}^{1/2}$.
\end{proof}

To prove Theorem \ref{thm:Thm2}, we assume Conditions 1-3 hold. In the proof of Theorem \ref{thm:Thm1}, we have shown that for the first M step, we obtain the estimate
$ \hat{\bO}_{k}^{(1)}$ such that $\sum_{k=0}^{K} \left\| \hat{\bO}_{k}^{(1)} - \bO^{*}_k \right\|_F = O_p \left[ \left\{ (p+q) ( \log p)/n  \right\}^{1/2} \right].$ In the E-step, if we can show $\| \dot{\Sigma}_{k} - \Sigma^{*}_{k}\|_F = O_p \left[ \left\{ (p+q) (\log p)/n \right\}^{1/2} \right]$, then by Lemma \ref{Lemma.A3}, the next M-step estimate $\hat{\bO}_{k}^{(2)}$ would also satisfy  $\sum_{k=0}^K  \left\| \hat{\bO}_{k}^{(2)} - \bO^{*}_{k} \right\|_F = O_p\left[ \left\{ (p+q) (\log p)/n \right\}^{1/2} \right]$. Therefore, the estimate from the EM algorithm after finite iterations would have the same bound as the one-step algorithm.

From Condition 3, we assume there exists $\tilde{\Sigma}_Y$ such that 
\[
	\|\tilde{\Sigma}_Y - \Sigma^{*}_Y \|_F 
	= O_p \left[ \left\{ (p+q) (\log p)/n \right\}^{1/2} \right].
\]
From the E-step expression \eqref{eq:all1}, we know that
\begin{align}
	\dot{\Sigma}_{0} &= (\hat{A}^{(1)})^{-1} + (\hat{A}^{(1)})^{-1}
	\sum_{l,k=1}^K
	\Big( \hat{\bO}_{l}^{(1)} \tilde{\Sigma}_{Y(l,k)} \hat{\bO}_{k}^{(1)}\Big)
	(\hat{A}^{(1)})^{-1}. \notag
\end{align}
Define  $\Delta_A =\hat{A}^{(1)} - A^*$, where $A^* = \sum_{k = 0}^K \bO_k^*$. From  Theorem \ref{thm:Thm1}, we know that
\begin{align}
	\|\Delta_A \|_F = 
	O_p \left[ \left\{ (p+q) (\log p)/n \right\}^{1/2} \right]. \notag
\end{align}
Using the Woodbury matrix identity twice, we have
\begin{align*}
	(\hat{A}^{(1)})^{-1} = & \; (A^* + \Delta_A )^{-1} \\
	= & \;  A^{*-1} - A^{*-1}( \Delta_A^{-1} + A^{*-1} )^{-1} A^{*-1} \\
	= & \; A^{*-1} 
	- A^{*-1}\{ \Delta_A - \Delta_A(A^* + \Delta_A)^{-1} \Delta_A\} A^{*-1}\\
	= & \; A^{*-1} - A^{*-1} \Delta_A A^{*-1} 
	+ A^{*-1} \Delta_A (\Delta_A + A^*)^{-1} \Delta_A A^{*-1}.
\end{align*}
By Condition 1, we have $ \tau_1 < \phi_{min}(A^*) <\phi_{max}(A^*) < (K + 1) \tau_2 $. Using Lemmas \ref{Lemma.A4} -- \ref{Lemma.A5}, we have
\begin{align}
	\| (\hat{ A}^{(1)})^{-1} - A^{*-1} \|_F 
	\leq & \; \| A^{*-1} \Delta_A A^{*-1} \|_F 
	+ \| A^{*-1} \Delta_A (\Delta_A + A^*)^{-1} \Delta_A A^{*-1}\|_F \notag \\
	\leq & \; \| A^{*-1}  \|^2  \| \Delta_A  \|_F 
	+ \|A^{*-1} \|^2 \| \Delta_A \|_F^2 \| (\Delta_A + A^*)^{-1} \| \label{eq:thm2I} \\
	\leq & \; \| \Delta_A  \|_F / (\tau_1^2)  
	+ \| \Delta_A  \|_F^2
	\| (I_p +  A^{*-1} \Delta_A)^{-1}  A^{*-1} \| / (\tau_1^2) \label{eq:28} \\
	\leq & \; \| \Delta_A \|_F / (\tau_1^2) 
	+ \| \Delta_A  \|_F^2
	\| ( I_p +  A^{*-1} \Delta_A)^{-1}\| / (\tau_1^3) \notag \\
	\leq & \; \frac{\| \Delta_A \|_F}{\tau_1^2}
	+ \frac{ \| \Delta_A  \|_F^2} 
	{(1 - \|A^{*-1} \Delta_A \|)\tau_1^3} \label{eq:29} \\
	\lesssim  & \; \| \Delta_A \|_F/ (\tau_1)^2 
	+ \frac{2}{\tau_1^3} \| \Delta_A \|_F^2  \label{eq:30} \\
	= & O_p\left[ \left\{ \frac{(p+q) \log p}{n} \right\}^{1/2 } \right]. \notag
\end{align}
The inequality of \eqref{eq:thm2I} is due to Lemma \ref{Lemma.A5}; the inequality of \eqref{eq:28} is due to the fact that $\| A^{*-1} \|^2 = 1/(\phi_{min}(A))^2 \leq  1/(\tau_1^2) $; the inequality of \eqref{eq:29} is due to Lemma \ref{Lemma.A4}; and the inequality of \eqref{eq:30} can be achieved when $n$ is large enough since
\begin{align}
	\|A^{*-1} \Delta_A  \| 
	\leq \|A^{*-1} \Delta_A \|_F 
	\leq \|A^{*-1} \| \; \| \Delta_A \|_F
	\leq  \frac{1}{\tau_1^2} \| \Delta_A \|_F  \rightarrow 0 \,. \notag
\end{align}
Next, we define 
\begin{align*}
	\Delta_{k,2} = & \hat{\bO}_{k}^{(1)} - \bO^{*}_{k}, \quad
	\Delta_{(l,k),3}  
	= \tilde{\Sigma}_{Y(l,k)} - \Sigma^{*}_{Y(l,k)} \quad (1 \leq l,k \leq K). 
\end{align*}
It is easy to show that $\| \Delta_A \|_F, \| \Delta_{k,2} \|_F$ and $\| \Delta_{(l,k),3} \|_F$ all have the same rate $ O_p \left[ \left\{ (p+q) (\log p) / n \right\}^{1/2} \right]$, and then we have 
\begin{align}
	 \dot{\Sigma}^{(1)}_{0} - \Sigma_0^* = & \; 
	 (\hat{A}^{(1)})^{-1} +  (\hat{A}^{(1)})^{-1} 
	 \sum_{l,k=1}^K
	 \Big( \hat{\bO}_{l}^{(1)} \tilde{\Sigma}_{Y(l,k)} \hat{\bO}^{(1)}_k\Big)
	 (\hat{A}^{(1)})^{-1}  \notag \\
	 & \; - A^{*-1} - A^{* -1} \sum_{l,k=1}^K 
	 \Big( \bO^{*}_{l} \Sigma^{*}_{Y(l,k)} \bO^{*}_{k}\Big) 
	 A^{*-1} \notag \\
	= & \; \Delta_A + \Delta_A \sum_{l,k=1}^K \Big( \bO^{*}_{l}
	\Sigma^{*}_{Y(l,k)} \bO^{*}_{k} \Big) A^{*-1}
	+ A^{*-1} \sum_{l,k=1}^K \Big( \Delta_{l,2} 
	\Sigma^{*}_{Y(l,k)} \bO^{*}_{k} \Big) A^{*-1} \notag \\
	& \; + A^{*-1} \sum_{l,k=1}^K \Big( \bO^*_l 
 	\Delta_{(l,k),3} \bO^{*}_{k} \Big)A^{*-1}
 	+  A^{*-1} \sum_{l,k=1}^K \Big( \bO^*_l 
 	\Sigma^{*}_{Y(l,k)} \Delta_{k,2}\Big)A^{*-1} \notag \\
 	& \; + A^{*-1}  \sum_{l,k=1}^K \Big( \bO^{*}_{l} 
 	\Sigma^{*}_{Y(l,k)} \bO^{*}_{k}\Big) \Delta_A + B, \notag
\end{align}
where $B$ is the remainder terms with the following values
\begin{align*}
B = & \; \Delta_A \sum_{l,k=1}^K \Big( \Delta_{l,2}
	\Sigma^{*}_{Y(l,k)} \bO^{*}_{k} \Big) A^{*-1}
	+ \Delta_A \sum_{l,k=1}^K \Big(\bO_{l}^*
	\Delta_{(l,k),3} \bO^{*}_{k} \Big) A^{*-1} \\
	& \; + \Delta_A \sum_{l,k=1}^K\Big(  \bO_{l}^*
	\Sigma^{*}_{Y(l,k)} \Delta_{k,2} \Big) A^{*-1} 
	 + \Delta_A \sum_{l,k=1}^K \Big( \bO_l^*
	\Sigma^{*}_{Y(l,k)} \bO^{*}_{k} \Big) \Delta_A \\
	& \; + A^{*-1} \sum_{l,k=1}^K \Big( \Delta_{l,2} 
	\Delta_{(l,k),3} \bO^{*}_{k} \Big) A^{*-1} 
	+ A^{*-1} \sum_{l,k=1}^K \Big( \Delta_{l,2} 
	\Sigma^{*}_{Y(l,k)} \Delta_{k,2} \Big) A^{*-1} \\ 
	& \; + A^{*-1} \sum_{l,k=1}^K \Big( \Delta_{l,2} 
	\Sigma^{*}_{Y(l,k)} \bO^{*}_{k} \Big) \Delta_A
	+ A^{*-1} \sum_{l,k=1}^K \Big( \bO^*_{l} 
 	\Delta_{(l,k),3} \Delta_{k,2} \Big)A^{*-1}\\
 	& \; + A^{*-1} \sum_{l,k=1}^K \Big( \bO^*_{l} 
 	\Delta_{(l,k),3} \bO^{*}_{k} \Big) \Delta_A 
 	 + A^{*-1} \sum_{l,k=1}^K \Big( \bO^{*}_{l} 
 	\Sigma^{*}_{Y(l,k)} \Delta_{k,2}\Big)\Delta_A \\
 	& \; + \Delta_A \sum_{l,k=1}^K \Big( \Delta_{l,2} 
 	\Delta_{(l,k),3} \bO^*_{k}\Big)A^{*-1} 
 	+ \Delta_A \sum_{l,k=1}^K \Big( \Delta_{l,2} 
 	\Sigma^{*}_{Y(l,k)} \Delta_{k,2} \Big)A^{*-1} \\
 	& \; + \Delta_A \sum_{l,k=1}^K \Big( \Delta_{l,2} 
 	\Sigma^{*}_{Y(l,k)} \bO^*_k \Big) \Delta_A
 	+ \Delta_A \sum_{l,k=1}^K \Big( \bO^*_l 
 	\Delta_{(l,k),3} \Delta_{k,2}\Big)A^{*-1} \\
 	& \; + \Delta_A \sum_{l,k=1}^K \Big( \bO^*_l 
 	\Delta_{(l,k),3} \bO^*_k\Big)\Delta_A 
 	 + \Delta_A \sum_{l,k=1}^K \Big( \bO^*_l 
 	\Sigma^{*}_{Y(l,k)} \Delta_{k,2}\Big)\Delta_A \\
 	& \; + \cdots + \Delta_A  \sum_{l,k=1}^K \Big( \Delta_{l,2} 
 	\Delta_{(l,k),3} \Delta_{k,2}\Big) \Delta_A.  
\end{align*}
Each term of $B$ is a product of at least two $\Delta$ terms, where $\Delta$ are $\Delta_A$, $\Delta_{k,2}$ or $\Delta_{(l,k),3}$. Also, we know that $\| \Delta_A \|_F, \| \Delta_{k,2} \|_F$ and $\| \Delta_{(l,k),3} \|_F$ all have the same rate $ O_p \left[ \left\{ (p+q) (\log p) / n \right\}^{1/2} \right]$, and $\| \bO^*_l \|$, $ \| \Sigma^{*}_{Y(l,k)}  \|$ and $\| A \|$ are bounded. Thus, we have $\| B \|_F = O_p(\| \Delta_A \|_F^2 ) = o_p(\| \Delta_A \|_F )$. We then can bound $ \| \dot{\Sigma}^{(1)}_{0} - \Sigma_0^*  \|_F $ as 
\begin{align}
 	\| \dot{\Sigma}^{(1)}_{0} - \Sigma_0^* \|_F \leq  
 	& \;  \| \Delta_A \|_F + \Big\| \Delta_A  
	\sum_{l,k=1}^K \Big(\bO^{*}_{l} \Sigma^{*}_{Y(l,k)} \bO^{*}_{k} \Big)
	 A^{*-1} \Big\|_F
	+ \, \Big\| A^{*-1} 
	\sum_{l,k=1}^K \Big( \Delta_{l,2} \Sigma^{*}_{Y(l,k)} \bO^{*}_{k} 
	\Big) A^{*-1} \Big\|_F \notag \\
 	& \; + \Big\| A^{*-1} 
 	 \sum_{l,k=1}^K \Big( \bO^{*}_{l} \Delta_{(l,k),3} \bO^*_{k} \Big)
 	 A^{*-1} \Big\|_F   
 	+ \Big\| A^{*-1} 
 	\sum_{l,k=1}^K \Big( \bO^{*}_{l} \Sigma^{*}_{Y(l,k)} \Delta_{k,2} 
 	\Big)  A^{*-1} \Big\|_F  \notag \\
	& \; + \Big\| A^{*-1} \sum_{l,k=1}^K \Big( \bO^{*}_{l} \Sigma^{*}_{Y(l,k)} 
	\bO^{*}_{k}	\Big)\Delta_A \Big\|_F 
	 + o_p( \| \Delta_A \|_F ) \notag \\
	\leq & \;  \| \Delta_A \|_F + 
	 K^2  \frac{\tau_2^2}{\tau_1^2} \| \Delta_A \|_F  
	+ K \frac{\tau_2}{\tau_1^3 } \sum_{l=1}^K \| \Delta_{l,2} \|_F 
	+ \frac{\tau_2^2}{\tau_1^2} \sum_{l,k=1}^K \|\Delta_{(l,k),3} \|_F  \notag \\
	&  \; + K \frac{\tau_2}{\tau_1^3}  \sum_{k=1}^K \| \Delta_{k,2} \|_F \notag 
	+K^2 \frac{\tau_2^2}{\tau_1^2} \| \Delta_A\|_F  + o_p( \| \Delta_A \|_F ) \\
	 = &; O_p \left[ \left\{  \frac{(p+q)\log p}{n} \right\}^{1/2}
	 \right] \notag.
 \end{align}
 Similarly for $\dot{\Sigma}^{(1)}_{k}$, we can prove that $\| \dot{\Sigma}^{(1)}_{k} - \Sigma^{*}_{k}  \|_F = O_p \left[ \left\{ (p+q)  (\log p)/n \right\}^{1/2} \right]$. Then by Lemma \ref{Lemma.A3}, the corresponding M-step estimate $\hat{\bO}_{k}^{(2)}$ would also have the rate  $\sum_{k=0}^K  \left\| \hat{\bO}_{k}^{(2)} - \bO^{*}_{k} \right\|_F = O_p\left[ \left\{ (p+q) ( \log p)/n \right\}^{1/2} \right]$. Following the previous step, we can show $$\| \dot{\Sigma}^{(2)}_{k} - \Sigma^{*}_{k}  \|_F = O_p \left[ \left\{ (p+q) ( \log p)/n \right\}^{1/2} \right] (k = 0, \ldots, K)$$ and so on. Therefore, the solution of our EM algorithm after finite iterations would have the same bound as the one-step method. This completes the proof. 
\section{Proof of Theorem \ref{thm:Thm3}}
 This proof follows a similar argument as shown in \citet[Theorem 2]{Lam2009}. Let $\text{sign}(a)$ denote the sign of $a$. The derivative for $\mathcal{W}_k( \bO_k) $ with respect to $\omega_{k(i,j)}$ is
\begin{align}
	\frac{\partial \mathcal{W}_k(\bO_k)}{\partial \omega_{k(i,j)}}
	= 2\{ \hat{\sigma}_{k(i,j)}^\prime - \sigma_{k(i,j)}
    + \lambda \, \text{sign}(\omega_{k(i,j)}) \} \,, \notag
\end{align}
where $\lambda = \lambda_2$ for $k = 0$, and otherwise $\lambda = \lambda_1$. 
For $(i,j) \in T_k^c$, it is sufficient to show that the sign of $\partial \mathcal{W}_{k}(\bO_k)/\partial \omega_{k(i,j)}$ at the minimum solution $\hat{\omega}_{k(i,j)}^{\mone}$ only depends on the sign of $\hat{\omega}_{k(i,j)}^{\mone}$ with probability tending to $1$. Namely, the rate for $\lambda$ dominates the rate of $\hat{\sigma}_{k(i,j)}^\prime - \sigma_{k(i,j)}$. To see that, without loss of generality, we suppose $ \hat{\omega}_{k(i,j)}^{\mone} < 0 $ and $(i, j) \in T_k^c$. Then there is a small $\v > 0$ such that $\hat{\omega}_{k(i,j)}^{\mone}  + \v < 0$. Since $\hat{\omega}_{k(i,j)}^{\mone}$ is the minimum solution, $\partial \mathcal{W}_{k}(\bO_k)/\partial \omega_{k(i,j)}$ is positive at $\hat{\omega}_{k(i,j)}^{\mone} + \v$ for the small $\v > 0$. Because $\partial \mathcal{W}_{k}(\bO_k)/\partial \omega_{k(i,j)}$ at $\hat{\omega}_{k(i,j)}^{\mone}$ has the same sign as $\hat{\omega}_{k(i,j)}^{\mone}$ and is a continuous function, $\partial \mathcal{W}_{k}(\bO_k)/\partial \omega_{k(i,j)}$ should be negative at $\hat{\omega}_{k(i,j)}^{\mone} + \v$ for a small $\v$, which contradicts the previous conclusion. Therefore, the optimum $\hat{\omega}_{k(i,j)}^{\mone}$ is $0$ in this case. 
 
Let $\hat{\bO}^{\mone}_k = \bO_k^* + \Delta_k$ and $\check{\Sigma}^{\mone}_k = (\hat{\bO}_k^{\mone})^{-1}$. Since  $\sum_{k = 0}^{K} \left\| \hat{\bO}_{k}^{\mone} -  \bO^{*}_{k} \right\| = O_p(\eta_n)$, we have $\| \Delta_k \| = O_p(\eta_n)$. Using the Woodbury formula, we have that
\begin{align}
	\check{\Sigma}^{\mone}_k = & \; (\bO^*_k + \Delta_k )^{-1}
	= \Sigma^*_k 
	- \Sigma^*_k( \Delta_k^{-1} + \Sigma^*_k )^{-1} \Sigma^*_k \notag \\
	= & \; \Sigma^*_k 
	- \Sigma^*_k( \Delta_k - \Delta_k(\bO^*_k + \Delta_k )^{-1} \Delta_k)
	\Sigma^*_k \notag \\
	= & \; \Sigma^*_k - \Sigma^*_k \Delta_k \Sigma^*_k 
	+ \Sigma^*_k \Delta_k (\Delta_k + \bO_k^*)^{-1} \Delta_A \Sigma^*_k . \notag
\end{align}

By Condition 1, we have
\begin{align}
	\|\check{\Sigma}_k^{\mone} - \Sigma^*_k \| 
	\leq & \; \| \Sigma^*_k \Delta_k \Sigma^*_k \| 
	+ \|\Sigma^*_k \Delta_k (\Delta_k 
	+ \bO_k^*)^{-1} \Delta_k \Sigma^*_k \| \notag \\
	\leq & \; \| \Sigma_k^* \|^2 \| \Delta_k \| 
	+ \|\Sigma_k^* \|^2 \| \Delta_k \|^2 \| (\Delta_k + \bO_k^*)^{-1} \| 
	\notag \\
	\leq & \; \| \Delta_k \| / (\tau_1^2) 
	+ \| \Delta_k \|^2 \| (I_p + \Sigma_k^* \Delta_k )^{-1} \| 
	\, \|\Sigma_k^* \| / (\tau_1^2)  \notag \\
	\leq & \; \|\Delta_k \| / (\tau_1^2) 
	+ \frac{1}{\tau_1^3} \| \Delta_k \|^2 
	(1 - \| \Sigma_k^* \Delta_k\|)^{-1} \label{eq:thm3f1} \\
	\lesssim & \; \|\Delta_k \|  / (\tau_1^2) 
	+ \frac{2}{\tau_1^3} \| \Delta_k \|^2 
	 \label{eq:thm3f2} \\
	= & \; O_p(\eta_n)\notag,
\end{align}
where the inequality of \eqref{eq:thm3f1} is due to Lemma \ref{Lemma.A4} and the inequality of \eqref{eq:thm3f2} holds when $n$ large enough because
\begin{align}
	\|\Sigma_k^* \Delta_k \| \leq \| \Sigma_k^* \Delta_k \|_F 
	\leq \|\Sigma_k^*  \| \; \| \Delta_k \|
	\leq  \frac{1}{\tau_1} \| \Delta_k \|  \rightarrow 0 \,. \notag
\end{align}

From the proof of Theorem \ref{thm:Thm1}, we know $\| \hat{\Sigma}_k^\prime - \Sigma^*_k \|_\infty = O_p[ \{ (\log p) / n\}^{1/2}]$, and the derivative  $\partial \mathcal{W}_{k}(\bO_k)/\partial \omega_{k(i,j)}$ at the minimum $\hat{\omega}_{k(i,j)}$ is $2\{ \hat{\sigma}_{k(i,j)}^\prime - \check{\sigma}_{k(i,j)}
    + \lambda \, \text{sign}(\hat{\omega}_{k(i,j)}) \}$. Combining the results, we have
\begin{align*}
	\max_{i,j}|  \hat{\sigma}_{k(i,j)}^\prime - \check{\sigma}^{\mone}_{k(i,j)} | 
	= & \; 	\max_{i,j}|  \hat{\sigma}_{k(i,j)}^\prime -\sigma^{*}_{k(i,j)}  
	+ \sigma^{*}_{k(i,j)} - \check{\sigma}_{k(i,j)}^{\mone} |
	\\
	\leq & \; \max_{i,j}|\hat{\sigma}_{k(i,j)}^\prime - \sigma^{*}_{k(i,j)} | 
	+ \max_{i,j}|\sigma^{*}_{k(i,j)} - \check{\sigma}_{k(i,j)}^{\mone} | \\
	\leq & \; \|\hat{\Sigma}_k^\prime - \Sigma^{*}_k \|_\infty  
	+ \| \check{\Sigma}_k^{\mone} - \Sigma^{*}_k  \| \\
	= & \; O_p[ \{ (\log p) / n\}^{1/2}  + \eta_n].
\end{align*}
Therefore, if $\lambda \succeq \{ (\log p) / n  \}^{1/2}  + \eta_n $, then the term $\lambda \,\text{sign}(\omega_{k(i,j)})$ dominates over $\hat{\sigma}_{k(i,j)}^\prime - \sigma_{k(i,j)}$ with probability tending to 1. This completes the proof.


\section{Proof of Theorem \ref{thm:Thm4}}

Assume the last iteration of the EM algorithm minimizes
\begin{align*}
	 \mathcal{W}_{k}^{\prime}(\bO_k) =  \tr(\dot{\Sigma}_{k} \Omega_k) 
	 - \log \det(\bO_k) + \lambda \sum_{i \neq j} \left|\omega_{k(i, j)} \right|,
\end{align*}
where $\lambda = \lambda_2$ when $k = 0$, and otherwise $\lambda = \lambda_1$. The derivative for $\mathcal{W}_k^\prime$ with respect to $\omega_{k(i,j)}$ is
\begin{align*}
	\frac{\partial \mathcal{W}_{k}^{\prime}(\bO_k)}{\partial \omega_{k(i,j)}} 
	= 2\{ \dot{\sigma}_{k(i,j)} - \sigma_{k(i,j)} 
	+ \lambda \text{sign}(\omega_{k(i,j)}) \}.
\end{align*} 
Similar to the proof of Theorem \ref{thm:Thm3}, it is enough to show that for $(i, j) \in T_k^c$, the sign of $\partial \mathcal{W}_{k}(\bO_k)/\partial \omega_{k(i,j)}$ at the minimum $\hat{\omega}_{k(i,j)}^{\mEM}$ only depends on the sign of $\hat{\omega}_{k(i,j)}^{\mEM}$ with probability tending to $1$. Let $\hat{\bO}_k^{\mEM}$ be the minimum in Theorem  \ref{thm:Thm2}, and define $\check{\Sigma}^{\mEM}_k = (\hat{\bO}_k^{\mEM})^{-1}$. 

 From the proof of Theorems \ref{thm:Thm2} and \ref{thm:Thm3}, we have shown $\| \dot{\Sigma}_k - \Sigma^*_k \|_F = O_p[ \{(p + q) (\log p) / n\}^{1/2}]$ and $\| \check{\Sigma}^{\mEM}_k - \Sigma^*_k \| = O_p( \zeta_n)$. Combining the results yields
\begin{align*}
	\max_{i,j}| \dot{\sigma}_{k(i,j)} - \check{\sigma}^{\mEM}_{k(i,j)} |
	\leq & \; \max_{i,j}|\dot{\sigma}_{k(i,j)} - \sigma^{*}_{k(i,j)}| 
	+ \max_{i,j}|\sigma^{*}_{k(i,j)} - \check{\sigma}_{k(i,j)}^{\mEM}| \\
	\leq & \; \| \dot{\Sigma}_k - \Sigma^*_k \|_\infty 
	+ \| \check{\Sigma}_k^{\mEM} - \Sigma^*_k \|\\
	\leq & \;\| \dot{\Sigma}_k - \Sigma^{*}_k \|_F 
	+ \|\check{\Sigma}_k^{\mEM} - \Sigma^{*}_k \|\\
	= & \;  O_p[ \{(p + q) (\log p) / n\}^{1/2} + \zeta_n] \notag.
\end{align*}
Therefore, if $\lambda \succeq  \{(p + q)  (\log p) / n\}^{1/2} + \eta_n$, the sign of $\partial \mathcal{W}_{k}^\prime (\bO_k)/ \partial \omega_{k(i,j)}$  at the optimum point only depends on $\text{sign}(\omega_{k(i, j)})$. This completes the proof.


\section{Extension to include the Systemic Intensity parameter $\alpha_k$}
\label{appendix:extend}
We here consider an elaboration of our model to allow the influence of the systemic layer to vary among tissues as suggested by one of the reviewers. This variation could be motivated in several ways. For example, muscle and adipose both develop from mesoderm. Thus, we might expect them to be more closely related to each other (and be similarly affected by systemic factors) compared with the pancreas, which develops from endoderm. The extended model is described as follows:
\begin{align*} 
	Y_{k, i} = X_{k, i} + \alpha_k Z_i  \quad (k = 1, \ldots, K; \ i = 1, \ldots, n),
\end{align*}
where $\alpha_k$ quantifies the level of systemic influence in each tissue $k$. For the  identifiability issue, we assume $\text{max}(\text{diag}(\Sigma_0)) = 1$.

Similar to Section \ref{appendix:likelihood}, we derive the probability density function of $Y$ as follows:
\begin{align*}
	f_Y(s) = & \; \int_{-\infty}^{\infty}{ f_Y(s \mid Z = t)f_Z(t)
	 \mathrm{d} t} \\
	\propto & \; \int_{-\infty}^{\infty} \exp
	\left[ \sum_{k=1}^K \Big\{ (s_s - \alpha_k t)^\T \bO_k (s_k - \alpha_k t) \Big\} 
	+ t^\T \bO_0 t \right] \mathrm{d} t\\
	= & \; \exp \Big\{ \sum_{k=1}^K \big( s_{k}^\T \bO_k s_{k} \big) \Big\}
	\int_{-\infty}^{\infty}{\exp \Big\{ t^\T \big(\sum_{k = 1}^{K}
	 \alpha_k^2 \bO_k + \bO_0 \big)t 
	 -2 \big(\sum_{k=1}^K \alpha_k s_{k}^\T \bO_{k} \big)
	 t \Big\} \mathrm{d} t} \\
	= & \; \exp \Big\{\sum_{k=1}^K \big( s_k^\T \bO_{k} s_k \big) \Big\}
		\int_{-\infty}^{\infty}{\exp\left(t^\T A_{\text{ext}} t 
		- 2 c_{\text{ext}}^\T t \right) \mathrm{d} t} \\
	= & \; \exp \Big\{ \sum_{k=1}^K \big( s_k^\T \bO_{k} s_k \big) \Big\}           
		\exp \big(-c^\T_{\text{ext}} 
		A_{\text{ext}}^{-1} c_{\text{ext}} \big) 
		\int{ \exp \left\{( A_{\text{ext}} t- c_{\text{ext}})^\T 
		A_{\text{ext}}^{-1}(A_{\text{ext}}t 
		- c_{\text{ext}}) \right\} \mathrm{d}t} \\
		\propto & \; \exp \bigg[\sum_{k=1}^K \big( s_k^\T \bO_k s_k \big)
		- \Big\{  \sum_{k = 1}^K 
		\big( \alpha_k s_k^\T \bO_k\big)\Big\} A_{\text{ext}}^{-1}
		\Big\{ \sum_{k=1}^K \big( \alpha_k \bO_k s_k\big)\Big\} \bigg] \\[5pt]
		= & \; \exp \Big\{ s^\T \left(\{_d \bO_{k} \}_{1 \leq k \leq K} 
		- \{ \alpha_l \alpha_k \bO_k A^{-1}_{\text{ext}} \bO_k 
		\}_{1 \leq l, k \leq K} 
		\right) s\Big\} \\[8pt]
		= & \; \exp ( s^{\T} \bO_Y s) \,,
\end{align*}
where $ A_{\text{ext}} = \sum_{k = 1}^{K}\alpha_k^2 \bO_{k} + \bO_0$ and $c_{\text{ext}} = \sum_{k=1}^K \alpha_k \bO_{k} Y_{k} $.  Therefore, we have
\[
	\bO_Y = \{_d \bO_{k} \}_{1 \leq  k \leq K} - 
	\{ \alpha_l \alpha_k \bO_k A_{\text{ext}}^{-1} \bO_k \}_{1 \leq l, k \leq K}.
\]

Next, we want to derive $\det(\Omega_Y)$. We know that
\begin{align*}
	\Sigma_Y = & \; \{_d \, \Sigma_{k} \}_{1 \leq k \leq K} + 
	\begin{pmatrix}
		\alpha_1 I_p  \\
  		\vdots  \\
  		\alpha_K  I_p
 	\end{pmatrix}
 	\begin{pmatrix}
  		\alpha_1 \Sigma_0, & \ldots &, \alpha_K \Sigma_0  
	\end{pmatrix}  \\ 
	= & \; \{_d \  \Sigma_{k} \}_{1 \leq k \leq K} 
	\left\{ I_{Kp} + \{_d  \  \bO_{k} \}_{1 \leq k \leq K}    
	\begin{pmatrix}
   		\alpha_1 I_p  \\
  		\vdots  \\
  		\alpha_K I_p
 	\end{pmatrix}
 	\begin{pmatrix}
  		\alpha_1 \Sigma_{0}, & \ldots & ,\alpha_K \Sigma_0
	\end{pmatrix} \right\}.
\end{align*}
Therefore, the $\det(\Sigma_Y)$ can be expressed as
\begin{align*}
	\det(\Sigma_Y) 
	& = 	\det \left\{ I_p +
	\begin{pmatrix}
  		\alpha_1 \Sigma_{0} & \ldots & \alpha_K \Sigma_{0}  
  	\end{pmatrix}
  	\{_d \bO_{k} \}_{1 \leq k \leq K} 
  	\begin{pmatrix}
 		\alpha_1 I_p  \\
 		\vdots  \\
		\alpha_K I_p
	\end{pmatrix}
	\right\}  \prod_{k = 1}^K \det( \Sigma_{k}) \\
	& = 	\det \Big(I_p + \Sigma_{0}\sum_{k = 1}^K \alpha_k^2  \bO_{k} \Big)
	 \prod_{k = 1}^K \det(\Sigma_{k})  \\
	& =  \det \Big( \Sigma_0 \bO_0 + \Sigma_0 \sum_{k = 1}^K 
	\alpha_k^2 \bO_k \Big) \prod_{k = 1}^K \det(\Sigma_{k})  \\
   & = \det(A_{\text{ext}}) \prod_{k = 0}^{K }\det(\Sigma_{k}).
\end{align*}
Therefore, we have 
\[ 
	\log  \det(\bO_Y)  = \sum_{k = 0}^{K} \{
	\log \det(\bO_{k})\} - \log \det (A_{\text{ext}}).
\]
 Combining the previous results, we could write the log-likelihood given $Y=y$ as
\begin{align*}
	\mathcal{L}(\bO_Y; y )
	= & -\frac{npK}{2} \log (2\pi) 
	+ \frac{n}{2} \big\{ \log  \det(\bO_Y) 
	- \tr \big(\hat{\Sigma}_Y \bO_Y \big) \big\} \\
	= & - \frac{npK}{2} \log (2\pi) 
	+ \frac{n}{2}  \Big[ \sum_{k = 0}^{K} \{ \log \det(\bO_{k}) \}
	- \log \det (A_{\text{ext}}) \Big] \\
	& - \frac{n}{2} \tr \Big( \hat{\Sigma}_Y
	\Big[ \{_d \bO_{k} \}_{1 \leq k \leq K} 
	- \{ \alpha_l \alpha_k  \bO_{l} A_{\text{ext}}^{-1} \bO_{k} \}_{1 \leq l, k \leq K} \Big] 
	\Big). \\ 
\end{align*}
Under this setting, we have 
\begin{align*}
	\begin{pmatrix}
 		Z  \\
 		Y_1 \\
 		\vdots  \\
		Y_K 
	\end{pmatrix}  & \sim \mathcal{N}  \left( 
		\begin{pmatrix}
 		0  \\
 		0 \\
 		\vdots  \\
		0 
	\end{pmatrix}, 	
  	\begin{pmatrix}
 		\Sigma_0            & \alpha_1 \Sigma_0 		   		&
 		\ldots 				& \alpha_K \Sigma_0 \\
		\alpha_1 \Sigma_0   & \; \Sigma_1 + \alpha_1^2 \Sigma_0 & 
		\ldots				&  \alpha_1\alpha_k \Sigma_0 \\
 		\vdots 				& \vdots 		    				& 
 		\vdots				& \vdots \\
		\alpha_K \Sigma_0   & \; \alpha_1\alpha_k \Sigma_0 		&
		\ldots				& \Sigma_K + \alpha_k^2 \Sigma_0 
	\end{pmatrix}
	\right).
\end{align*}
For simplicity, we denote $\{ \bO \}_{k = 0}^K$ as $\bO$ and  $\{ \alpha_k \}_{k = 1}^K$ as $\alpha$. We can derive $E (Z \mid Y, \bO , \alpha)$, $ \text{var}(Z \mid Y, \bO, \alpha)$  and $E(Z Z^\T \mid Y, \bO, \alpha)$ as follows:
\begin{align*}
	 E (Z \mid  Y, \bO, \alpha ) = & \; (\alpha_1 \Sigma_0, \ldots, \alpha_K \Sigma_0) \bO_Y
	\begin{pmatrix}
		Y_1 \\
		\vdots\\
		Y_K 
	\end{pmatrix} \\[8pt]
	 = & \; (\alpha_1 \Sigma_0, \ldots, \alpha_K \Sigma_0) 
	\Big\{ \{_d \bO_k  \}_{1 \leq k \leq K} \\
	& \; - \begin{pmatrix}
		\alpha_1 \bO_1 \\
		\vdots\\
		\alpha_K \bO_k 
	\end{pmatrix}  
	(\alpha_1 A_{\text{ext}}^{-1}  \bO_1,
	\ldots,  \alpha_k A_{\text{ext}}^{-1}  \bO_k  )  \Big\} 
	\begin{pmatrix}
		Y_1 \\
		\vdots\\
		Y_K  
	\end{pmatrix} \\
	= & \; \Big( (\alpha_1 \Sigma_0 \bO_1, \ldots, \alpha_K \Sigma_0 \bO_K) \\
	 & \; - \Sigma_0 \Big\{ \sum_{k = 1}^K \big( \alpha_k^2 \bO_k \big) \Big\}
	(\alpha_1 A_{\text{ext}}^{-1}  \bO_1,
	\ldots, \alpha_k A_{\text{ext}}^{-1}  \bO_k  )  \Big)  
	\begin{pmatrix}
		Y_1 \\
		\vdots\\
		Y_K  
	\end{pmatrix} \\[8pt]
	= & \; \Big\{ (\alpha_1 \Sigma_0 \bO_1, \ldots, \alpha_K \Sigma_0 \bO_K) 
	- \Sigma_0 \big( A_{\text{ext}} - \bO_0 \big)
	(\alpha_1 A_{\text{ext}}^{-1}  \bO_1,
	\ldots, \alpha_k A_{\text{ext}}^{-1}  \bO_k  )  \Big\}
	\begin{pmatrix}
		Y_1 \\
		\vdots\\
		Y_K  
	\end{pmatrix} \\
	= & \; (\alpha_1 A_{\text{ext}}^{-1}  \bO_1,
	\ldots,  \alpha_k A_{\text{ext}}^{-1}  \bO_k  )
	\begin{pmatrix}
		Y_1 \\
		\vdots\\
		Y_K  
	\end{pmatrix} \\	
	= & \; A_{\text{ext}}^{-1}  c_{\text{ext}}, \\
	\text{var} (Z \mid Y, \bO, \alpha ) = & \; \Sigma_0 
	-  (\alpha_1 \Sigma_0, \ldots, \alpha_K \Sigma_0) \bO_Y
		\begin{pmatrix}
		\alpha_1 \Sigma_0 \\
		\vdots\\
		\alpha_K \Sigma_0 
	\end{pmatrix} \\
	= & \; \Sigma_0 - 
	(\alpha_1 A_{\text{ext}}^{-1}  \bO_1,
	\ldots, \alpha_k A_{\text{ext}}^{-1}  \bO_k  )
	\begin{pmatrix}
		\alpha_1 \Sigma_0\\
		\vdots\\
		\alpha_K \Sigma_0  
	\end{pmatrix} \\
	= & \; \Sigma_0 - A_{\text{ext}}^{-1} 
	\Big\{ \sum_{k = 1}^K 
	\big( \alpha_k^2 \bO_k \big)\Big\} \Sigma_0 \\
	= & \; \Sigma_0 - A_{\text{ext}}^{-1} 
	\Big( A_{\text{ext}} -  \bO_0 \Big) \Sigma_0 \\
	= & \; A_{\text{ext}}^{-1}, \\
	E (Z Z^\T \mid Y, \bO, \alpha)  = &\; \text{var} (Z \mid Y, \bO, \alpha) 
	+ E (Z \mid Y, \bO, \alpha) E (Z \mid Y, \bO, \alpha)^\T \\ 
	= & \; A_{\text{ext}}^{-1} 
	+ A_{\text{ext}}^{-1} c_{\text{ext}}c_{\text{ext}}^\T A_{\text{ext}}^{-1},
\end{align*}
where $c_{\text{ext}} = \sum_{k = 1}^K \alpha_k \bO_k Y_k$.

As in the main paper, let $y_{k, i}$ be the realization of $Y_{k,i}$ and $y$ be the $n$ by $Kp$ dimensional data matrix. We can modify our EM algorithm to calculate $\alpha_k$ and $\Omega_k$ jointly. The modified EM algorithm is described as follows:
\begin{itemize}

	\item[E step] Update the expectation of the log-likelihood conditional on $\bO$.
		\begin{align}
		 	\mathcal{Q}( \bO; \bO^{(t)}, \alpha^{(t)}) 
			\propto & \; \log  \det( \bO_0) - \tr\Big[ \bO_{0} 
            E_{Z | \bO^{(t)}, \alpha^{(t)}} \Big\{ \sum_{i=1}^n \big( \frac{z_i z_i^\T}{n} \big) \Big\} \Big] 
            + \sum_{k=1}^K \log  \det(\bO_{k})  \notag\\
			& \; - \sum_{k=1}^K \tr \Big[ \bO_{k} E_{Z | \bO^{(t)}, \alpha^{(t)}} \Big\{
			\sum_{i=1}^n (y_{k,i} - \alpha_k z_i)(y_{k,i} - \alpha_k z_i)^\T /n
			\Big\} \Big] \notag \\
			= & \; \sum\limits_{k = 0}^{K}
			\Big\{ \log  \det(\bO_{k}) -\tr \big( \bO_{k} \dot{\Sigma}_{k}^{(t)} \big)
			\Big\}, \label{eq:ext2}
		 \end{align}
where $\dot{\Sigma}_{k}^{(t)}$ and $\dot{\Sigma}_0^{(t)}$ are
\begin{subequations}\label{eq:extall1}
  \vspace{5pt}
	\begin{align}
		\dot{\Sigma}_{k}^{(t)} 
		= & \; \hat{\Sigma}_{Y(k,k)} 
		- \alpha_k \sum_{l=1}^K \left( \alpha_l^{(t)} \hat{\Sigma}_{Y(k,l)}\bO_{l}^{(t)} \right)
		(A^{(t)}_{\text{ext}})^{-1} 
		-  \alpha_k (A^{(t)}_{\text{ext}})^{-1}
		\sum_{l=1}^K \left(\alpha_l^{(t)} \bO_{l}^{(t)}\hat{\Sigma}_{Y(l,k)} \right) \notag \\
		& \; + \alpha_k^2 (A_{\text{ext}}^{(t)})^{-1}
		\sum_{l,k=1}^K\left(\alpha_l^{(t)} \alpha_k^{(t)} 
		\bO_{l}^{(t)} \hat{\Sigma}_{Y(l,k)} 
		\bO_{k}^{(t)} \right)
		(A_{\text{ext}}^{(t)})^{-1} \notag \\ 
		& \; + \alpha_k^2(A_{\text{ext}}^{(t)})^{-1} \quad (k =1, \ldots, K), \notag\\
		\dot{\Sigma}_{0}^{(t)} =
		& \; (A^{(t)}_{\text{ext}})^{-1} +(A^{(t)}_{\text{ext}})^{-1} 
		\sum_{l,k=1}^K\left( 
		\alpha_l^{(t)} \alpha_k^{(t)}\bO_{l}^{(t)} \hat{\Sigma}_{Y(l,k)}
		 \bO_{k}^{(t)} \right)
		(A_{\text{ext}}^{(t)})^{-1} \notag.
  \end{align}
\end{subequations} 
 \item[M step]
 Since the function \eqref{eq:ext2} is a biconcave function, we can first fix $\alpha^{(t)}$ and update $\Omega$  by solving 
\begin{align}
	\bO^{(t + 1)} = \argmin_{ \bO} 
	- \mathcal{Q}(\bO; y, \bO^{(t)},\alpha^{(t)}) 
	+ \lambda_1 \sum_{k=1}^K | \bO_k^- |_1 
	+ \lambda_2  |\bO_0^-|_1,   \quad \notag
\end{align}
where $\alpha^{(t)}$ and $\bO^{(t)}$ denote the estimates from the $t$th iteration.
Then fixing $\Omega^{(t + 1)}$, we update $\alpha_k$ as
\begin{align}
    	\hat{\alpha}_k^{(t + 1)} = & \; \frac{\tr \Big[ \hat{\bO}_k^{(t + 1)} 
    	 \sum_{i=1}^n \Big\{  y_{k,i} 
    	 E_{Z | \bO^{(t + 1)}, \alpha^{(t)}}(z_i)^\T 
    	 + E_{Z | \bO^{(t + 1)}, \alpha^{(t)}}(z_i) y_{k,i}^\T \Big\} \Big] }
    	{2 \tr \Big(\hat{\bO}_k^{(t + 1)} \sum_{i=1}^n  E_{Z|\bO^{(t + 1)}, \alpha^{(t)}}(z_i z_i^\T) \Big)} \notag \\
    	= & \; \frac{\tr \Big\{ \hat{\bO}_k^{(t + 1)} 
    	 \sum_{l = 1}^K \Big( \alpha_l^{(t)} \hat{\Sigma}_{Y(k, l)} 
    	 \bO_l^{(t)}A^{(t) -1}_{\text{ext}} 
    	 + \alpha_l^{(t)} A^{(t) -1}_{\text{ext}}
    	 \bO_l^{(t)} \hat{\Sigma}_{Y(l, k)} \Big) \Big\} }
    	{2 \tr \Big\{ \hat{\bO}_k^{(t + 1)}  \big( A^{(t) -1}_{\text{ext}} 
    	+ A^{(t) -1}_{\text{ext}} c^{(t)}_{\text{ext}} c^{(t)\T}_{\text{ext}} 
    	A^{(t) -1}_{\text{ext}} \big) \Big\}}.	\notag
\end{align}
\end{itemize}
Iterate the M and E steps until converge.

Let $\{ \hat{\bO}_k^{\mEM} \}_{k = 0}^K$ and $\hat{\alpha}$ be the solution of extended EM method. We then normalize $\hat{\bO}_0{\mEM}$ and $\hat{\alpha}$ to avoid the identifiability issue as follows: 
\begin{align*}
	\tilde{\bO}_0=  \text{max}\{ (\hat{\bO}_0^{{\mEM}-1})^+ \} \hat{\bO}_0, \quad
	\tilde{\alpha} = \text{max}\{ (\hat{\bO}_0^{{\mEM}-1})^+ \}  \hat{\alpha}.
\end{align*}
Thus, our final solution is $( \tilde{\bO}_0, \{ \hat{\bO}^{\mEM}_k \}_{k = 1}^K, \tilde{\alpha})$. 

\section{Additional Simulation Results}

\subsection{Estimation of category-specific $\bO_k$ and systemic networks $\bO_0$}
\label{appendix:Sim}

In this subsection we report additional simulation results generated from  Chain/Chain denoted as (III) and Nearest Neighbor/Chain denoted as (IV) graph structure with dimension $p = 100$ and $1000$. Similar to the results in Table \ref{tab:simresults1}, the one-step method results in a higher entropy loss, Frobenius loss, false positive rates and hamming distance as shown in Table \ref{tab:simresults2}. For both methods, extended Bayesian information criterion  tends to select a sparser graph with higher false negative rates than cross-validation, especially when $p = 1000$. 

The corresponding  receiver operating characteristic curves are plotted in Fig. \ref{fig:roc1II}--\ref{fig:roc3} based on 100 replications. In Fig. \ref{fig:roc1II}, the corresponding receiver operating characteristic  curves represent the average false positive rates and average true positive rates for both category-specific and systemic networks. Similarly to Fig. \ref{fig:roc}, the EM method dominates the one-step method for both $p = 100$ and $1000$. In Fig. \ref{fig:roc2} and \ref{fig:roc3}, we show the  receiver operating characteristic curves for category-specific and systemic networks separately. The results show that the EM method is superior to the one-step method for estimating the category-specific networks, but is similar to the one-step method for estimating the systemic network. One possible explanation for this result is that $\hat{\Sigma}_k$ using \eqref{eq:Sigmahatk} has greater variation from the underlying truth than $\hat{\Sigma}_0$ since the estimation error in \eqref{eq:Sigmahatk} involves the error from estimating both $\Sigma_{Y(k, k)}$ and $\Sigma_0$. Therefore, the EM method gains more advantages by updating $\hat{\Sigma}_k$ from the E-step.  

\begin{table}[H]\footnotesize
\def~{\hphantom{0}}
\caption{\label{tab:simresults2}Summary statistics reporting the performance of the EM and one-step methods inferring graph structure for (III) and (IV) networks. The numbers before and after the slash correspond to the results using the extended Bayesian information criterion and cross-validation, respectively.}
\begin{center}
  \begin{tabular}{c c c c c c c c c}
  \hline
 $p$ & \specialcell{True networks\\ Category / Systemic}  & $\rho$ & Method & EL & FL & FP$(\%)$ & FN$(\%)$  & HD $(\%)$  \\
  \hline\multirow{12}{*}{$100$}
    &\multirow{6}{*}{(III)}
    &  0   & One-step & 11.8/9.7           & 0.22/0.15           & 5.7/20.4             &  
    0.0/0.0           & 5.6/20.0  \\
    && 0   & EM       & 6.3/4.4  & 0.14/0.07 & 4.4/14.1 & 
   0.0/ 0.0& 4.3/13.8 \\
    && 0.2 & One-step & 10.0/8.1            & 0.21/0.15           & 4.8/18.5             &
    0.8/ 0.2          & 4.7/17.9 \\
    && 0.2 &  EM      & 5.6/4.3   & 0.13/ 0.09& 4.6/13.0   &
    0.0/0.0 & 4.5/12.7  \\
    && 1   & One-step & 12.1/9.5            & 0.24/0.16           & 6.5/22.9             &
    7.3/1.8           & 6.5/22.1  \\
    && 1   &  EM      & 7.5/5.5 & 0.16/0.11 & 6.1/15.0 &
    1.6/0.2 & 5.9/14.4 \\
    \cline{2-9}
    &\multirow{6}{*}{(IV)}
    & 0    & One-step & 11.7/9.4            & 0.26/0.18           & 3.8/18.9       &
    20.3/7.5          & 4.4/18.5 \\
    &&0    & EM       & 7.6/5.6   & 0.18/0.13 & 3.9/13.4       &
    10.5/3.3 & 4.2/ 13.0 \\
    &&0.2  & One-step & 11.8/9.1            & 0.25/0.18           & 4.2/18.9       &
    23.8/8.8          & 5.0/18.5 \\
    &&0.2  & EM       & 7.6/5.7 & 0.18/0.12 & 4.6/14.1   & 
    12.3/3.6& 4.9/ 13.7\\
    && 1   & One-step & 15.6/12.1           & 0.26/0.17           & 8.2/28.1            &
    24.2/9.6          & 9.4/26.7 \\
    && 1   & EM       & 10.8/7.3 & 0.20/0.11 & 7.3/21.8  &
    13.4/2.7 & 7.7/ 20.4 \\
    \hline\multirow{12}{*}{$1000$}
    &\multirow{6}{*}{(III)}
    &  0  &   One-step & 263.2/229.3          & 0.42/0.34            & 0.6/8.0             &
     0.8/0.1           & 0.6/ 8.0    \\
    && 0  &   EM       & 104.7/76.9 & 0.21/0.14 & 0.4/ 2.0   &
   0.0/0.0  & 0.4/2.0  \\
    && 0.2&   One-step & 166.8/142.1          & 0.32/0.25           & 0.2/6.3              & 17.9 /3.6          & 0.3 /6.3 \\
    && 0.2&   EM       & 89.1/65.3  & 0.16/0.11 & 0.2/1.7  &
    4.2/0.6  &0.2/1.7   \\
    && 1  &  One-step  & 205.1/158.5          & 0.35/0.26           &  0.1/7.9        &
    41.6/7.0     & 0.3/ 7.8  \\
    && 1  &  EM        & 121.7/86.2 & 0.19/0.13 & 0.3 /2.8   &
    18.0/3.3 & 0.3/2.8  \\
    \cline{2-9}
    &\multirow{6}{*}{(IV)}
    & 0   & One-step   & 237.2/191.0          & 0.34/0.27           &0.1/5.1               &
    91.1/58.7     & 1.7/6.1  \\
    && 0  & EM         & 180.9/136.3 & 0.24/0.19 & 0.1/2.2  &
    84.7/59.4     & 1.6/3.3 \\
    && 0.2& One-step   & 302.1/231.3          & 0.41/0.30           & 0.1/7.5        &
    88.4/54.7   & 2.0/8.5 \\
    && 0.2& EM         & 228.9/162.6& 0.28/0.21 & 0.2/3.2        &
    82.1/55.1     &2.0\ 4.3 \\
    && 1  & One-step   & 290.4/241.2          & 0.38/0.27           & 0.0/8.8             &
    98.0/65.3     & 3.6/10.8\\
    && 1  & EM         & 265.3/203.6& 0.29/0.22 & 0.0/3.9  &
    98.0/68.2     & 3.6/6.2\\
      \hline
  \end{tabular}
  \end{center} 
\end{table}

\begin{figure}[H]\footnotesize
	\begin{center}
	\hspace{-0.5in}
		\includegraphics[width=5.5in]{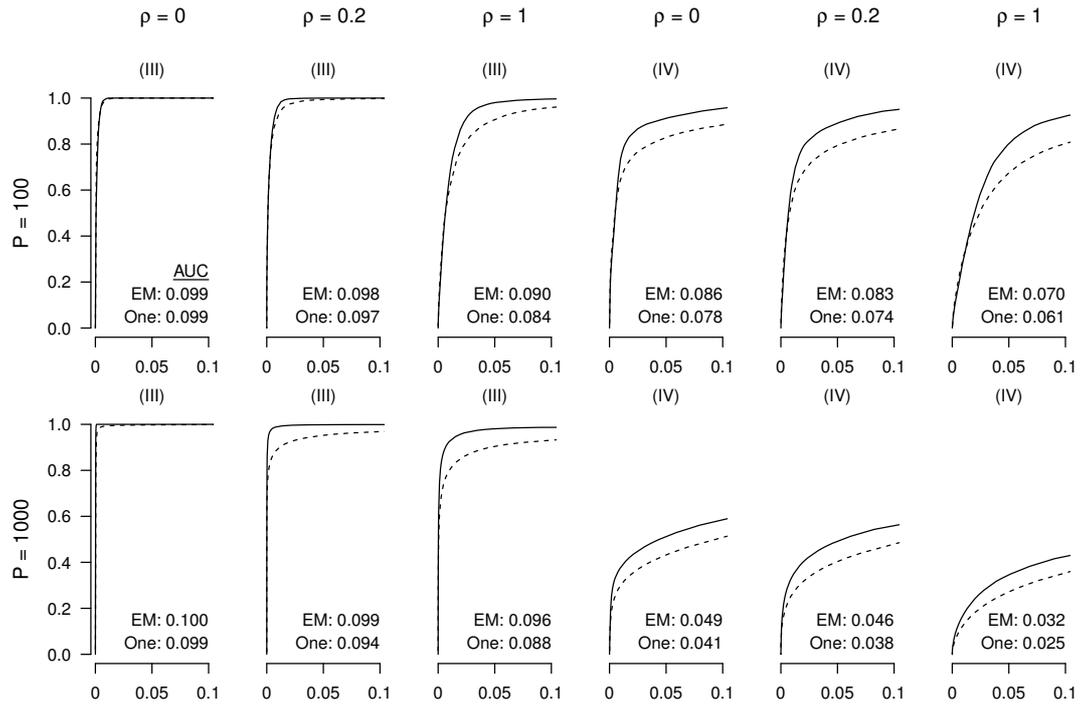}
		\caption{Receiver operating characteristic curves assessing power and discrimination for all networks. Each panel reports performance of the EM method (solid line) and the one-step method (dashed line), plotting true positive rate (y-axis) against false positive rate (x-axis) for a given noise ratio $\rho$, sample size $n = 300$, number of neighbour $m = 5$ and $25$ for $p = 100$ and $1000$, respectively. The numbers in each panel represent the area under the curve for both methods.}
	\label{fig:roc1II}
	\end{center}
\end{figure}

\begin{figure}[H]\footnotesize
	\begin{center}
		\hspace{-0.2in}
		\includegraphics[width=\textwidth]{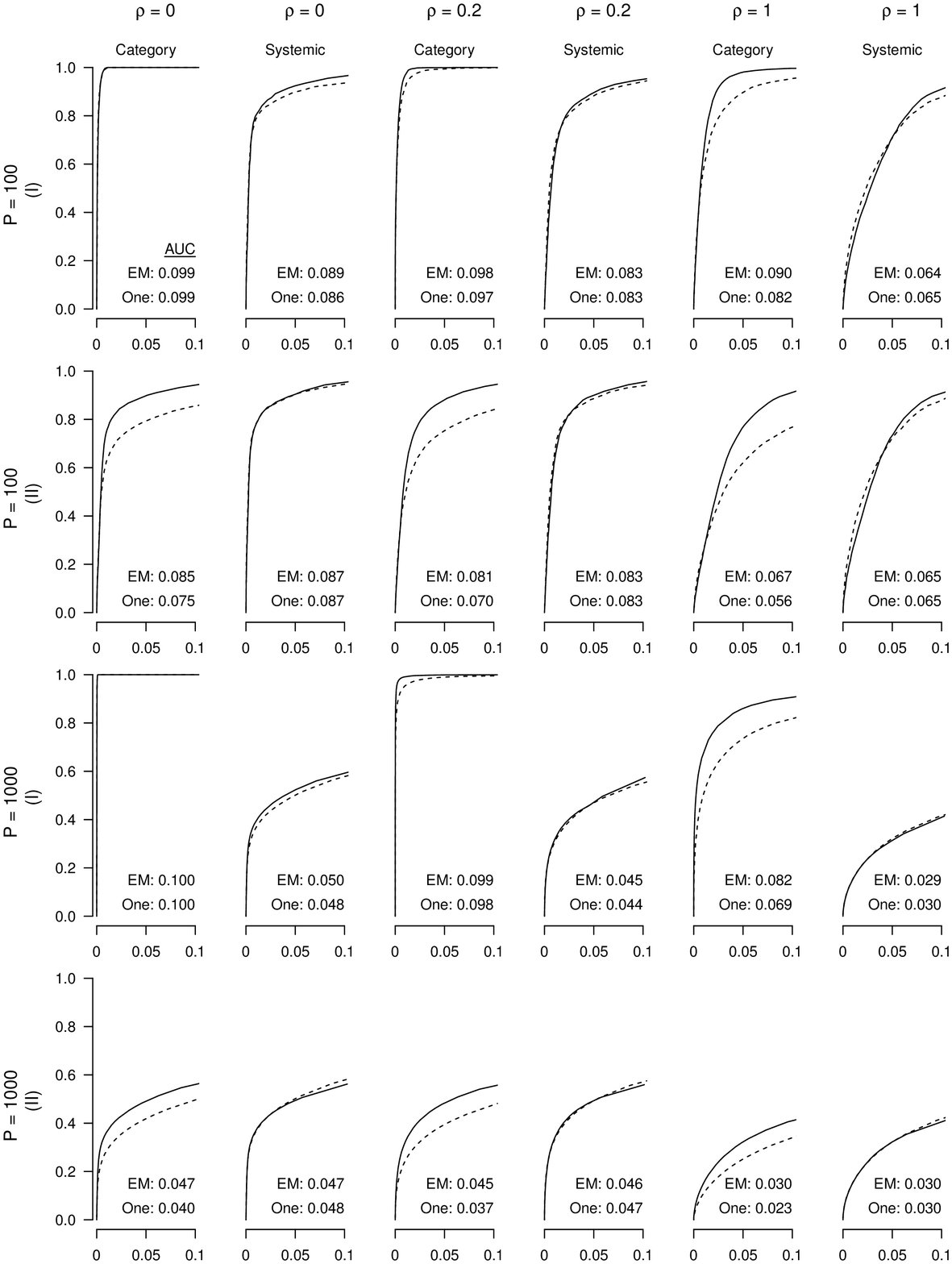}
		\caption{Receiver operating characteristic  curves assessing power and discrimination for category-specific and systemic networks. Each panel reports performance of the EM method (solid line) and the one-step method (dashed line), plotting true positive rate (y-axis) against false positive rate (x-axis) for a given noise ratio $\rho$, sample size $n = 300$, number of neighbour $m = 5$ and $25$ for $p = 100$ and $1000$, respectively. The numbers in each panel represent the area under the curve for both methods.}
	\label{fig:roc2}
	\end{center}
\end{figure}

\begin{figure}[H]\footnotesize
	\begin{center}
		\hspace{-0.13in}
		\includegraphics[width=\textwidth]{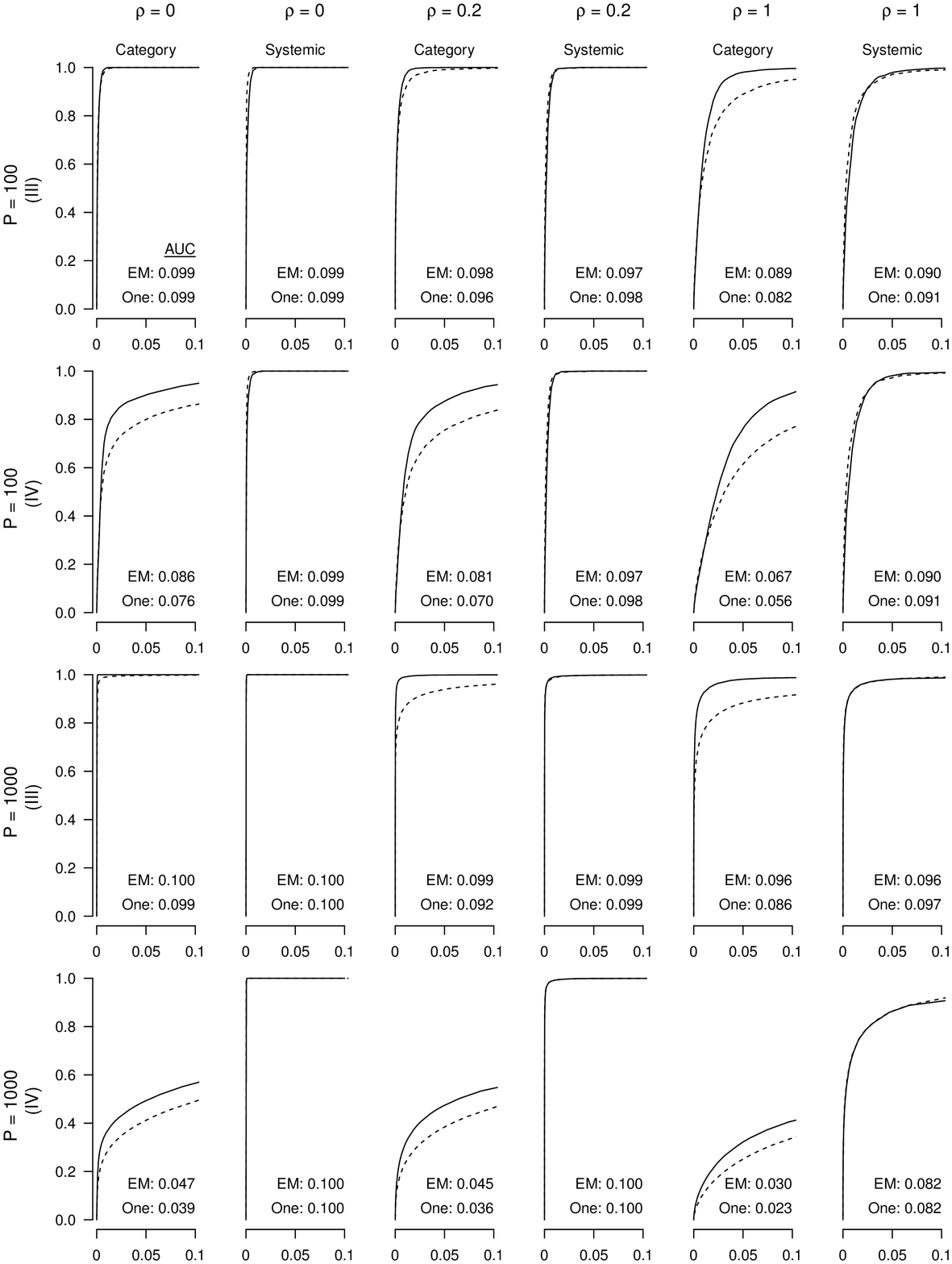}
		\caption{Receiver operating characteristic  curves assessing power and discrimination for category-specific and systemic networks. Each panel reports performance of the EM method (solid line) and the one-step method (dashed line), plotting true positive rate (y-axis) against false positive rate (x-axis) for a given noise ratio $\rho$, sample size $n = 300$, number of neighbour $m = 5$ and $25$ for $p = 100$ and $1000$, respectively. The numbers in each panel represent the area under the curve for both methods.}
	\label{fig:roc3}
	\end{center}
\end{figure}


\subsection{Estimation of aggregate networks $\bO_{Y_k}$}
\label{appendix:aggre}

Here we report the results for estimating aggregate networks using the EM method, the one-step method, the hierarchical penalized likelihood method  proposed by \citet{Guo2011}, and the joint graphical lasso proposed by  \citet{Danaher2013}. 
As shown in Tables \ref{tab:simaggrep100} and \ref{tab:simaggrep1000}, under most simulation settings except (III) with $\rho = 0$ and $0.2$, the EM method performs the best in terms of entropy and Frobenius losses for both $p = 100$ and $1000$. For the (III) setting with $\rho = 0$ and $0.2$, the corresponding $\bO_{Y_k}$ would also have a strong banding structure with large absolute value within the band and small absolute value outside the band. The hierarchical penalized likelihood method performs well because it is designed to work on such structures. 

\begin{table}[H]\footnotesize
\caption{\label{tab:simaggrep100} Summary statistics for estimating aggregate network, $\bO_Y$, under different simulation settings with the dimension $p = 100$.}
\vspace{-8pt}
{
\begin{center}
  \begin{tabular}{c c c| c c c |c c c c }
  \hline
  True networks & $\rho$
  & \multirow{2}{*}{Method} & \multirow{2}{*}{EL} & \multirow{2}{*}{FL} & \multirow{2}{*}{HD} & \multicolumn{4}{c}{HD to the estimate of}\\
  & & & &  &  & HP & JGL & One-step & EM\\
  \hline
  \multirow{12}{*}{(I)}
    &0    & HP          & 9.10    & 0.131 & 0.845  & 0     &       &       & \\
    &0    & JGL         & 9.65    & 0.161 & 0.603  & 0.244 & 0     &       & \\
    &0    &  One-step   & 7.38    & 0.129 & 0.000  & 0.845 & 0.603 & 0     & \\
    &0    &  EM         & 4.63    & 0.069 & 0.000  & 0.845 & 0.603 & 0.000 & 0\\
    \cline{2-10}
    &0.2   & HP         & 9.11    & 0.120 & 0.819  & 0     &       &       &\\
    &0.2   & JGL        & 8.92    & 0.136 & 0.575  & 0.247 & 0     &       &\\
    &0.2   & One-step   & 6.51    & 0.114 & 0.000  & 0.819 & 0.575 &       &\\
    &0.2   & EM         & 4.36    & 0.070 & 0.000  & 0.819 & 0.575 & 0.000 & 0\\
    \cline{2-10}
    &1     & HP         & 11.00   & 0.140 & 0.778  & 0     &       &       & \\
    &1     & JGL        & 9.95    & 0.132 & 0.472  & 0.308 & 0     &       & \\
    &1     & One-step   & 7.69    & 0.121 & 0.000  & 0.778 & 0.472 & 0     & \\
    &1     & EM         & 5.52    & 0.086 & 0.000  & 0.778 & 0.472 & 0.000 & 0\\
    \cline{1-10}
    \multirow{12}{*}{(II)}
    &0     & HP        & 6.72     & 0.101 & 0.829  & 0     &       &       & \\
    &0     & JGL       & 6.41     & 0.106 & 0.633  & 0.199 & 0     &       & \\
    &0     &  One-step & 4.89     & 0.085 & 0.000  & 0.829 & 0.633 & 0     & \\
    &0     &  EM       & 3.37     & 0.058 & 0.000  & 0.829 & 0.633 & 0.000 & 0 \\
    \cline{2-10}
    &0.2   & HP        & 8.91     & 0.132 & 0.808  & 0     &       &       &\\
    &0.2   & JGL       & 8.26     & 0.133 & 0.562  & 0.249 & 0     &       & \\
    &0.2   & One-step  & 6.26     & 0.114 & 0.000  & 0.808 & 0.562 & 0     & \\
    &0.2   & EM        & 4.50     & 0.078 & 0.000  & 0.808 & 0.562 & 0.000 & 0\\
    \cline{2-10}
    &1     & HP        & 10.73    & 0.129 & 0.669  & 0     &       &       & \\
    &1     & JGL       & 9.24     & 0.120 & 0.426  & 0.253 & 0     &       & \\
    &1     & One-step  & 7.38     & 0.104 & 0.000  & 0.669 & 0.426 & 0     & \\
    &1     & EM        & 5.28     & 0.072 & 0.000  & 0.669 & 0.426 & 0.000 & 0\\
  \hline
    \multirow{12}{*}{(III)}
    &0     & HP         & 2.74    & 0.037 & 0.959  & 0     &       &       & \\
    &0     & JGL        & 4.93    & 0.093 & 0.828  & 0.138 & 0     &       & \\
    &0     & One-step   & 8.70    & 0.152 & 0.000  & 0.958 & 0.828 & 0     & \\
    &0     & EM         & 5.10    & 0.089 & 0.000  & 0.958 & 0.828 & 0.001 & 0\\
    \cline{2-10}
    &0.2   & HP         & 6.14    & 0.072 & 0.911  & 0     &       &       & \\
    &0.2   & JGL        & 6.91    & 0.099 & 0.743  & 0.172 & 0     &       & \\
    &0.2   & One-step   & 7.55    & 0.131 & 0.000  & 0.911 & 0.743 & 0     & \\
    &0.2   & EM         & 4.82    & 0.078 & 0.000  & 0.911 & 0.743 & 0.000 & \\
    \cline{2-10}
    &1     & HP         & 10.80   & 0.129 & 0.825  & 0     &       &       & \\
    &1     & JGL        & 9.93    & 0.134 & 0.579  & 0.251 & 0     &       & \\
    &1     & One-step   & 8.47    & 0.135 & 0.000  & 0.825 & 0.579 & 0     & \\
    &1     & EM         & 5.72    & 0.090 & 0.000  & 0.825 & 0.579 & 0.000 & 0 \\
    \cline{1-10}
    \multirow{12}{*}{(IV)}
    &0     & HP         & 7.73    & 0.111 & 0.827  & 0     &       &       & \\
    &0     & JGL        & 7.74    & 0.116 & 0.588  & 0.242 & 0     &       & \\
    &0     & One-step   & 5.29    & 0.093 & 0.000  & 0.827 & 0.588 & 0     & \\
    &0     & EM         & 3.39    & 0.058 & 0.000  & 0.827 & 0.588 & 0.000 & 0 \\
    \cline{2-10}
    &0.2   & HP         & 9.63    & 0.139 & 0.824  & 0     &       &       & \\
    &0.2   & JGL        & 8.90    & 0.138 & 0.618  & 0.210 & 0     &       & \\
    &0.2   & One-step   & 6.61    & 0.114 & 0.000  & 0.824 & 0.618 & 0     & \\
    &0.2   & EM         &  4.67   & 0.074 & 0.000  & 0.824 & 0.618 & 0.000 & 0 \\
    \cline{2-10}
    &1     & HP         & 13.22   & 0.169 & 0.777  & 0     &       &       & \\
    &1     & JGL        & 11.45   & 0.150 & 0.466  & 0.313 & 0     &       & \\
    &1     & One-step   & 8.94    & 0.130 & 0.000  & 0.777 & 0.466 & 0     & \\
    &1     & EM         & 6.34    & 0.085 & 0.000  & 0.777 & 0.466 & 0.000 & 0 \\
  \hline 
  \end{tabular}
  \end{center}}
\end{table}

\begin{table}[H]\footnotesize
\caption{\label{tab:simaggrep1000} Summary statistics for estimating aggregate network, $\bO_Y$, under different simulation settings with the dimension $1000$.}
\vspace{-8pt}
{
\begin{center}
  \begin{tabular}{c c c| c c c |c c c c }
  \hline
  True networks & $\rho$
  & \multirow{2}{*}{Method} & \multirow{2}{*}{EL} & \multirow{2}{*}{FL} & \multirow{2}{*}{HD} & \multicolumn{4}{c}{HD to the estimate of}\\
  & & & &  &  & HP & JGL & One-step & EM\\
  \hline
  \multirow{12}{*}{(I)}
    &0    & HP          & 180.56  & 0.236 & 0.996 & 0     &       &       &\\
    &0    & JGL         & 184.13  & 0.275 & 0.935 & 0.061 & 0     &       &\\
    &0    &  One-step   & 226.18  & 0.378 & 0.000 & 0.996 & 0.935 & 0     &\\
    &0    &  EM         & 116.93  & 0.185 & 0.000 & 0.996 & 0.935 & 0.000 & 0 \\
    \cline{2-10}
    &0.2   & HP         & 182.27  & 0.229 & 0.997 & 0     &       &       &\\
    &0.2   & JGL        & 172.75  & 0.229 & 0.930 & 0.068 & 0     &       &\\
    &0.2   & One-step   & 194.12  & 0.321 & 0.000 & 0.997 & 0.930 & 0     &\\
    &0.2   & EM         & 124.84  & 0.177 & 0.000 & 0.997 & 0.930 & 0.000 & 0 \\
    \cline{2-10}
    &1     & HP         & 199.90  & 0.249 & 0.999 & 0     &       & &\\
    &1     & JGL        & 180.97  & 0.226 & 0.940 & 0.059 & 0     & &\\
    &1     & One-step   & 191.54  & 0.298 & 0.000 & 0.999 & 0.940 & 0 &\\
    &1     & EM         & 150.99  & 0.198 & 0.000 & 0.999 & 0.940 & 0.000 & 0\\
    \cline{1-10}
    \multirow{12}{*}{(II)}
    &0     & HP        & 289.89  & 0.363  & 0.998 & 0     & & & \\
    &0     & JGL       & 201.59  & 0.269  & 0.927 & 0.071 & 0     & & \\
    &0     &  One-step & 209.72  & 0.325  & 0.000 & 0.998 & 0.927 & 0 & \\
    &0     &  EM       & 138.62  & 0.210  & 0.000 & 0.998 & 0.927 & 0.000 & 0\\
    \cline{2-10}
    &0.2   & HP        & 274.06  & 0.301  & 0.998 & 0 & & & \\
    &0.2   & JGL       & 231.09  & 0.268  & 0.896 & 0.102 & 0 & &\\
    &0.2   & One-step  & 226.57  & 0.308  & 0.000 & 0.998 & 0.896 & 0 & \\
    &0.2   & EM        & 169.39  & 0.213  & 0.000 & 0.998 & 0.896 & 0.000 & 0 \\
    \cline{2-10}
    &1     & HP        & 271.00  & 0.300 & 0.999 & 0 & & &\\
    &1     & JGL       & 240.23  & 0.269 & 0.894 & 0.106 & 0 & &\\
    &1     & One-step  & 237.92  & 0.300 & 0.000 & 0.999 & 0.894 & 0 &\\
    &1     & EM        & 206.75  & 0.235 & 0.000 & 0.999 & 0.894 & 0.000 & 0\\
  \hline
    \multirow{12}{*}{(III)}
    &0     & HP        & 41.40   & 0.071 & 0.640 & 0 & & &\\
    &0     & JGL       & 57.52   & 0.114 & 0.631 & 0.031 & 0 & &\\
    &0     & One-step  & 168.78  & 0.320 & 0.000 & 0.640 & 0.631 & 0 &\\
    &0     & EM        & 69.34   & 0.143 & 0.000 & 0.640 & 0.631 & 0.000 & 0 \\
    \cline{2-10}
    &0.2   & HP        & 58.64   & 0.089 & 0.925 & 0 & & &\\
    &0.2   & JGL       & 69.24   & 0.110 & 0.904 & 0.024 & 0     & &\\
    &0.2   & One-step  & 134.28  & 0.256 & 0.000 & 0.925 & 0.904 & 0 &\\
    &0.2   & EM        & 63.99   & 0.109 & 0.000 & 0.925 & 0.904 & 0.000 & 0 \\
    \cline{2-10}
    &1     & HP        & 125.25  & 0.150 & 0.998 & 0 & & &\\
    &1     & JGL       & 133.06  & 0.178 & 0.955 & 0.043 & 0 & &\\
    &1     & One-step  & 158.45  & 0.258 & 0.000 & 0.998 & 0.955 & 0 &\\
    &1     & EM        & 96.54   & 0.137 & 0.000 & 0.998 & 0.955 & 0.000 & 0\\
    \cline{1-10}
    \multirow{12}{*}{(IV)}
    &0     & HP         & 143.18  & 0.194 & 0.997 & 0 & & &\\
    &0     & JGL        & 129.12  & 0.174 & 0.936 & 0.061 & 0 & &\\
    &0     & One-step   & 154.60  & 0.247 & 0.000 & 0.997 & 0.936 & 0 &\\
    &0     & EM         & 91.92   & 0.130 & 0.000 & 0.997 & 0.936 & 0.000& 0\\
    \cline{2-10}
    &0.2   & HP         & 193.65  & 0.229 & 0.998 & 0 & & &\\
    &0.2   & JGL        & 183.22  & 0.236 & 0.931 & 0.067 & 0 & & \\
    &0.2   & One-step   & 198.50  & 0.295 & 0.000 & 0.998 & 0.931 & 0 &\\
    &0.2   & EM         & 121.55  & 0.164 & 0.000 & 0.998 & 0.931 & 0.000 & 0 \\
    \cline{2-10}
    &1     & HP         & 196.32  & 0.250 & 0.999 & 0 & & &\\
    &1     & JGL        & 184.40  & 0.225 & 0.949 & 0.051 & 0 & &\\
    &1     & One-step   & 202.75  & 0.274 & 0.000 & 0.999 & 0.949 & 0 &\\
    &1     & EM         & 151.56  & 0.182 & 0.000 & 0.999 & 0.949 & 0.000 &0 \\
  \hline 
  \end{tabular}
  \end{center}}
\end{table}

\subsection{Summary of Computational Time}
\label{appendix:Runtime}

In this subsection, we report the computational time, the number of iterations,	 and the corresponding total edge numbers among $\bO_k$ $(k = 0, \ldots, K)$ for both the one-step and the EM methods. These results are generated from a personal laptop with 	8GB RAM running the Linux system. Table \ref{tab:runtime} and Fig. \ref{fig:SupFig5} show that the computation time and the number of iteration depend on the value of $\lambda$ ($\lambda_1$ and $\lambda_2$). When $\lambda$ is reasonably large and hence the corresponding $\Omega$'s are sparse, the computation is quite efficient. The computation can take longer for very small $\lambda$'s, and the resulting $\hat{\Omega}$'s are typically dense. 

\begin{table}[H]\footnotesize
\caption{\label{tab:runtime} Summary of computational time on simulation data. In each entry, the numbers before and after the slash correspond to $\lambda = 0.04$ and $\lambda = 0.3$, respectively. The results show that the run time decreases as $\lambda$ increases, for both the  one-step and EM methods. }
\vspace{-8pt}
{
\begin{center}
  \begin{tabular}{c c c| c c c c c}
  \hline
  \specialcell{True networks\\ Category/ Systemic} & $p$ &
   $\rho$ &Method & $\lambda$ & Number of Iterations &\specialcell{Time \\ (In seconds)}
     & Number of Edges      \\
      \cline{1-8}
     \multirow{8}{*}{(I)}
    &100   & 0      &One-step & 0.04/0.3 & 1.0/1.0 & 0.4/0.1 &  10092.0/808.8\\
    &100   & 0      &EM       & 0.04/0.3 & 6.2/4.0 & 6.0/1.5 &  5042.1/557.1 \\
    &100   & 1      &One-step & 0.04/0.3 & 1.0/1.0 & 0.8/0.1 &  10680.2/953.7 \\
    &100   & 1      &EM       & 0.04/0.3 & 6.0/4.0 & 6.1/1.4 &  6322.4/322.4 \\
    \cline{2-8}
    &1000  & 0      &One-step & 0.04/0.3 & 1.0/1.0  & 1913.8/26.4 & 599406.1/8355.1 \\
    &1000  & 0      &EM       & 0.04/0.3 & 6.0/3.0 & 7141.0/129.3 & 330421.7/5481.3\\
    &1000  & 1      &One-step & 0.04/0.3 & 1.0/1.0 & 1645.4/21.9  & 607064.0/14411.2 \\
    &1000  & 1      &EM       & 0.04/0.3 & 5.0/4.1 & 7140.1/103.3 &339917.6/10144.8 \\
        \cline{1-8}
\multirow{8}{*}{(II)}
    &100   & 0      &One-step & 0.04/0.3 & 1.0/1.0 & 1.2/0.1 &  11764.0/783.2 \\
    &100   & 0      &EM       & 0.04/0.3 & 7.0/4.0 & 7.1/1.5 &  6188.0/740.5  \\
    &100   & 1      &One-step & 0.04/0.3 & 1.0/1.0 & 0.8/0.1 &  11347.6/1422.6\\
    &100   & 1      &EM       & 0.04/0.3 & 6.0/5.0 & 6.0/1.4 & 6995.5/464.9      \\
    \cline{2-8}
    &1000  & 0      &One-step & 0.04/0.3 & 1.0/1.0   & 2095.4/32.4 &  593094.0/8166.2  \\
    &1000  & 0      &EM       & 0.04/0.3 & 5.0/3.0  & 7257.1/152.0 &  385892.5/3261.0     \\
    &1000  & 1      &One-step & 0.04/0.3 & 1.0/1.0  & 1211.6/22.2  &  614542.5/1533.3 \\
    &1000  & 1      &EM       & 0.04/0.3 & 8.0/4.0  & 5256.3/95.3  &  414699.0/1250.0 \\
    \cline{1-8}
  \hline\multirow{8}{*}{(III)}
    &100   & 0      &One-step & 0.04/0.3 & 1.0/1.0 &0.9/0.1        & 7962.8/771.3\\
    &100   & 0      &EM       & 0.04/0.3 & 5.9/5.1  &5.7/2.2 &  3541.3/560.5     \\
    &100   & 1      &One-step & 0.04/0.3 & 1.0/1.0 & 0.8/0.1 &     9413.1/813.7 \\
    &100   & 1      &EM       & 0.04/0.3 & 6.0/4.0 & 6.4/1.8 &     4857.7/237.6  \\
    \cline{2-8}
    &1000  & 0      &One-step & 0.04/0.3 & 1.0/1.0 & 1480.3/32.5 &  539388.2/7002.7 \\
    &1000  & 0      &EM       & 0.04/0.3 & 6.1/3.0 & 7409.4/141.4 &  318793.4/5439.7 \\
    &1000  & 1      &One-step & 0.04/0.3 & 1.0/1.0 & 1176.9/24.2 &  604940.3/4086.0 \\
    &1000  & 1      &EM       & 0.04/0.3 & 5.0/3.0 & 4657.8/115.2 & 394795.6/721.6\\
    \cline{1-8}
  \hline\multirow{8}{*}{(IV)}
    &100   & 0      &One-step & 0.04/0.3 &  1.0/1.0 & 0.9/0.1  & 10784.4/752.6\\
    &100   & 0      &EM       & 0.04/0.3 &  6.0/4.0 & 6.4/1.4  & 5913.5/437.3 \\
    &100   & 1      &One-step & 0.04/0.3 &  1.0/1.0 & 0.8/0.1  & 11155.9/1281.3\\
    &100   & 1      &EM       & 0.04/0.3 &  6.0/4.0 & 6.4/1.6  & 6785.3/406.6 \\
    \cline{2-8}
    &1000  & 0      &One-step & 0.04/0.3 &  1.0/1.0 & 2189.4/38.0 & 623020.4/8097.7    \\
    &1000  & 0      &EM       & 0.04/0.3 &  5.0/4.0 & 9740.5/151.8 &  379012.5/3549.2  \\
    &1000  & 1      &One-step & 0.04/0.3 &  1.0/1.0 & 1646.2/21.2 & 620378.8/1339.9 \\
    &1000  & 1      &EM       & 0.04/0.3 &  5.0/4.0 & 4190.4/92.1 &  415572.8/102.0     \\
    \cline{1-8}
  \hline
  \end{tabular}
  \end{center}}
\end{table}

\begin{figure}[H]\footnotesize
	\begin{center}
		\includegraphics[width=5.4in, height=1.8in]{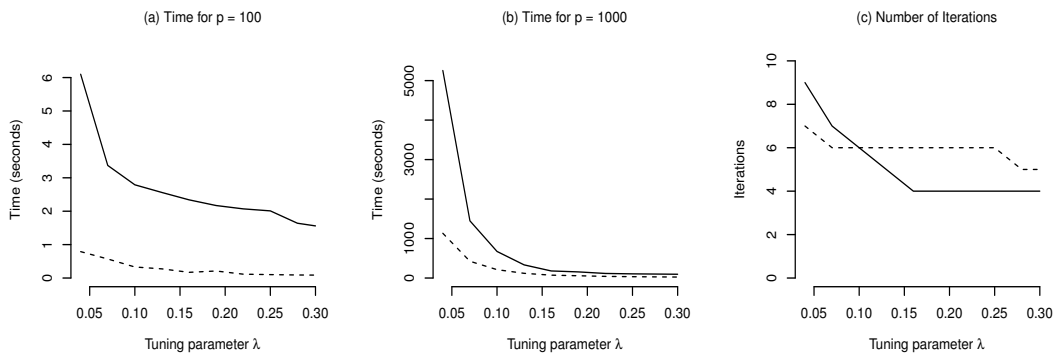}
		\caption{ { Comparisons of computational time between EM and one-step methods under NN/NN structure with $\rho = 1$. Panels (a) and (b) display the computing time for the EM method (solid line) and the one-step method (dashed line) with $p = 100$ and $1000$, respectively. Panel (c) shows the number of iterations required for the EM method to converge under different tuning parameters for $p = 100$ (dashed line) and $p = 1000$ (solid line). The results show that the run time and number of iterations for the EM method decreases as $\lambda$ increases. }}
	\label{fig:SupFig5}
	\end{center}
\end{figure}


\section{Further details on the normalization and analysis of the real data}

In our analysis of expression data from both \citet{Dobrin2009} and \citet{Crowley2014}, we perform normalizations to the data before network estimation. In this way, it allows us to focus on covariances (and thereby dependencies) between genes rather than on gross differences in means. 

For the \citet{Dobrin2009} data, this involves removing the mean effect of tissue, specifically, normalizing each gene within each tissue to have mean 0 and standard deviation 1. 
In the case of \citet{Crowley2014}, the normalization also requires removing effects of gross genetic background. \citet{Crowley2014} performed three independent reciprocal crosses using all pairs from three genetically dissimilar inbred strains, CAST, PWK and WSB, to give the following six types of hybrid mice, listed as mother $\times$ father: CAST$\times$PWK and PWK$\times$CAST; CAST$\times$WSB and WSB$\times$CAST; and PWK$\times$WSB and WSB$\times$PWK. The study sample, therefore, had a nested design, with a parent of origin nested within strain pair, for example, CAST$\times$PWK vs. PWK$\times$CAST, nested within the pairing of strains PWK and CAST. Of these, the outer level factor, strain pair, would trivially be expected to have extremely strong mean effects on gene expression owing to the fact that gene expression is heritable and the three strains are highly genetically dissimilar. To remove this outer level factor and focus primarily on dependencies induced by varying parent-of-origin, we, therefore, centered the expression of each gene within its strain pair.


\section{Test for the existence of the systemic layer in \FTwo Mice}

To test the existence of a systemic layer in the \FTwo mice, which is a key assumption of our model, we define a set of genes to be examined and perform two significance tests: one testing for the existence of a systemic graph shared across tissues, and another examining support for additional structure beyond this, specifically, testing for the existence of graphs shared between tissue pairs. The set of genes used was the same set of 1000 genes used in the main manuscript, these having the largest within-group variance among the four tissues in the \FTwo population. The data matrix $y_{k,\cdot} = \{y_{k, 1}, \ldots, y_{k, n}  \}^\T$ for the $k$th is of dimension $n \times p$. The significance tests are described below.

\begin{figure}[H]\footnotesize
	\begin{center}
		\includegraphics[height=0.9\textheight]{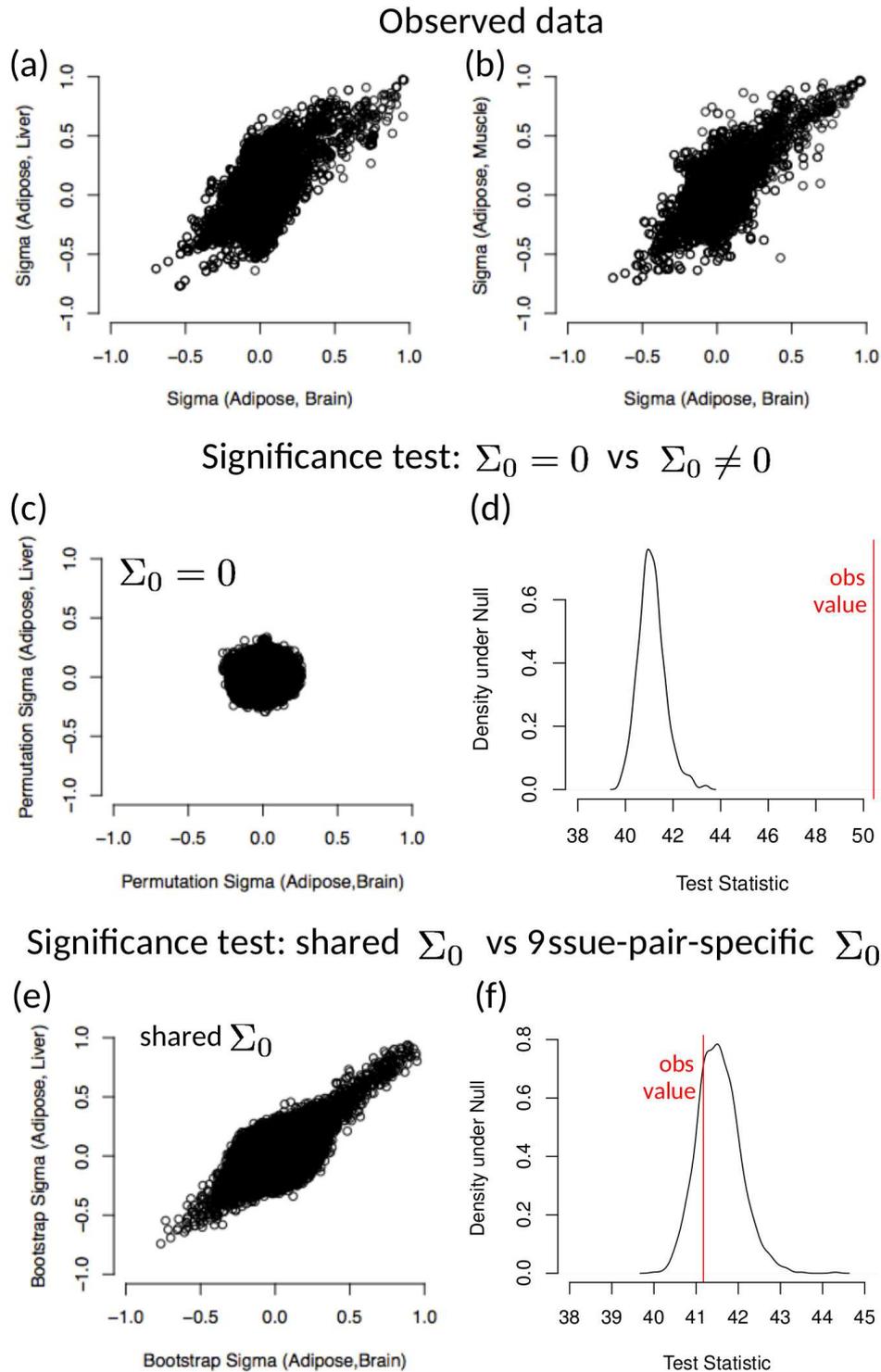}
		\caption{Cross-tissue covariance matrices comparison. Each dot represents an entry in the covariance matrix between different tissues. Panels (a) and (b) correspond to the comparison between $\Sigma_{\text{Adipose, Brain}}$ and $\Sigma_{\text{Adipose, Liver}}$, and between $\Sigma_{\text{Adipose, Brain}}$ and $\Sigma_{\text{Adipose, Muscle}}$, respectively. Panel (c) is comparison between the permuted cross-tissue covariance which represents $H_{01}$: $\Sigma_0 = 0$, and the density under $H_{01}$ is shown in panel (d) with the red vertical line representing the observed test statistics with $p$-value $= 0$. Panel (e) is comparison between cross-tissue covariance from parametric bootstrap data which represents $H_{02}$: $\Sigma_{y(l,k)} = \Sigma_0$ for any $l \neq k$. The density under $H_{02}$ is shown in panel (f) with the red vertical line representing the observed test statistics with $p$-value $= 0.714$. }
	\label{fig:SupFig4}
	\end{center}
\end{figure}

\subsection{Test for the existence of $\Sigma_0$}

Here we test $H_{01}$: $\Sigma_0 = 0$ vs $H_{11}$: $\Sigma_0 \neq 0$. To generate data under the model of $H_{01}$, we permute the mouse order within each tissue so that we can remove the between-tissue correlation within each mouse. Specifically, for any permutation $\pi$ of $ \{ 1, \ldots, n \}$, let $y_{k, \pi} = (y_{k, \pi(1)}, \ldots , y_{k,\pi(n)})^T$ be the corresponding permuted version of the matrix $y_{k,\cdot}$. Let $\pi_1^1, \ldots, \pi_K^1$ represent $K$ different sets of permutations from  $ \{ 1, \ldots, n \}$. We then obtain 1000 permuted data as $\{y_{1, \pi_1^{m}}, \ldots, y_{K, \pi_K^{m}}\}$ $\; (m = 1, \ldots, 1000)$. With the permuted data, we calculate the between-tissue covariance for the permuted mice (between mice) as 
\[
	\hat{\Sigma}_{Y(l, k)}^{\pi^m} = \frac{1}{2n} \Big( 
 	y^\T_{l, \pi_l^{m}} y_{k, \pi_k^{m}} + y^\T_{k, \pi_k^{m}}y_{l, \pi_l^{m}} \Big).
\]
The scatter plot for entries of $\hat{\Sigma}_{Y(\text{Adipose, Brain})}^{\pi^1}$ and $\hat{\Sigma}_{Y(\text{Adipose, Liver})}^{\pi^1}$ from a typical set of permutation is shown in Fig. \ref{fig:SupFig4} (c) where the round shape around the origin indicates that $H_{01}$  holds for the permuted data. However, as shown in Fig. \ref{fig:SupFig4}(a) and (b), the entries of $\hat{\Sigma}_{Y(\text{Adipose, Brain})}$, $\hat{\Sigma}_{Y(\text{Adipose, Liver})}$ and   $\hat{\Sigma}_{Y(\text{Adipose, Muscle})}$ from the observed data are more spread out from the origin, indicating $\Sigma_0 \neq 0$. 

To test formally for the existence of $\Sigma_0$ in the real data, that is, the existence of cross-tissue dependence, we define a test statistic $\text{F}_0$ to be the Frobenius norm between 0 and between-tissue covariance matrices and calculate $\; \text{F}_0$ as follows:
\[
	\text{F}_{0}(\Sigma_Y) = \sum_{l \neq k} \| 
	\Sigma_{Y(l, k)} - 0 \|_F.
\]
With the 1000 permuted datasets, we derived the corresponding null distribution for $\text{F}_0$ as shown in Fig. \ref{fig:SupFig4}(d). The red vertical line represents the $\text{F}_0$ calculated from the real data, and the corresponding empirical $p$-value was 0, supporting the existence of non-zero $\Sigma_0$ in our \FTwo Mice. 

\subsection{Test for additional shared structure beyond $\Sigma_0$}

There are a number of different ways in which additional structure can be defined. Here we specifically address shared structure across tissue pairs, by testing $H_{02}: \Sigma_{Y(l,k)}$ are all equal for any $l \neq k$ vs $H_{12}: \Sigma_{Y(l,k)}$ are not all equal. To generate data under the model of $H_{02}$, we use a parametric bootstrap approach as follows. Recalling Equation \eqref{eq:yk} in the paper,  $\Sigma_Y = \{_d \Sigma_k  \} + J \otimes \Sigma_0 $, we first use the original data to estimate $\Sigma_Y$ as $\tilde{\Sigma}_Y$ by forcing the off-diagonal block $\Sigma_{Y(l, k)}$ for $l \neq k$ to be identical. From the distribution $ \mathcal{N}( 0, \tilde{\Sigma}_Y)$, we generate 1000 sets of data with  the sample size $n = 301$. From each simulated dataset, we calculate the between-tissue covariance $\hat{\Sigma}_{Y(l,k)}$.
The relationship between  $\hat{\Sigma}_{Y(\text{1, 2})}$ and $\hat{\Sigma}_{Y(\text{1, 3})}$ from a typical simulation is shown in Fig. \ref{fig:SupFig4}(e), where the dots in the diagonal line suggest that $H_{02}$  holds. Similarly, as shown in Fig. \ref{fig:SupFig4}(a) and (b), we also observe a similar strong diagonal line pattern in our real data, suggesting that our model is reasonable.  

To test formally the hypothesis $H_{02}$ vs. $H_{12}$, we define the test statistic, $\text{F}_{\text{mean}}$, as the mean distance for each off-diagonal block matrix from their corresponding mean matrix,
\[
	\text{F}_{\text{mean}}(\Sigma_Y) = \sum_{l \neq k} \| \Sigma_{Y(l, k)} - \bar{\Sigma}_{Y,\text{off}} \|_F\,,
\]
where $\bar{\Sigma}_{Y,\text{off}}  = \sum_{l \neq k} \Sigma_{Y(l,k)}/\{K (K -1)  \} $.
As in Fig \ref{fig:SupFig4}(f), the density curve is generated using 1000 simulated datasets reflecting the null distribution of $\text{F}_{\text{mean}}$ under  $H_{02}$, and the vertical line represents the statistic from the real data with corresponding  $p$-value $= 0.714$. This suggests that the covariance matrices between different tissues are not significantly different, which supports our model assumption. We need to point out that here we use Frobenius norm to measure the difference between matrices. One can also use different norms, for example $\| \cdot \|_\infty $ or $ \| \cdot \|_1$ , and the corresponding $p$-value may vary.

\end{appendices}
\bibliographystyle{biometrika}
\bibliography{JDE_BKA-1}
\end{document}